\documentclass[twoside,11pt]{article}

\usepackage{amsmath}
\usepackage{amsthm}
\usepackage{amssymb}
\usepackage{graphicx}
\usepackage{subfigure}
\usepackage[numbers]{natbib}
\usepackage{algorithm}
\usepackage{algorithmic}
\usepackage[hidelinks]{hyperref}
\usepackage{authblk}
\usepackage{fullpage}
\usepackage{multirow}
\usepackage{natbib}

\usepackage[shortlabels]{enumitem}

\DeclareMathOperator*{\argmin}{argmin}
\DeclareMathOperator*{\sign}{sign}

\newtheorem{theorem}{Theorem}
\newtheorem{corollary}[theorem]{Corollary}
\newtheorem{lemma}[theorem]{Lemma}
\newtheorem{definition}[theorem]{Definition}
\newtheorem{remark}{Remark}

\renewcommand{\eqref}[1]{Eq.~(\ref{#1})}
\newcommand{\figref}[1]{Fig.~\ref{#1}}

\title{Lasso Screening Rules via Dual Polytope Projection}
\author[1]{Jie Wang}
\author[1]{Peter Wonka}
\author[1]{Jieping Ye}
\affil[1]{Computer Science and Engineering, Arizona State University,
            USA}
\date{}

\begin{document}

\maketitle

\begin{abstract}
Lasso is a widely used regression technique to find sparse representations. When the dimension of the feature space and the number of samples are extremely large, solving the Lasso problem remains challenging. To improve the efficiency of solving large-scale Lasso problems, El Ghaoui and his colleagues have proposed the SAFE rules which are able to quickly identify the inactive predictors, i.e., predictors that have $0$ components in the solution vector. Then, the inactive predictors or features can be removed from the optimization problem to reduce its scale. By transforming the standard Lasso to its dual form, it can be shown that the inactive predictors include the set of inactive constraints on the optimal dual solution. In this paper, we propose an efficient and effective screening rule via Dual Polytope Projections (DPP), which is mainly based on the uniqueness and nonexpansiveness  of the optimal dual solution due to the fact that the feasible set in the dual space is a convex and closed polytope. Moreover, we show that our screening rule can be extended to identify inactive groups in group Lasso. To the best of our knowledge, there is currently no exact screening rule for group Lasso. We have evaluated our screening rule using synthetic and real data sets. Results show that our rule is more effective in identifying inactive predictors than existing state-of-the-art screening rules for Lasso.
\end{abstract}


\section{Introduction}

Data with various structures and scales comes from almost every aspect of daily life. To effectively extract patterns in the data and build interpretable models with high prediction accuracy is always desirable.  One popular technique to identify important explanatory features is by sparse regularization. For instance, consider the widely used $\ell_1$-regularized least squares regression problem known as Lasso \citep{Tibshirani1996}. The most appealing property of Lasso is the sparsity of the solutions, which is equivalent to feature selection. Suppose we have $N$ observations and $p$ features. Let ${\bf y}$ denote the $N$ dimensional response vector and ${\bf X} = [{\bf x}_1, {\bf x}_2, \ldots, {\bf x}_p]$ be the $N\times p$ feature matrix. Let $\lambda\geq0$ be the regularization parameter. The Lasso problem is formulated as the following optimization problem:
\begin{equation}\label{prob:Lasso_primal}
	\inf_{\beta\in\mathbb{R}^p}\frac{1}{2}\left\|{\bf y - X\beta}\right\|_2^2+\lambda\|{\beta}\|_1.
\end{equation}
Lasso has achieved great success in a wide range of applications \citep{Chen2001,Cands2006,Zhao2006,Bruckstein2009,Wright2010} and in recent years many algorithms have been developed to efficiently solve the Lasso problem \citep{Efron04,Kim2007,Park2007,Donoho2008,Friedman2007,Becker2010,Friedman2010}.
However, when the dimension of feature space and the number of samples are very large, solving the Lasso problem remains challenging because we may not even be able to load the data matrix into main memory. The idea of {\it screening} has been shown very promising in solving Lasso for large-scale problems. Essentially, screening aims to quickly identify the {\it inactive features} that have $0$ components in the solution and then remove them from the optimization. Therefore, we can work on a reduced feature matrix to solve the Lasso problem, which may lead to substantial savings in computational cost and memory usage.

Existing screening methods for Lasso can be roughly divided into two categories: the {\it Heuristic Screening Methods} and the {\it Safe Screening Methods}. As the name indicated, the heuristic screening methods can not guarantee that the discarded features have zero coefficients in the solution vector. In other words, they may mistakenly discard the {active features} which have nonzero coefficients in the sparse representations. Well-known heuristic screening methods for Lasso include SIS \citep{Fan2008} and strong rules \citep{Tibshirani11}. SIS is based on the associations between features and the prediction task, but not from an optimization point of view. Strong rules rely on the assumption that the absolute values of the inner products between features and the residue are {\it nonexpansive} \citep{Bauschke2011} with respect to the parameter values. Notice that, in real applications, this assumption is not always true. In order to ensure the correctness of the solutions, strong rules check the KKT conditions for violations. In case of violations, they weaken the screened set and repeat this process. In contrast to the heuristic screening methods, the safe screening methods for Lasso can guarantee that the discarded features are absent from the resulting sparse models. Existing safe screening methods for Lasso includes SAFE \citep{Ghaoui2012} and DOME \citep{Xiang2011,Xiang2012}, which are based on an estimation of the dual optimal solution. The key challenge of searching for effective safe screening rules is how to accurately estimate the dual optimal solution. The more accurate the estimation is, the more effective the resulting screening rule is in discarding the inactive features. Moreover, \citet{Xiang2011} have shown that the SAFE rule for Lasso can be read as a special case of their testing rules.


In this paper, we develop novel efficient and effective screening rules for the Lasso problem; our screening rules are safe in the sense that no active features will be discarded. As the name indicated (DPP), the proposed approaches heavily rely on the geometric properties of the Lasso problem. Indeed, the dual problem of problem (\ref{prob:Lasso_primal}) can be formulated as a projection problem. More specifically, the {\bf d}ual optimal solution of the Lasso problem is the {\bf p}rojection of the scaled response vector onto a nonempty closed and convex {\bf p}olytope (the feasible set of the dual problem). This nice property provides us many elegant approaches to accurately estimate the dual optimal solutions, e.g., nonexpansiveness, firmly nonexpansiveness \citep{Bauschke2011}. In fact, the estimation of the dual optimal solution in DPP is a direct application of the nonexpansiveness of the projection operators. Moreover, by further exploiting the properties of the projection operators, we can significantly improve the estimation of the dual optimal solution. Based on this estimation, we develop the so called {\it enhanced DPP} (EDPP) rules which are able to detect far more inactive features than DPP. Therefore, the speedup gained by EDPP is much higher than the one by DPP.

In real applications, the optimal parameter value of $\lambda$ is generally unknown and needs to be estimated. To determine an appropriate value of $\lambda$, commonly used approaches such as cross validation and stability selection involve solving the Lasso problems over a grid of tuning parameters $\lambda_1>\lambda_2>\ldots>\lambda_{\mathcal{K}}$. Thus, the process can be very time consuming. To address this challenge, we develop the sequential version of the DPP families. Briefly speaking, for the Lasso problem, suppose we are given the solution $\beta^*(\lambda_{k-1})$ at $\lambda_{k-1}$. We then apply the screening rules to identify the inactive features of problem (\ref{prob:Lasso_primal}) at $\lambda_k$ by making use of $\beta^*(\lambda_{k-1})$. The idea of the sequential screening rules is proposed by \cite{Ghaoui2012} and \cite{Tibshirani11} and has been shown to be very effective for the aforementioned scenario. In \cite{Tibshirani11}, the authors demonstrate that the sequential strong rules are very effective in discarding inactive features especially for very small parameter values and achieve the state-of-the-art performance. However, in contrast to the  recursive SAFE (the sequential version of SAFE rules) and the sequential version of DPP rules, it is worthwhile to mention that the sequential strong rules may mistakenly discard active features because they are heuristic methods. Moreover, it is worthwhile to mention that, for the existing screening rules including SAFE and strong rules, the basic versions are usually special cases of their sequential versions, and the same applies to our DPP and EDPP rules. For the DOME rule \citep{Xiang2011,Xiang2012}, it is unclear whether its sequential version exists.

The rest of this paper is organized as follows. We present the family of DPP screening rules, i.e., DPP and EDPP, in detail for the Lasso problem in Section \ref{section:DPP}. Section \ref{section:glasso} extends the idea of DPP screening rules to identify inactive groups in group Lasso \citep{Yuan2006}. We evaluate our screening rules on synthetic and real data sets from many different applications. In Section \ref{section:experiments}, the experimental results demonstrate that our rules are more effective in discarding inactive features than existing state-of-the-art screening rules. We show that the efficiency of the solver can be improved by {\it several orders of magnitude} with the enhanced DPP rules, especially for the high-dimensional data sets (notice that, the screening methods can be integrated with any existing solvers for the Lasso problem). Some concluding remarks are given in Section \ref{section:conclusion}.

\section{Screening Rules for Lasso via Dual Polytope Projections}\label{section:DPP}

In this section, we present the details of the proposed DPP and EDPP screening rules for the Lasso problem. We first review some basics of the dual problem of Lasso including its geometric properties in Section \ref{subsection:fundamental_DPP}; we also briefly discuss some basic guidelines for developing safe screening rules for Lasso. Based on the geometric properties discussed in Section \ref{subsection:fundamental_DPP}, we then develop the basic DPP screening rule in Section \ref{subsection:DPP}.
As a straightforward extension in dealing with the model selection problems, we also develop the sequential version of DPP rules. In Section \ref{subsection:EDPP}, by exploiting more geometric properties of the dual problem of Lasso, we further improve the DPP rules by developing the so called {\it enhanced DPP} (EDPP) rules. The EDPP screening rules significantly outperform DPP rules in identifying the inactive features for the Lasso problem.

\subsection{Basics}\label{subsection:fundamental_DPP}

Different from \cite{Xiang2011,Xiang2012}, we do not assume ${\bf y}$ and all ${\bf x}_i$ have unit length.The dual problem of problem (\ref{prob:Lasso_primal}) takes the form of (to make the paper self-contained, we provide the detailed derivation of the dual form in the appendix):
\begin{align}\label{prob:Lasso_dual}
	\sup_{\theta}\quad \left\{\frac{1}{2}\|{\bf y}\|_2^2 - \frac{\lambda^2}{2}\left\|\theta - \frac{{\bf y}}{\lambda}\right\|_2^2:\,\, |{\bf x}_i^{T}\theta|\leq 1,\,i=1,2,\ldots,p\right\},
\end{align}
where $\theta$ is the dual variable. For notational convenience, let the optimal solution of problem (\ref{prob:Lasso_dual}) be $\theta^{\ast}(\lambda)$ [recall that the optimal solution of problem (\ref{prob:Lasso_primal}) with parameter $\lambda$ is denoted by $\beta^*(\lambda)$].
Then, the KKT conditions are given by:
\begin{align}\label{eqn:Lasso_KKT1}
	{\bf y}&={\bf X}{\beta}^{\ast}(\lambda)+\lambda\theta^{\ast}(\lambda),\\
	\label{eqn:Lasso_KKT2}
	{\bf x}_i^T\theta^*(\lambda)&\in
	\begin{cases}
		{\sign([\beta^{\ast}(\lambda)]_i)},\hspace{2mm}\rm{if}\hspace{1mm}[\beta^{\ast}(\lambda)]_i\neq 0,  \\
		[-1, 1],\hspace{14mm}\rm{if}\hspace{1mm}[\beta^{\ast}(\lambda)]_i=0,   \\
	\end{cases}i = 1,\ldots,p,
\end{align}
where $[\cdot]_k$ denotes the $k^{th}$ component.
In view of the KKT condition in (\ref{eqn:Lasso_KKT2}), we have
\begin{align}\tag{R1}\label{rule1}
	|{\bf x}_i^T(\theta^*(\lambda))^T|<1\Rightarrow [\beta^*(\lambda)]_i=0\Rightarrow {\bf x}_i\,\, \mbox{\rm is an inactive feature.}
\end{align}
In other words, we can potentially make use of (\ref{rule1}) to identify the inactive features for the Lasso problem.
However, since $\theta^*(\lambda)$ is generally unknown, we can not directly apply (\ref{rule1}) to identify the inactive features. Inspired by the SAFE rules \citep{Ghaoui2012}, we can first estimate a region ${\bf \Theta}$ which contains $\theta^*(\lambda'')$. Then, (\ref{rule1}) can be relaxed as follows:
\begin{align}\tag{R1'}\label{rule1'}
	\sup_{\theta\in{\bf \Theta}}|{\bf x}_i^T\theta|<1\Rightarrow [\beta^*(\lambda)]_i=0\Rightarrow {\bf x}_i\,\, \mbox{\rm is an inactive feature.}
\end{align}
Clearly, as long as we can find a region ${\bf \Theta}$ which contains $\theta^*(\lambda)$, (\ref{rule1'}) will lead to a screening rule to detect the inactive features for the Lasso problem. Moreover, in view of (\ref{rule1}) and (\ref{rule1'}), we can see that the smaller the region ${\bf \Theta}$ is, the more accurate the estimation of $\theta^*(\lambda)$ is. As a result, more inactive features can be identified by the resulting screening rules.


{\bf Geometric Interpretations of the Dual Problem} By a closer look at the dual problem (\ref{prob:Lasso_dual}), we can observe that the dual optimal solution is the feasible point which is closest to ${\bf y}/{\lambda}$. For notational convenience, let the feasible set of problem (\ref{prob:Lasso_dual}) be $F$. Clearly, $F$ is the intersection of $2p$ closed half-spaces, and thus a closed and convex polytope. (Notice that, $F$ is also nonempty since $0\in F$.) In other words, $\theta^{\ast}(\lambda)$ is the projection of ${\bf y}/{\lambda}$ onto the polytope $F$. Mathematically, for an arbitrary vector ${\bf w}$ and a convex set $C$ in a Hilbert space $\mathcal{H}$, let us define the projection operator as
\begin{equation}\label{eqn:projection_def}
	P_C({\bf w})=\argmin_{{\bf u}\in C}\|{\bf u}-{\bf w}\|_2.
\end{equation}
Then, the dual optimal solution $\theta^*(\lambda)$ can be expressed by
\begin{equation}\label{eqn:Lasso_dual_projection}
	\theta^{\ast}(\lambda)=P_F({\bf y}/{\lambda})=\argmin_{\theta\in F}\left\|\theta-\frac{\bf y}{\lambda}\right\|_2.
\end{equation}

Indeed, the nice property of problem (\ref{prob:Lasso_dual}) illustrated by \eqref{eqn:Lasso_dual_projection} leads to many interesting results. For example, it is easy to see that ${\bf y}/\lambda$ would be an {\it interior point} of $F$ when $\lambda$ is large enough. If this is the case, we immediately have the following assertions: 1) ${\bf y}/\lambda$ is an interior point of $F$ implies that none of the constraints of problem (\ref{prob:Lasso_dual}) would be {\it active} on ${\bf y}/\lambda$, i.e., $|{\bf x}_i^T({\bf y}/(\lambda)|)<1$ for all $i=1,\ldots,p$; 2) $\theta^*(\lambda)$ is an interior point of $F$ as well since $\theta^*(\lambda)=P_{F}({\bf y}/\lambda)={\bf y}/\lambda$ by \eqref{eqn:Lasso_dual_projection} and the fact ${\bf y}/\lambda\in F$. Combining the results in 1) and 2), it is easy to see that $|{\bf x}_i^T\theta^*(\lambda)|<1$ for all $i=1,\ldots,p$. By (\ref{rule1}), we can conclude that $\beta^*(\lambda)=0$, under the assumption that $\lambda$ is large enough.

The above analysis may naturally lead to a question: does there exist a specific parameter value $\lambda_{\rm max}$ such that the optimal solution of problem (\ref{prob:Lasso_primal}) is $0$ whenever $\lambda>\lambda_{\rm max}$? The answer is affirmative. Indeed, let us define
\begin{align}\label{eqn:Lasso_lambdamx}
	\lambda_{\rm max}=\max_{i}|{\bf x}_i^T{\bf y}|.
\end{align}
It is well known \citep{Tibshirani11} that $\lambda_{\rm max}$ defined by \eqref{eqn:Lasso_lambdamx} is the smallest parameter such that problem (\ref{prob:Lasso_primal}) has a trivial solution, i.e.,
\begin{align}\label{eqn:Lasso_beta_0}
	\beta^*(\lambda)=0,\hspace{2mm}\forall\hspace{1mm}\lambda\in[\lambda_{\rm max},\infty).
\end{align}
Combining the results in (\ref{eqn:Lasso_beta_0}) and \eqref{eqn:Lasso_KKT1}, we immediately have
\begin{align}\label{eqn:Lasso_theta_closed_form}
	\theta^*(\lambda)=\frac{\bf y}{\lambda},\hspace{2mm}\forall\hspace{1mm}\lambda\in[\lambda_{\rm max},\infty).
\end{align}
Therefore, through out the rest of this paper, we will focus on the cases with $\lambda\in(0,\lambda_{\rm max})$.

\vspace{2mm}
In the subsequent sections, we will follow (\ref{rule1'}) to develop our screening rules. More specifically, the derivation of the proposed screening rules can be divided into the following three steps:
\setlist[enumerate,1]{leftmargin=1.6cm}
\begin{enumerate}
	\item[{\bf Step 1.}] We first estimate a region ${\bf \Theta}$ which contains the dual optimal solution $\theta^*(\lambda)$.
	\item[{\bf Step 2.}] We solve the maximization problem in (\ref{rule1'}), i.e., $\sup_{\theta\in{\bf \Theta}}|{\bf x}_i^T\theta|$.
	\item[{\bf Step 3.}] By plugging in the upper bound we find in {\bf Step} ${\bf 2}$, it is straightforward to develop the screening rule based on (\ref{rule1'}).
\end{enumerate}
The geometric property of the dual problem illustrated by \eqref{eqn:Lasso_dual_projection} serves as a fundamentally important role in developing our DPP and EDPP screening rules.

\subsection{Fundamental Screening Rules via Dual Polytope Projections (DPP)}\label{subsection:DPP}

In this Section, we propose the so called DPP screening rules for discarding the inactive features for Lasso. As the name indicates, the idea of DPP heavily relies on the properties of projection operators, e.g., the {\it nonexpansiveness} \citep{Bertsekas03}. We will follow the three steps stated in Section \ref{subsection:fundamental_DPP} to develop the DPP screening rules.

First, we need to find a region ${\bf \Theta}$ which contains the dual optimal solution $\theta^*(\lambda)$. Indeed, the result in (\ref{eqn:Lasso_theta_closed_form}) provides us an important clue. That is, we may be able to estimate a possible region for $\theta^*(\lambda)$ in terms of a known $\theta^*(\lambda_0)$ with $\lambda<\lambda_0$. Notice that, we can always set $\lambda_0=\lambda_{\rm max}$ and make use of the fact that $\theta^*(\lambda_{\rm max})={\bf y}/\lambda_{\rm max}$ implied by (\ref{eqn:Lasso_theta_closed_form}).
Another key ingredient comes from \eqref{eqn:Lasso_dual_projection}, i.e., the dual optimal solution $\theta^*(\lambda)$ is the projection of ${\bf y}/\lambda$ onto the feasible set $F$, which is nonempty closed and convex. A nice property of the projection operators defined in a Hilbert space with respect to a nonempty closed and convex set is the so called {\it nonexpansiveness}. For convenience, we restate the definition of nonexpansiveness in the following theorem.

\begin{theorem}\label{thm:projection_nonexpan}
	Let $C$ be a nonempty closed convex subset of a Hilbert space $\mathcal{H}$. Then the projection operator defined in \eqref{eqn:projection_def} is continuous and nonexpansive, i.e.,
	\begin{equation}\label{def:nonexpan}
		\|P_C({\bf w}_2)-P_C({\bf w}_1)\|_2\leq\|{\bf w}_2-{\bf w}_1\|_2,\,\,\forall {\bf w}_2, {\bf w}_1\in \mathcal{H}.
	\end{equation}
\end{theorem}

In view of \eqref{eqn:Lasso_dual_projection}, a direct application of Theorem \ref{thm:projection_nonexpan} leads to the following result:

\begin{theorem}\label{thm:DPP_estimation}
	For the Lasso problem, let $\lambda,\lambda_0>0$ be two regularization parameters. Then,
	\begin{align}\label{equation:projection_nonexpansion}
		\|\theta^*(\lambda)-\theta^*(\lambda_0)\|_2\leq \left|\frac{1}{\lambda}-\frac{1}{\lambda_0}\right|\|{\bf y}\|_2.
	\end{align}
\end{theorem}

For notational convenience, let a ball centered at ${\bf c}$ with radius $\rho$ be denoted by $B({\bf c},\rho)$. Theorem \ref{thm:DPP_estimation} actually implies that the dual optimal solution must be inside a ball centered at $\theta^*(\lambda_0)$ with radius $\left|1/\lambda-1/\lambda_0\right|\|{\bf y}\|_2$, i.e.,
\begin{align}\label{eqn:Lasso_DPP_ball}
	\theta^*(\lambda)\in B\left(\theta^*(\lambda_0),\left|\frac{1}{\lambda}-\frac{1}{\lambda_0}\right|\|{\bf y}\|_2\right).
\end{align}
We thus complete the first step for developing DPP. Because it is easy to find the upper bound of a linear functional over a ball, we combine the remaining two steps as follows.

\begin{theorem}\label{thm:Lasso_DPP}
	For the Lasso problem, assume we are given the solution of its dual problem $\theta^{\ast}(\lambda_0)$ for a specific $\lambda_0$. Let $\lambda$ be a positive value different from $\lambda_0$. Then $[\beta^{\ast}(\lambda)]_i=0$ if
	\begin{align}\label{equation:DPP}
		\left|{\bf x}_i^T\theta^{\ast}(\lambda)\right|<1-\|{\bf x}_i\|_2\|{\bf y}\|_2\left|\frac{1}{\lambda}-\frac{1}{\lambda_0}\right|.
	\end{align}
\end{theorem}
\begin{proof}
	The dual optimal solution $\theta^*(\lambda)$ is estimated to be inside the ball given by \eqref{eqn:Lasso_DPP_ball}. To simplify notations, let ${\bf c}=\theta^{\ast}(\lambda_0)$ and $\rho=\left|1/\lambda-1/\lambda_0\right|\|{\bf y}\|_2$. To develop a screening rule based on (\ref{rule1'}), we need to solve the optimization problem:
	$\sup_{\theta\in B({\bf c},\rho)}|{\bf x}_i^T\theta|$.
	
	Indeed, for any $\theta\in B({\bf c},\rho)$, it can be expressed by:
	\begin{align*}
		\theta = \theta^*(\lambda_0)+{\bf v},\hspace{2mm}\|{\bf v}\|_2\leq \rho.
	\end{align*}
	Therefore, the optimization problem can be easily solved as follows:
	\begin{align}\label{eqn:Lasso_DPP_upper_bound}
		\sup_{\theta\in B({\bf c},\rho)}\left|{\bf x}_i^T\theta\right|=\sup_{\|{\bf v}\|_2\leq\rho}\left|{\bf x}_i^T\left(\theta^*(\lambda_0)+{\bf v}\right)\right|=\left|{\bf x}_i^T\theta^*(\lambda_0)\right|+\rho\|{\bf x}_i\|_2.
	\end{align}
	By plugging the upper bound in \eqref{eqn:Lasso_DPP_upper_bound} to (\ref{rule1'}), we obtain the statement in Theorem \ref{thm:Lasso_DPP}, which completes the proof.
\end{proof}

Theorem \ref{thm:Lasso_DPP} implies that we can develop applicable screening rules for Lasso as long as the dual optimal solution $\theta^{\ast}(\cdot)$ is known for a certain parameter value $\lambda_0$. By simply setting $\lambda_0=\lambda_{\rm max}$ and noting that $\theta^*(\lambda_{\rm max})={\bf y}/\lambda_{\rm max}$ [please refer to \eqref{eqn:Lasso_theta_closed_form}], Theorem \ref{thm:Lasso_DPP} immediately leads to the following result.

\begin{corollary}\label{corollary:DPP}
	{\bf Basic DPP}: For the Lasso problem (\ref{prob:Lasso_primal}), let $\lambda_{\rm max}=\max_{i}|{\bf x}_i^T{\bf y}|$. If $\lambda\geq\lambda_{\rm max}$, then $[\beta^{\ast}]_i=0,\forall i\in \mathcal{I}$. Otherwise, $[\beta^{\ast}(\lambda)]_i=0$ if
	$$\left|{\bf x}_i^T\frac{\bf y}{\lambda_{\rm max}}\right|<1-\left(\frac{1}{\lambda}-\frac{1}{\lambda_{\rm max}}\right)\|{\bf x}_i\|_2 \|{\bf y}\|_2.$$
\end{corollary}



\begin{remark}
	Notice that, DPP is not the same as ST1 \cite{Xiang2011} and SAFE \cite{Ghaoui2012}, which discards the $i^{th}$ feature if
	\begin{align}\label{rule:SAFE}
		|{\bf x}_i^T{\bf y}|<\lambda-\|{\bf x}_i\|_2\|{\bf y}\|_2\frac{\lambda_{\rm max}-\lambda}{\lambda_{\rm max}}.
	\end{align}
	From the perspective of the sphere test, the radius of ST1/SAFE and DPP are the same. But the centers of ST1 and DPP are ${\bf y}/\lambda$ and ${\bf y}/\lambda_{\rm max}$ respectively, which leads to different formulas, i.e., \eqref{rule:SAFE} and Corollary \ref{corollary:DPP}.
\end{remark}

In real applications, the optimal parameter value of $\lambda$ is generally unknown and needs to be estimated. To determine an appropriate value of $\lambda$, commonly used approaches such as cross validation and stability selection involve solving the Lasso problem over a grid of tuning parameters $\lambda_1>\lambda_2>\ldots>\lambda_{\mathcal{K}}$, which is very time consuming. Motivated by the ideas of \cite{Tibshirani11} and \cite{Ghaoui2012}, we develop a sequential version of DPP rules. We first apply the DPP screening rule in Corollary \ref{corollary:DPP} to discard inactive features for the Lasso problem (\ref{prob:Lasso_primal}) with parameter being $\lambda_1$. After solving the reduced optimization problem for $\lambda_1$, we obtain the exact solution $\beta^{\ast}(\lambda_1)$. Hence by \eqref{eqn:Lasso_KKT1}, we can find $\theta^{\ast}(\lambda_1)$. According to Theorem \ref{thm:Lasso_DPP}, once we know the optimal dual solution $\theta^{\ast}(\lambda_1)$, we can construct a new screening rule by setting $\lambda_0=\lambda_1$ to identify inactive features for problem (\ref{prob:Lasso_primal}) with parameter being $\lambda_2$. By repeating the above process, we obtain the sequential version of the DPP rule as in the following corollary.
\begin{corollary}\label{corollary:DPPs}
	{\bf Sequential DPP}: For the Lasso problem (\ref{prob:Lasso_primal}), suppose we are given a sequence of parameter values $\lambda_{\rm max}=\lambda_0>\lambda_1>\ldots>\lambda_m$. Then for any integer $0\leq k<m$, we have $[\beta^{\ast}(\lambda_{k+1})]_i=0$ if $\beta^{\ast}(\lambda_k)$ is known and the following holds:
	$$
	\left|{\bf x}_i^T\frac{{\bf y}-{\bf X}\beta^{\ast}(\lambda_k)}{\lambda_{k}}\right|<1-\left(\frac{1}{\lambda_{k+1}}-\frac{1}{\lambda_{k}}\right)\|{\bf x}_i\|_2\|{\bf y}\|_2.
	$$
\end{corollary}
\begin{remark}
	From Corollaries \ref{corollary:DPP} and \ref{corollary:DPPs}, we can see that both of the DPP and sequential DPP rules discard the inactive features for the Lasso problem with a smaller parameter value by assuming a known dual optimal solution at a larger parameter value. This is in fact a standard way to construct screening rules for Lasso \citep{Tibshirani11,Ghaoui2012,Xiang2011,Xiang2012}.
\end{remark}

\begin{remark}
	For illustration purpose, we present both the basic and sequential version of the DPP screening rules. However, it is easy to see that the basic DPP rule can be easily derived from its sequential version by simply setting $\lambda_k=\lambda_{\rm max}$ and $\lambda_{k+1}=\lambda$. Therefore, in this paper, we will focus on the development and evaluation of the sequential version of the proposed screening rules. To avoid any confusions, DPP and EDPP all refer to the corresponding sequential versions.
\end{remark}

\subsection{Enhanced DPP Rules for Lasso}\label{subsection:EDPP}

In this section, we further improve the DPP rules presented in Section \ref{subsection:DPP} by a more careful analysis of the projection operators. Indeed, from the three steps by which we develop the DPP rules, we can see that the first step is a key. In other words, the estimation of the dual optimal solution serves as a fundamentally important role in developing the DPP rules. Moreover, (\ref{rule1'}) implies that the more accurate the estimation is, the more effective the resulting screening rule is in discarding the inactive features. The estimation of the dual optimal solution in DPP rules is in fact a direct consequence of the nonexpansiveness of the projection operators. Therefore, in order to improve the performance of the DPP rules in discarding the inactive features, we propose two different approaches to find more accurate estimations of the dual optimal solution. These two approaches are presented in detail in Sections \ref{subsubsection:Improvement1} and \ref{subsubsection:Improvement2} respectively. By combining the ideas of these two approaches, we can further improve the estimation of the dual optimal solution. Based on this estimation, we develop the enhanced DPP rules (EDPP) in Section \ref{subsubsection:EDPP}. Again, we will follow the three steps in Section \ref{subsection:fundamental_DPP} to develop the proposed screening rules.

\subsubsection{Improving the DPP rules via Projections of Rays}\label{subsubsection:Improvement1}

In the DPP screening rules, the dual optimal solution $\theta^*(\lambda)$ is estimated to be inside the ball $B\left(\theta^*(\lambda_0),|1/\lambda-1/\lambda_0|\|{\bf y}\|_2\right)$ with $\theta^*(\lambda_0)$ given. In this section, we show that $\theta^*(\lambda)$ lies inside a ball centered at $\theta^*(\lambda_0)$ with a smaller radius.

Indeed, it is well known that the projection of an arbitrary point onto a nonempty closed convex set $C$ in a Hilbert space $\mathcal{H}$ always exists and is unique \citep{Bauschke2011}. However, the converse is not true, i.e., there may exist ${\bf w}_1, {\bf w}_2\in\mathcal{H}$ such that ${\bf w}_1\neq{\bf w}_2$ and $P_C({\bf w}_1)=P_C({\bf w}_2)$. In fact, it is known that the following result holds:
\begin{lemma}\label{lemma:projection_ray}
	\citep{Bauschke2011} Let $C$ be a nonempty closed convex subset of a Hilbert space $\mathcal{H}$. For a point ${\bf w}\in\mathcal{H}$, let ${\bf w}(t)=P_C({\bf w})+t({\bf w}-P_C({\bf w}))$. Then, the projection of the point ${\bf w}(t)$ is $P_C({\bf w})$ for all $t\geq0$, i.e.,
	\begin{align}
		P_C({\bf w}(t))=P_C({\bf w}), \forall t\geq0.
	\end{align}
\end{lemma}

Clearly, when ${\bf w}\neq P_C({\bf w})$, i.e., ${\bf w}\notin C$, ${\bf w}(t)$ with $t\geq0$ is the ray starting from $P_C({\bf w})$ and pointing in the same direction as ${\bf w}-P_C({\bf w})$. By Lemma \ref{lemma:projection_ray}, we know that the projection of the ray ${\bf w}(t)$ with $t\geq0$ onto the set $C$ is a single point $P_C({\bf w})$. [When ${\bf w}= P_C({\bf w})$, i.e., ${\bf w}\in C$, ${\bf w}(t)$ with $t\geq0$ becomes a single point and the statement in Lemma \ref{lemma:projection_ray} is trivial.]

By making use of Lemma \ref{lemma:projection_ray} and the nonexpansiveness of the projection operators, we can improve the estimation of the dual optimal solution in DPP [please refer to Theorem \ref{thm:DPP_estimation} and \eqref{eqn:Lasso_DPP_ball}]. More specifically, we have the following result:

\begin{theorem}\label{thm:Lasso_improve1_estimation}
	For the Lasso problem, suppose the dual optimal solution $\theta^*(\cdot)$ at $\lambda_0\in(0,\lambda_{\rm max}]$ is known. For any $\lambda\in(0,\lambda_0]$,  let us define
	\begin{align}\label{eqn:Lasso_v1}
		{\bf v}_1(\lambda_0)=
		\begin{cases}
			\frac{\bf y}{\lambda_0}-\theta^*(\lambda_0),\hspace{4.5mm}{\rm if}\hspace{2mm}\lambda_0\in(0,\lambda_{\rm max}),\\
			{\rm sign}({\bf x}_*^T{\bf y}){\bf x}_*,\hspace{3mm}{\rm if}\hspace{2mm}\lambda_0=\lambda_{\rm max},
		\end{cases}
		{\rm where} \hspace{2mm}{\bf x}_*={\rm argmax}_{{\bf x}_i}|{\bf x}_i^T{\bf y}|,
	\end{align}
	\begin{align}\label{eqn:Lasso_v2}
		{\bf v}_2(\lambda,\lambda_0)=\frac{\bf y}{\lambda}-\theta^*(\lambda_0),
	\end{align}
	\begin{align}\label{eqn:Lasso_v2_perp}
		{\bf v}_2^{\perp}(\lambda,\lambda_0)={\bf v}_2(\lambda,\lambda_0)-\frac{\langle{\bf v}_1(\lambda_0),{\bf v}_2(\lambda,\lambda_0)\rangle}{\|{\bf v}_1(\lambda_0)\|_2^2}{\bf v}_1(\lambda_0).
	\end{align}
	Then, the dual optimal solution $\theta^*(\lambda)$ can be estimated as follows:
	\begin{align}
		\theta^*(\lambda)\in B\left(\theta^*(\lambda_0),\|{\bf v}_2^{\perp}(\lambda,\lambda_0)\|_2\right)\subseteq B\left(\theta^*(\lambda_0),\left|\frac{1}{\lambda}-\frac{1}{\lambda_0}\right|\|{\bf y}\|_2\right).
	\end{align}
\end{theorem}
\begin{proof}
	By making use of Lemma \ref{lemma:projection_ray}, we present the proof of the statement for the cases with $\lambda_0\in(0,\lambda_{\rm max})$. We postpone the proof of the statement for the case with $\lambda_0=\lambda_{\rm max}$ after we introduce more general technical results.
	
	In view of the assumption $\lambda_0\in(0,\lambda_{\rm max})$, it is easy to see that \begin{align}\label{eqn:Lasso_projection_infeasible}
		\frac{\bf y}{\lambda_0}\notin F\Rightarrow\frac{\bf y}{\lambda_0}\neq P_F\left(\frac{\bf y}{\lambda_0}\right)=\theta^*(\lambda_0)\Rightarrow \frac{\bf y}{\lambda_0}-\theta^*(\lambda_0)\neq0.
	\end{align}
	For each $\lambda_0\in(0,\lambda_{\rm max})$, let us define
	\begin{align}\label{eqn:ray}
		\theta_{\lambda_0}(t)=\theta^*(\lambda_0)+t{\bf v}_1(\lambda_0)=\theta^*(\lambda_0)+t\left(\frac{\bf y}{\lambda_0}-\theta^*(\lambda_0)\right),\hspace{2mm}t\geq0.
	\end{align}
	By the result in (\ref{eqn:Lasso_projection_infeasible}), we can see that $\theta_{\lambda_0}(\cdot)$ defined by \eqref{eqn:ray} is a ray which starts at $\theta^*(\lambda_0)$ and points in the same direction as ${\bf y}/\lambda_0-\theta^*(\lambda_0)$. In view of \eqref{eqn:Lasso_dual_projection}, a direct application of Lemma \ref{lemma:projection_ray} leads to that:
	\begin{align}\label{eqn:Lasso_projection_ray}
		P_{F}(\theta_{\lambda_0}(t))=\theta^*(\lambda_0),\hspace{2mm}\forall\hspace{1mm}t\geq0.
	\end{align}
	By applying Theorem \ref{thm:projection_nonexpan} again, we have
	\begin{align}\label{ineqn:nonexp_g}
		\|\theta^*(\lambda)-\theta^*(\lambda_0)\|_2&=\left\|P_F\left(\frac{\bf y}{\lambda}\right)-P_F(\theta_{\lambda_0}(t))\right\|_2\\ \nonumber
		&\leq\left\|\frac{\bf y}{\lambda}-\theta_{\lambda_0}(t)\right\|_2=\left\|t\left(\frac{\bf y}{\lambda_0}-\theta^*(\lambda_0)\right)-\left(\frac{\bf y}{\lambda}-\theta^*(\lambda_0)\right)\right\|_2\\ \nonumber
		&=\|t{\bf v}_1(\lambda_0)-{\bf v}_2(\lambda,\lambda_0)\|_2,\hspace{2mm}\forall\hspace{1mm}t\geq0.
	\end{align}
	Because the inequality in (\ref{ineqn:nonexp_g}) holds for all $t\geq0$, it is easy to see that
	\begin{align}\label{ineqn:Lasso_improve1_ball}
		\|\theta^*(\lambda)-\theta^*(\lambda_0)\|_2&\leq\min_{t\geq0}\,\,\|t{\bf v}_1(\lambda_0)-{\bf v}_2(\lambda,\lambda_0)\|_2\\ \nonumber
		&=
		\begin{cases}
			\|{\bf v}_2(\lambda,\lambda_0)\|_2,\hspace{5mm}\mbox{if }\langle{\bf v}_1(\lambda_0),{\bf v}_2(\lambda,\lambda_0)\rangle<0,\\
			\left\|{\bf v}_2^{\perp}(\lambda,\lambda_0)\right\|_2,\hspace{3mm}\mbox{otherwise}.
		\end{cases}
	\end{align}
	The inequality in (\ref{ineqn:Lasso_improve1_ball}) implies that, to prove the first half of the statement, i.e., $$\theta^*(\lambda)\in B(\theta^*(\lambda_0),\|{\bf v}_2^{\perp}(\lambda,\lambda_0)\|_2),$$ we only need to show that $\langle{\bf v}_1(\lambda_0),{\bf v}_2(\lambda,\lambda_0)\rangle\geq0$.
	
	Indeed, it is easy to see that $0\in F$. Therefore, in view of \eqref{eqn:Lasso_projection_ray}, the distance between $\theta_{\lambda_0}(t)$ and $\theta^*(\lambda_0)$ must be shorter than the one between $\theta_{\lambda_0}(t)$ and $0$ for all $t\geq0$, i.e.,
	\begin{align}\label{ineqn:Lasso_projection_shrink}
		&\|\theta_{\lambda_0}(t)-\theta^*(\lambda_0)\|_2^2\leq \|\theta_{\lambda_0}(t)-0\|_2^2\\ \nonumber
		\Rightarrow&\hspace{2mm} 0\leq \|\theta^*(\lambda_0)\|_2^2+2t\left(\left\langle\theta^*(\lambda_0),\frac{\bf y}{\lambda_0}\right\rangle-\|\theta^*(\lambda_0)\|_2^2\right),\hspace{2mm}\forall\hspace{1mm}t\geq0.
	\end{align}
	Since the inequality in (\ref{ineqn:Lasso_projection_shrink}) holds for all $t\geq0$, we can conclude that:
	\begin{align}\label{ineqn:Lasso_projection_shrink1}
		\left\langle\theta^*(\lambda_0),\frac{\bf y}{\lambda_0}\right\rangle-\|\theta^*(\lambda_0)\|_2^2\geq0\Rightarrow \frac{\|{\bf y}\|_2}{\lambda_0}\geq \|\theta^*(\lambda_0)\|_2.
	\end{align}
	Therefore, we can see that:
	\begin{align}\label{ineqn:v1v2_ip2}
		\langle{\bf v}_1(\lambda_0),{\bf v}_2(\lambda,\lambda_0)\rangle&=\left\langle\frac{\bf y}{\lambda_0}-\theta^*(\lambda_0),\frac{\bf y}{\lambda}-\frac{\bf y}{\lambda_0}+\frac{\bf y}{\lambda_0}-\theta^*(\lambda_0)\right\rangle\\ \nonumber
		&\geq\left(\frac{1}{\lambda}-\frac{1}{\lambda_0}\right)\left\langle\frac{\bf y}{\lambda_0}-\theta^*(\lambda_0),{\bf y}\right\rangle\\ \nonumber
		&=\left(\frac{1}{\lambda}-\frac{1}{\lambda_0}\right)\left(\frac{\|{\bf y}\|_2^2}{\lambda_0}-\langle\theta^*(\lambda_0),{\bf y}\rangle\right)\\ \nonumber
		&\geq\left(\frac{1}{\lambda}-\frac{1}{\lambda_0}\right)\left(\frac{\|{\bf y}\|_2^2}{\lambda_0}-\|\theta^*(\lambda_0)\|_2\|{\bf y}\|_2\right)\geq0.
	\end{align}
	The last inequality in (\ref{ineqn:v1v2_ip2}) is due to the result in (\ref{ineqn:Lasso_projection_shrink1}).
	
	Clearly, in view of (\ref{ineqn:Lasso_improve1_ball}) and (\ref{ineqn:v1v2_ip2}), we can see that the first half of the statement holds, i.e., $\theta^*(\lambda)\in B(\theta^*(\lambda_0),\|{\bf v}_2^{\perp}(\lambda,\lambda_0)\|_2)$. The second half of the statement, i.e., $B(\theta^*(\lambda_0),\|{\bf v}_2^{\perp}(\lambda,\lambda_0)\|_2)\subseteq B(\theta^*(\lambda_0),|1/\lambda-1/\lambda_0|\|{\bf y}\|_2)$, can be easily obtained by noting that the inequality in (\ref{ineqn:nonexp_g}) reduces to the one in (\ref{eqn:Lasso_DPP_ball}) when $t=1$. This completes the proof of the statement with $\lambda_0\in(0,\lambda_{\rm max})$.
\end{proof}

Before we present the proof of Theorem \ref{thm:Lasso_improve1_estimation} for the case with $\lambda_0=\lambda_{\rm max}$, let us briefly review some technical results from convex analysis first.
\begin{definition}
	\cite{Ruszczynski2006} Let $C$ be a nonempty closed convex subset of a Hilbert space $\mathcal{H}$ and ${\bf w}\in C$. The set
	\begin{align}
		N_C({\bf w}):=\{{\bf v}:\langle{\bf v},{\bf u}-{\bf w}\rangle\leq0,\forall {\bf u}\in C\}
	\end{align}
	is called the normal cone to $C$ at ${\bf w}$.
\end{definition}

In terms of the normal cones, the following theorem provides an elegant and useful characterization of the projections onto nonempty closed convex subsets of a Hilbert space.
\begin{theorem}\label{thm:projection_normal_cone}
	\citep{Bauschke2011}
	Let $C$ be a nonempty closed convex subset of a Hilbert space $\mathcal{H}$. Then, for every ${\bf w}\in \mathcal{H}$ and ${\bf w}_0\in C$, ${\bf w}_0$ is the projection of ${\bf w}$ onto $C$ if and only if ${\bf w}-{\bf w}_0\in N_C({\bf w}_0)$, i.e.,
	\begin{align}
		{\bf w}_0=P_C({\bf w})\Leftrightarrow \langle{\bf w}-{\bf w}_0,{\bf u}-{\bf w}_0\rangle\leq0,\forall {\bf u}\in C.
	\end{align}
\end{theorem}

In view of the proof of Theorem \ref{thm:Lasso_improve1_estimation}, we can see that \eqref{eqn:Lasso_projection_ray} is a key step. When $\lambda_0=\lambda_{\rm max}$, similar to \eqref{eqn:ray}, let us define
\begin{align}\label{eqn:Lasso_ray_lambdamx}
	\theta_{\lambda_{\rm max}}(t)=\theta^*(\lambda_{\rm max})+t{\bf v}_1(\lambda_{\rm max}),\hspace{2mm}\forall\hspace{1mm}t\geq0.
\end{align}
By Theorem \ref{thm:projection_normal_cone}, the following lemma shows that \eqref{eqn:Lasso_projection_ray} also holds for $\lambda_0=\lambda_{\rm max}$.
\begin{lemma}\label{lemma:projection_ray_lambdamx}
	For the Lasso problem, let ${\bf v}_1(\cdot)$ and $\theta_{\lambda_{\rm max}}(\cdot)$ be given by \eqref{eqn:Lasso_v1} and \eqref{eqn:Lasso_ray_lambdamx}, then the following result holds:
	\begin{align}\label{eqn:Lasso_projection_ray_lambdamx}
		P_F(\theta_{\lambda_{\rm max}}(t))=\theta^*(\lambda_{\rm max}), \hspace{1mm} \forall\,\, t\geq0.
	\end{align}
\end{lemma}
\begin{proof}
	To prove the statement, Theorem \ref{thm:projection_normal_cone} implies that we only need to show:
	\begin{align}\label{ineqn:Lasso_ip_v1mx_nc}
		\langle{\bf v}_1(\lambda_{\rm max}),\theta-\theta^*(\lambda_{\rm max})\rangle\leq 0,\hspace{2mm}\forall\hspace{1mm}\theta\in F.
	\end{align}
	Recall that ${\bf v}_1(\lambda_{\rm max})={\rm sign}({\bf x}_*^T{\bf y}){\bf x}_*$, ${\bf x}_*={\rm argmax}_{{\bf x}_i}|{\bf x}_i^T{\bf y}|$ [\eqref{eqn:Lasso_v1}],
	and $\theta^*(\lambda_{\rm max})={\bf y}/\lambda_{\rm max}$ [\eqref{eqn:Lasso_theta_closed_form}].
	It is easy to see that
	\begin{align}\label{eqn:Lasso_ip_v1mx_thetamx}
		\langle{\bf v}_1(\lambda_{\rm max}),\theta^*(\lambda_{\rm max})\rangle=\left\langle{\rm sign}({\bf x}_*^T{\bf y}){\bf x}_*,\frac{\bf y}{\lambda_{\rm max}}\right\rangle=\frac{|{\bf x}_*^T{\bf y}|}{\lambda_{\rm max}}=1.
	\end{align}
	Moreover, assume $\theta$ is an arbitrary point of $F$. Then, we have $|\langle{\bf x}_*,\theta\rangle|\leq1$, and thus
	\begin{align}\label{ineqn:Lasso_ip_v1mx_F}
		\langle{\bf v}_1(\lambda_{\rm max}),\theta\rangle=\langle{\rm sign}({\bf x}_*^T{\bf y}){\bf x}_*,\theta\rangle\leq|\langle{\bf x}_*,\theta\rangle|\leq1.
	\end{align}
	Therefore, the inequality in (\ref{ineqn:Lasso_ip_v1mx_nc}) easily follows by combing the results in (\ref{eqn:Lasso_ip_v1mx_thetamx}) and (\ref{ineqn:Lasso_ip_v1mx_F}), which completes the proof.
\end{proof}

We are now ready to give the proof of Theorem \ref{thm:Lasso_improve1_estimation} for the case with $\lambda_0=\lambda_{\rm max}$.

\begin{proof}
	In view of Theorem \ref{thm:projection_nonexpan} and Lemma \ref{lemma:projection_ray_lambdamx}, we have
	\begin{align}\label{ineqn:Lasso_improve1_estimate_lambdamx}
		\|\theta^*(\lambda)-\theta^*(\lambda_{\rm max})\|_2&=\left\|P_F\left(\frac{\bf y}{\lambda}\right)-P_F(\theta_{\lambda_{\rm max}}(t))\right\|_2\\ \nonumber
		&\leq\left\|\frac{\bf y}{\lambda}-\theta_{\lambda_{\rm max}}(t)\right\|_2=\left\|t{\bf v}_1(\lambda_{\rm max})-\left(\frac{\bf y}{\lambda}-\theta^*(\lambda_{\rm max})\right)\right\|_2\\ \nonumber
		&=\|t{\bf v}_1(\lambda_{\rm max})-{\bf v}_2(\lambda,\lambda_{\rm max})\|_2,\hspace{2mm}\forall\hspace{1mm}t\geq0.
	\end{align}
	Because the inequality in (\ref{ineqn:Lasso_improve1_estimate_lambdamx}) holds for all $t\geq0$, we can see that
	\begin{align}\label{ineqn:Lasso_improve1_ball_lambdamx}
		\|\theta^*(\lambda)-\theta^*(\lambda_{\rm max})\|_2&\leq\min_{t\geq0}\,\,\|t{\bf v}_1(\lambda_{\rm max})-{\bf v}_2(\lambda,\lambda_{\rm max})\|_2\\ \nonumber
		&=
		\begin{cases}
			\|{\bf v}_2(\lambda,\lambda_{\rm max})\|_2,\hspace{5mm}\mbox{if }\langle{\bf v}_1(\lambda_{\rm max}),{\bf v}_2(\lambda,\lambda_{\rm max})\rangle<0,\\
			\left\|{\bf v}_2^{\perp}(\lambda,\lambda_{\rm max})\right\|_2,\hspace{3mm}\mbox{otherwise}.
		\end{cases}
	\end{align}
	Clearly, we only need to show that $\langle{\bf v}_1(\lambda_{\rm max}),{\bf v}_2(\lambda,\lambda_{\rm max})\rangle\geq0$.
	
	Indeed, Lemma \ref{lemma:projection_ray_lambdamx} implies that ${\bf v}_1(\lambda_{\rm max})\in N_F(\theta^*(\lambda_{\rm max}))$ [please refer to the inequality in (\ref{ineqn:Lasso_ip_v1mx_nc})]. By noting that $0\in F$, we have
	\begin{align}
		\left\langle{\bf v}_1(\lambda_{\rm max}),0-\frac{\bf y}{\lambda_{\rm max}}\right\rangle\leq 0\Rightarrow\langle{\bf v}_1(\lambda_{\rm max}),{\bf y}\rangle\geq0.
	\end{align}
	Moreover, because ${\bf y}/\lambda_{\rm max}=\theta^*(\lambda_{\rm max})$, it is easy to see that
	\begin{align}\label{ineqn:Lasso_improve1_v1mxv2}
		\langle{\bf v}_1(\lambda_{\rm max}),{\bf v}_2(\lambda,\lambda_{\rm max})\rangle&=\left\langle{\bf v}_1(\lambda_{\rm max}),\frac{\bf y}{\lambda}-\frac{\bf y}{\lambda_{\rm max}}\right\rangle\\ \nonumber
		&=\left(\frac{1}{\lambda}-\frac{1}{\lambda_{\rm max}}\right)\langle{\bf v}_1(\lambda_{\rm max}),{\bf y}\rangle\geq0.
	\end{align}
	
	Therefore, in view of (\ref{ineqn:Lasso_improve1_ball_lambdamx}) and (\ref{ineqn:Lasso_improve1_v1mxv2}), we can see that the first half of the statement holds, i.e., $\theta^*(\lambda)\in B(\theta^*(\lambda_{\rm max}),\|{\bf v}_2^{\perp}(\lambda,\lambda_{\rm max})\|_2)$. The second half of the statement, i.e., $$B(\theta^*(\lambda_{\rm max}),\|{\bf v}_2^{\perp}(\lambda,\lambda_{\rm max})\|_2)\subseteq B(\theta^*(\lambda_{\rm max}),|1/\lambda-1/\lambda_{\rm max}|\|{\bf y}\|_2),$$ can be easily obtained by noting that the inequality in (\ref{ineqn:Lasso_improve1_ball_lambdamx}) reduces to the one in (\ref{eqn:Lasso_DPP_ball}) when $t=0$. This completes the proof of the statement with $\lambda_0=\lambda_{\rm max}$. Thus, the proof of Theorem \ref{thm:Lasso_improve1_estimation} is completed.
\end{proof}

Theorem \ref{thm:Lasso_improve1_estimation} in fact provides a more accurate estimation of the dual optimal solution than the one in DPP, i.e., $\theta^*(\lambda)$ lies inside a ball centered at $\theta^*(\lambda_0)$ with a radius $\|{\bf v}_2^{\perp}(\lambda,\lambda_0)\|_2$. Based on this improved estimation and (\ref{rule1'}), we can develop the following screening rule to discard the inactive features for Lasso.
\begin{theorem}\label{thm:Lasso_improve1_fundamental_rule}
	For the Lasso problem, assume the dual optimal solution $\theta^{\ast}(\cdot)$ at $\lambda_0\in(0,\lambda_{\rm max}]$ is known. Then, for each $\lambda\in(0,\lambda_0)$, we have $[\beta^{\ast}(\lambda)]_i=0$ if
	\begin{align*}
		|{\bf x}_i^T\theta^{\ast}(\lambda_0)|<1-\|{\bf v}_2^{\perp}(\lambda,\lambda_0)\|_2\|{\bf x}_i\|_2.
	\end{align*}
\end{theorem}

We omit the proof of Theorem \ref{thm:Lasso_improve1_fundamental_rule} since it is very similar to the one of Theorem \ref{thm:Lasso_DPP}. By Theorem \ref{thm:Lasso_improve1_fundamental_rule}, we can easily develop the following sequential screening rule.

\vspace{5mm}
{\bf Improvement 1}: {\it For the Lasso problem (\ref{prob:Lasso_primal}), suppose we are given a sequence of parameter values $\lambda_{\rm max}=\lambda_0>\lambda_1>\ldots>\lambda_{\mathcal{K}}$. Then for any integer $0\leq k<\mathcal{K}$, we have $[\beta^{\ast}(\lambda_{k+1})]_i=0$ if $\beta^{\ast}(\lambda_k)$ is known and the following holds:
	$$\left|{\bf x}_i^T\frac{{\bf y}-{\bf X}\beta^{\ast}(\lambda_k)}{\lambda_{k}}\right|<1-\|{\bf v}_2^{\perp}(\lambda_{k+1},\lambda_k)\|_2\|{\bf x}_i\|_2.$$
}

The screening rule in Improvement 1 is developed based on (\ref{rule1'}) and the estimation of the dual optimal solution in Theorem \ref{thm:Lasso_improve1_estimation}, which is more accurate than the one in DPP. Therefore, in view of (\ref{rule1'}), the screening rule in Improvement 1 are more effective in discarding the inactive features than the DPP rule.


%

\subsubsection{Improving the DPP rules via Firmly Nonexpansiveness}\label{subsubsection:Improvement2}

In Section \ref{subsubsection:Improvement1}, we improve the estimation of the dual optimal solution in DPP by making use of the projections of properly chosen rays. (\ref{rule1'}) implies that the resulting screening rule stated in Improvement 1 is more effective in discarding the inactive features than DPP. In this Section, we present another approach to improve the estimation of the dual optimal solution in DPP by making use of the so called {\it firmly nonexpansiveness} of the projections onto nonempty closed convex subset of a Hilbert space.

\begin{theorem}\label{thm:firmly_nonexpansive}
	\citep{Bauschke2011} Let $C$ be a nonempty closed convex subset of a Hilbert space $\mathcal{H}$. Then the projection operator defined in \eqref{eqn:projection_def} is continuous and firmly nonexpansive. In other words, for any ${\bf w}_1, {\bf w}_2\in\mathcal{H}$, we have
	\begin{align}\label{def:firmly_nonexpan}
		\|P_C({\bf w}_1)-P_C({\bf w}_2)\|_2^2+\|({\rm Id}-P_C)({\bf w}_1)-({\rm Id}-P_C)({\bf w}_2)\|_2^2\leq\|{\bf w}_1-{\bf w}_2\|_2^2,
	\end{align}
	where ${\rm Id}$ is the identity operator.
\end{theorem}

In view of the inequalities in (\ref{def:firmly_nonexpan}) and (\ref{def:nonexpan}), it is easy to see that firmly nonexpansiveness implies nonexpansiveness. But the converse is not true. Therefore, firmly nonexpansiveness of the projection operators is a stronger property than the nonexpansiveness. A direct application of Theorem \ref{thm:firmly_nonexpansive} leads to the following result.

\begin{theorem}\label{thm:Lasso_improve2_estimate}
	For the Lasso problem, let $\lambda,\lambda_0>0$ be two parameter values. Then
	\begin{align}\label{eqn:Lasso_improve2_ball}
		\theta^*(\lambda)\in B\left(\theta^*(\lambda_0)+\frac{1}{2}\left(\frac{1}{\lambda}-\frac{1}{\lambda_0}\right){\bf y}, \frac{1}{2}\left|\frac{1}{\lambda}-\frac{1}{\lambda_0}\right|\|{\bf y}\|_2\right)\subset B\left(\theta^*(\lambda_0),\left|\frac{1}{\lambda}-\frac{1}{\lambda_0}\right|\|{\bf y}\|_2\right).
	\end{align}
\end{theorem}
\begin{proof}
	In view of \eqref{eqn:Lasso_dual_projection} and the firmly nonexpansiveness in (\ref{def:firmly_nonexpan}), we have
	\begin{align}\label{ineqn:Lasso_improve2_firm_nonexpan}
		&\|\theta^*(\lambda)-\theta^*(\lambda_0)\|_2^2+\left\|\left(\frac{\bf y}{\lambda}-\theta^*(\lambda)\right)-\left(\frac{\bf y}{\lambda_0}-\theta^*(\lambda_0)\right)\right\|_2^2\leq\left\|\frac{\bf y}{\lambda}-\frac{\bf y}{\lambda_0}\right\|_2^2\\ \nonumber
		\Leftrightarrow\hspace{2mm}&\|\theta^*(\lambda)-\theta^*(\lambda_0)\|_2^2\leq\left\langle\theta^*(\lambda)-\theta^*(\lambda_0),\frac{\bf y}{\lambda}-\frac{\bf y}{\lambda_0}\right\rangle\\ \nonumber
		\Leftrightarrow\hspace{2mm}&\left\|\theta^*(\lambda)-\left(\theta^*(\lambda_0)+\frac{1}{2}\left(\frac{1}{\lambda}-\frac{1}{\lambda_0}\right){\bf y}\right)\right\|_2\leq\frac{1}{2}\left|\frac{1}{\lambda}-\frac{1}{\lambda_0}\right|\|{\bf y}\|_2,
	\end{align}
	which completes the proof of the first half of the statement. The second half of the statement is trivial by noting that the first inequality in (\ref{ineqn:Lasso_improve2_firm_nonexpan}) (firmly nonexpansiveness) implies the inequality in (\ref{equation:projection_nonexpansion}) (nonexpansiveness) but not vice versa. Indeed, it is easy to see that the ball in the middle of (\ref{eqn:Lasso_improve2_ball}) is inside the right one and has only a half radius.
\end{proof}

Clearly, Theorem \ref{thm:Lasso_improve2_estimate} provides a more accurate estimation of the dual optimal solution than the one in DPP, i.e., the dual optimal solution must be inside a ball which is a subset of the one in DPP and has only a half radius. Again, based on the estimation in Theorem \ref{thm:Lasso_improve2_estimate} and (\ref{rule1'}), we have the following result.

\begin{theorem}\label{thm:Lasso_improve2_fundamental_rule}
	For the Lasso problem, assume the dual optimal solution $\theta^{\ast}(\cdot)$ at $\lambda_0\in(0,\lambda_{\rm max}]$ is known. Then, for each $\lambda\in(0,\lambda_0)$, we have $[\beta^{\ast}(\lambda)]_i=0$ if
	\begin{align*}
		\left|{\bf x}_i^T\left(\theta^*(\lambda_0)+\frac{1}{2}\left(\frac{1}{\lambda}-\frac{1}{\lambda_0}\right){\bf y}\right)\right|<1-\frac{1}{2}\left(\frac{1}{\lambda}-\frac{1}{\lambda_0}\right)\|{\bf y}\|_2\|{\bf x}_i\|_2.
	\end{align*}
\end{theorem}

We omit the proof of Theorem \ref{thm:Lasso_improve2_fundamental_rule} since it is very similar to the proof of Theorem \ref{thm:Lasso_DPP}. A direct application of Theorem \ref{thm:Lasso_improve2_fundamental_rule} leads to the following sequential screening rule.

\vspace{5mm}
{\bf Improvement 2}: {\it For the Lasso problem (\ref{prob:Lasso_primal}), suppose we are given a sequence of parameter values $\lambda_{\rm max}=\lambda_0>\lambda_1>\ldots>\lambda_{\mathcal{K}}$. Then for any integer $0\leq k<\mathcal{K}$, we have $[\beta^{\ast}(\lambda_{k+1})]_i=0$ if $\beta^{\ast}(\lambda_k)$ is known and the following holds:
	$$\left|{\bf x}_i^T\left(\frac{{\bf y}-{\bf X}\beta^{\ast}(\lambda_k)}{\lambda_{k}}+\frac{1}{2}\left(\frac{1}{\lambda_{k+1}}-\frac{1}{\lambda_k}\right){\bf y}\right)\right|<1-\frac{1}{2}\left(\frac{1}{\lambda_{k+1}}-\frac{1}{\lambda_k}\right)\|{\bf y}\|_2\|{\bf x}_i\|_2.$$
}

Because the screening rule in Improvement 2 is developed based on (\ref{rule1'}) and the estimation in Theorem \ref{thm:Lasso_improve2_estimate}, it is easy to see that Improvement 2 is more effective in discarding the inactive features than DPP.

%
%

\subsubsection{The Proposed Enhanced DPP Rules}\label{subsubsection:EDPP}

In Sections \ref{subsubsection:Improvement1} and \ref{subsubsection:Improvement2}, we present two different approaches to improve the estimation of the dual optimal solution in DPP. In view of (\ref{rule1'}), we can see that the resulting screening rules, i.e., Improvements 1 and 2, are more effective in discarding the inactive features than DPP. In this section, we give a more accurate estimation of the dual optimal solution than the ones in Theorems \ref{thm:Lasso_improve1_estimation} and \ref{thm:Lasso_improve2_estimate} by combining the aforementioned two approaches together. The resulting screening rule for Lasso is the so called enhanced DPP rule (EDPP). Again, (\ref{rule1'}) implies that EDPP is more effective in discarding the inactive features than the screening rules in Improvements 1 and 2. We also present several experiments to demonstrate that EDPP is able to identify more inactive features than the screening rules in Improvements 1 and 2. Therefore, in the subsequent sections, we will focus on the generalizations and evaluations of EDPP.

To develop the EDPP rules, we still follow the three steps in Section \ref{subsection:fundamental_DPP}. Indeed, by combining the two approaches proposed in Sections \ref{subsubsection:Improvement1} and \ref{subsubsection:Improvement2}, we can further improve the estimation of the dual optimal solution in the following theorem.

\begin{theorem}\label{thm:Lasso_EDPP_estimation}
	For the Lasso problem, suppose the dual optimal solution $\theta^*(\cdot)$ at $\lambda_0\in(0,\lambda_{\rm max}]$ is known, and $\forall\hspace{1mm}\lambda\in(0,\lambda_0]$, let ${\bf v}_2^{\perp}(\lambda,\lambda_0)$ be given by \eqref{eqn:Lasso_v2_perp}.
	Then, we have
	\begin{align}
		\left\|\theta^*(\lambda)-\left(\theta^*(\lambda_0)+\frac{1}{2}{\bf v}_2^{\perp}(\lambda,\lambda_0)\right)\right\|_2\leq\frac{1}{2}\|{\bf v}_2^{\perp}(\lambda,\lambda_0)\|_2.
	\end{align}
\end{theorem}
\begin{proof}
	Recall that $\theta_{\lambda_0}(t)$ is defined by \eqref{eqn:ray} and \eqref{eqn:Lasso_ray_lambdamx}. In view of (\ref{def:firmly_nonexpan}), we have
	\begin{align}\label{ineqn:firmly_nonexpan_lasso}
		\left\|P_F\left(\frac{\bf y}{\lambda}\right)-P_F\left(\theta_{\lambda_0}(t)\right)\right\|_2^2+\left\|({\rm Id}-P_F)\left(\frac{\bf y}{\lambda}\right)-({\rm Id}-P_F)\left(\theta_{\lambda_0}(t)\right)\right\|_2^2\leq\left\|\frac{\bf y}{\lambda}-\theta_{\lambda_0}(t)\right\|_2^2.
	\end{align}
	By expanding the second term on the left hand side of (\ref{ineqn:firmly_nonexpan_lasso}) and rearranging the terms, we obtain the following equivalent form:
	\begin{align}\label{ineqn:firmly_nonexpan_lasso_1}
		\left\|P_F\left(\frac{\bf y}{\lambda}\right)-P_F\left(\theta_{\lambda_0}(t)\right)\right\|_2^2\leq\left\langle\frac{\bf y}{\lambda}-\theta_{\lambda_0}(t),P_F\left(\frac{\bf y}{\lambda}\right)-P_F\left(\theta_{\lambda_0}(t)\right)\right\rangle.
	\end{align}
	In view of \eqref{eqn:Lasso_dual_projection}, \eqref{eqn:Lasso_projection_ray} and \eqref{eqn:Lasso_projection_ray_lambdamx}, the inequality in (\ref{ineqn:firmly_nonexpan_lasso_1}) can be rewritten as
	\begin{align}\label{ineqn:firmly_nonexpan_lasso_2}
		\|\theta^*(\lambda)-\theta^*(\lambda_0)\|_2^2&\leq\left\langle\frac{\bf y}{\lambda}-\theta_{\lambda_0}(t),\theta^*(\lambda)-\theta^*(\lambda_0)\right\rangle\\ \nonumber
		&=\left\langle\frac{\bf y}{\lambda}-\theta^*(\lambda_0)-t{\bf v}_1(\lambda_0),\theta^*(\lambda)-\theta^*(\lambda_0)\right\rangle\\ \nonumber
		&=\langle{\bf v}_2(\lambda,\lambda_0)-t{\bf v}_1(\lambda_0),\theta^*(\lambda)-\theta^*(\lambda_0)\rangle,\hspace{2mm}\forall t\geq0.
	\end{align}
	[Recall that ${\bf v}_1(\lambda_0)$ and ${\bf v}_2(\lambda,\lambda_0)$ are defined by \eqref{eqn:Lasso_v1} and \eqref{eqn:Lasso_v2} respectively.]
	Clearly, the inequality in (\ref{ineqn:firmly_nonexpan_lasso_2}) is equivalent to
	\begin{align}\label{ineqn:firmly_nonexpan_lasso_3}
		\left\|\theta^*(\lambda)-\left(\theta^*(\lambda_0)+\frac{1}{2}({\bf v}_2(\lambda,\lambda_0)-t{\bf v}_1(\lambda_0))\right)\right\|_2^2\leq\frac{1}{4}\|{\bf v}_2(\lambda,\lambda_0)-t{\bf v}_1(\lambda_0)\|_2^2, \hspace{2mm}\forall t\geq0.
	\end{align}
	The statement follows easily by minimizing the right hand side of the inequality in (\ref{ineqn:firmly_nonexpan_lasso_3}), which has been done in the proof of Theorem \ref{thm:Lasso_improve1_estimation}.
\end{proof}

%

Indeed, Theorem \ref{thm:Lasso_EDPP_estimation} is equivalent to bounding $\theta^*(\lambda)$ in a ball as follows:
\begin{align}\label{eqn:Lasso_EDPP_ball}
	\theta^*(\lambda)\in B\left(\theta^*(\lambda_0)+\frac{1}{2}{\bf v}_2^{\perp}(\lambda,\lambda_0),\frac{1}{2}\|{\bf v}_2^{\perp}(\lambda,\lambda_0)\|_2\right).
\end{align}
Based on this estimation and (\ref{rule1'}), we immediately have the following result.
\begin{theorem}\label{thm:Lasso_EDPP_fundamental_rule}
	For the Lasso problem, assume the dual optimal problem $\theta^*(\cdot)$ at $\lambda_0\in(0,\lambda_{\rm max}]$ is known, and $\lambda\in(0,\lambda_0]$. Then $[\beta^*(\lambda)]_i=0$ if the following holds:
	\begin{align*}
		\left|{\bf x}_i^T\left(\theta^*(\lambda_0)+\frac{1}{2}{\bf v}_2^{\perp}(\lambda,\lambda_0)\right)\right|<1-\frac{1}{2}\|{\bf v}_2^{\perp}(\lambda,\lambda_0)\|_2\|{\bf x}_i\|_2.
	\end{align*}
\end{theorem}
We omit the proof of Theorem \ref{thm:Lasso_EDPP_fundamental_rule} since it is very similar to the one of Theorem \ref{thm:Lasso_DPP}. Based on Theorem \ref{thm:Lasso_EDPP_fundamental_rule}, we can develop the EDPP rules as follows.
%

\begin{corollary}\label{corollary:Lasso_EDPPs}
	{\rm\bf EDPP:} For the Lasso problem, suppose we are given a sequence of parameter values $\lambda_{\rm max}=\lambda_0>\lambda_1>\ldots>\lambda_{\mathcal{K}}$. Then for any integer $0\leq k < \mathcal{K}$, we have $[\beta^*(\lambda_{k+1})]_i=0$ if $\beta^*(\lambda_k)$ is known and the following holds:
	\begin{align}\label{ineqn:Lasso_EDPPs}
		\left|{\bf x}_i^T\left(\frac{{\bf y}-{\bf X}\beta^*(\lambda_k)}{\lambda_k}+\frac{1}{2}{\bf v}_2^{\perp}(\lambda_{k+1},\lambda_k)\right)\right|<1-\frac{1}{2}\|{\bf v}_2^{\perp}(\lambda_{k+1},\lambda_k)\|_2\|{\bf x}_i\|_2.
	\end{align}
\end{corollary}

It is easy to see that the ball in (\ref{eqn:Lasso_EDPP_ball}) has the smallest radius compared to the ones in Theorems \ref{thm:Lasso_improve1_estimation} and \ref{thm:Lasso_improve2_estimate}, and thus it provides the most accurate estimation of the dual optimal solution. According to (\ref{rule1'}), EDPP is more effective in discarding the inactive features than DPP, Improvements 1 and 2.

\begin{figure*}[ht]
	\centering{
		\includegraphics[width=0.31\columnwidth]{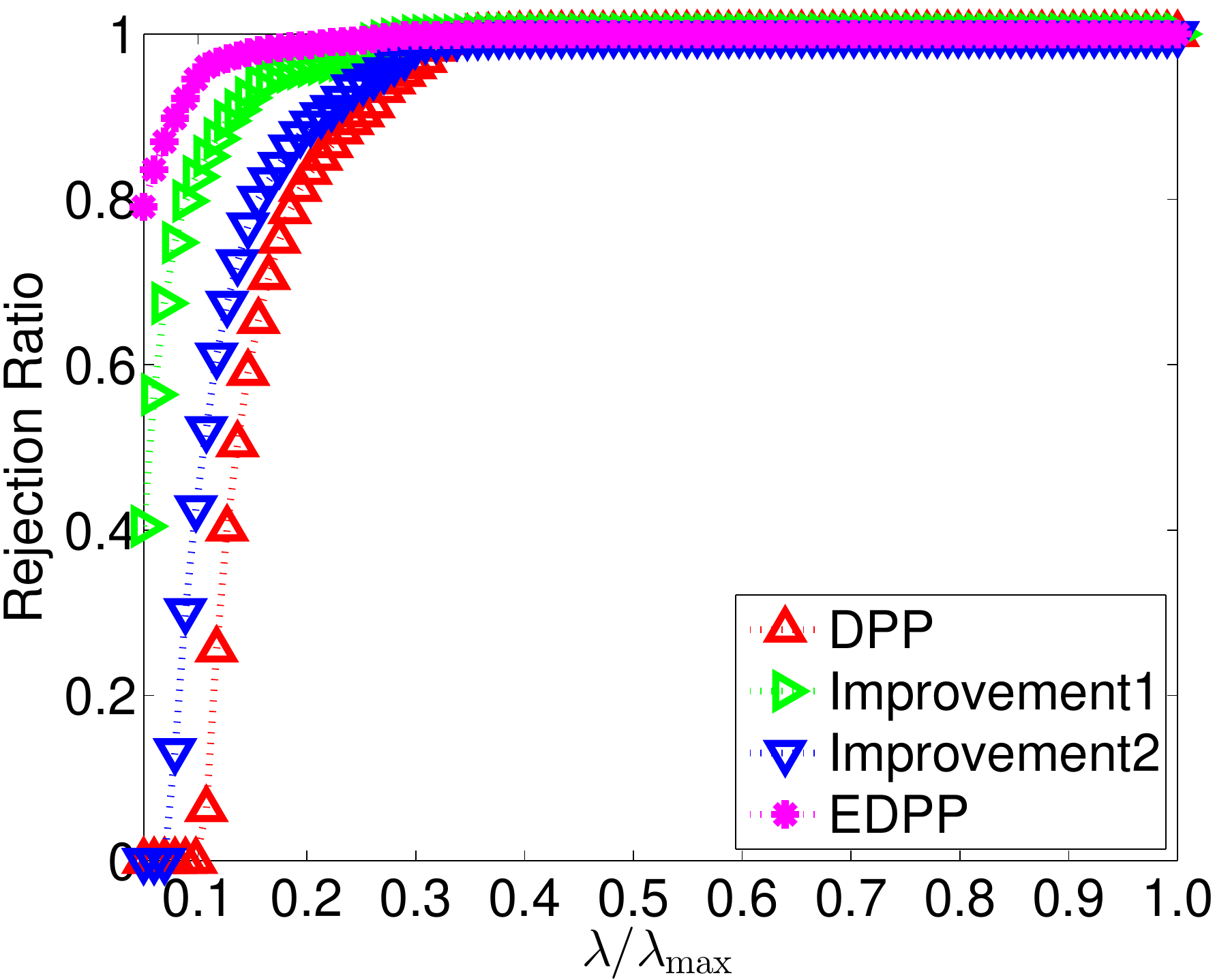}
		\includegraphics[width=0.31\columnwidth]{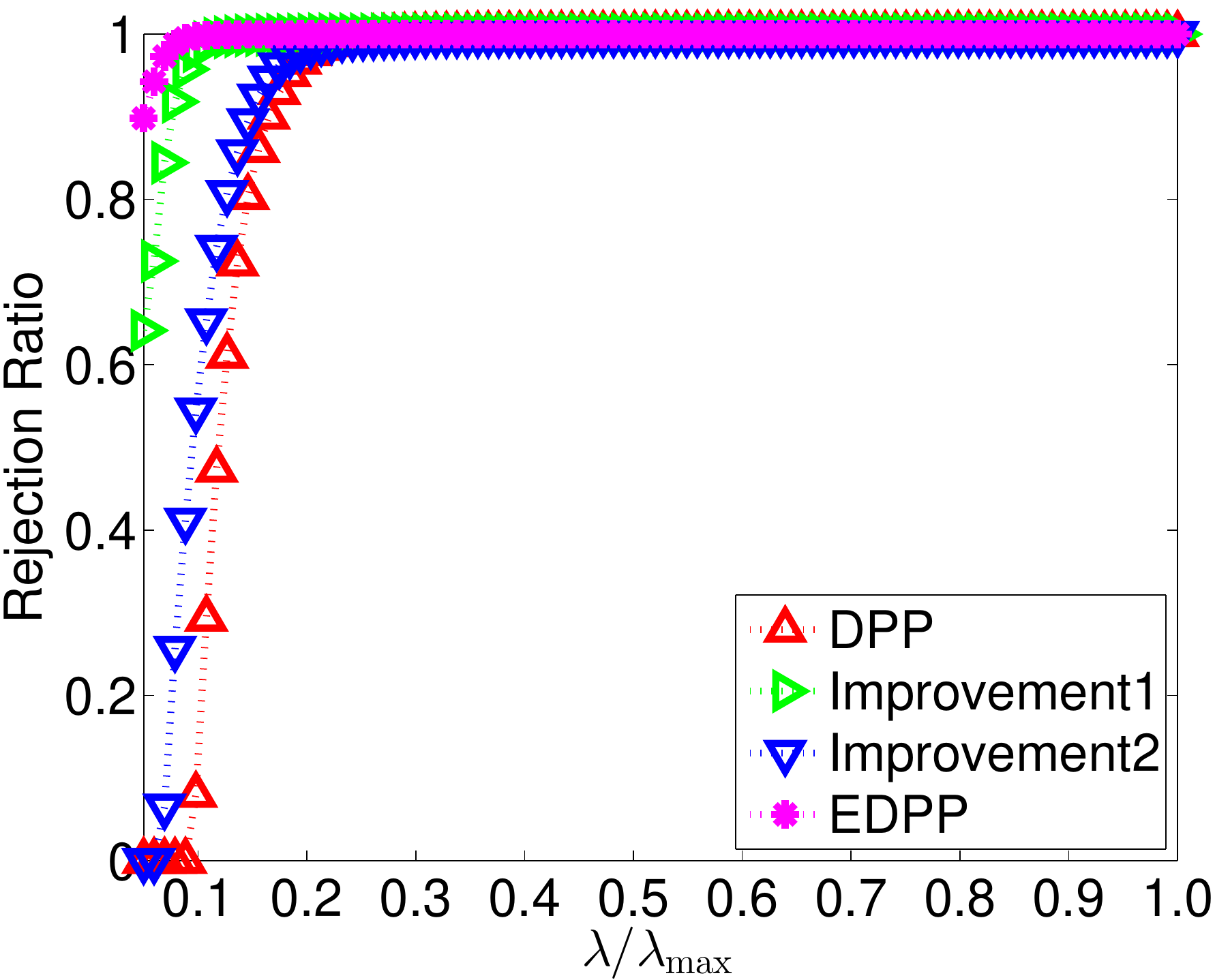}
		\includegraphics[width=0.31\columnwidth]{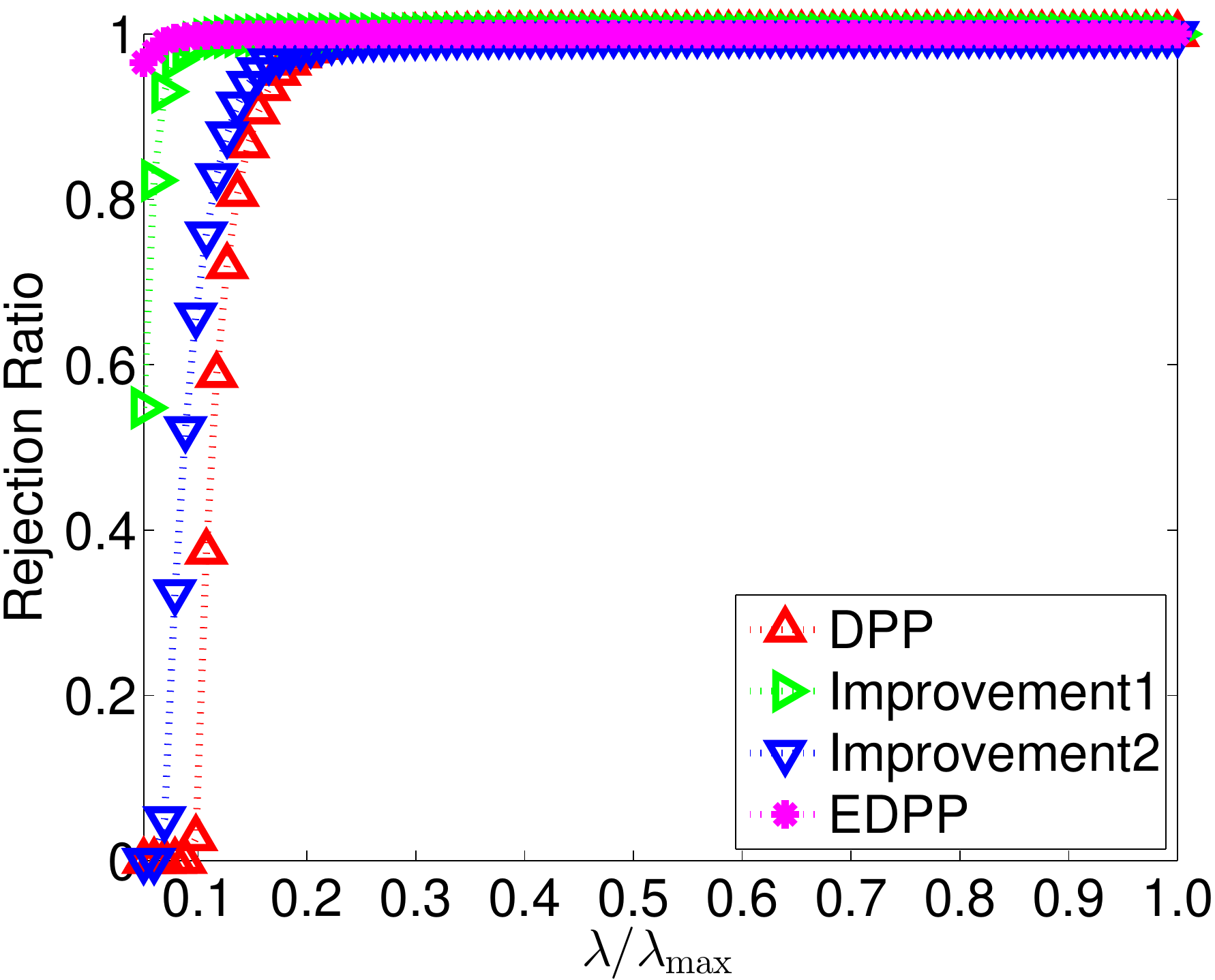}\\
		\subfigure[{\scriptsize Prostate Cancer}, ${\bf X}\in\mathbb{R}^{132\times 15154}$] { \label{fig:prostatecancer_e1}
			\includegraphics[width=0.31\columnwidth]{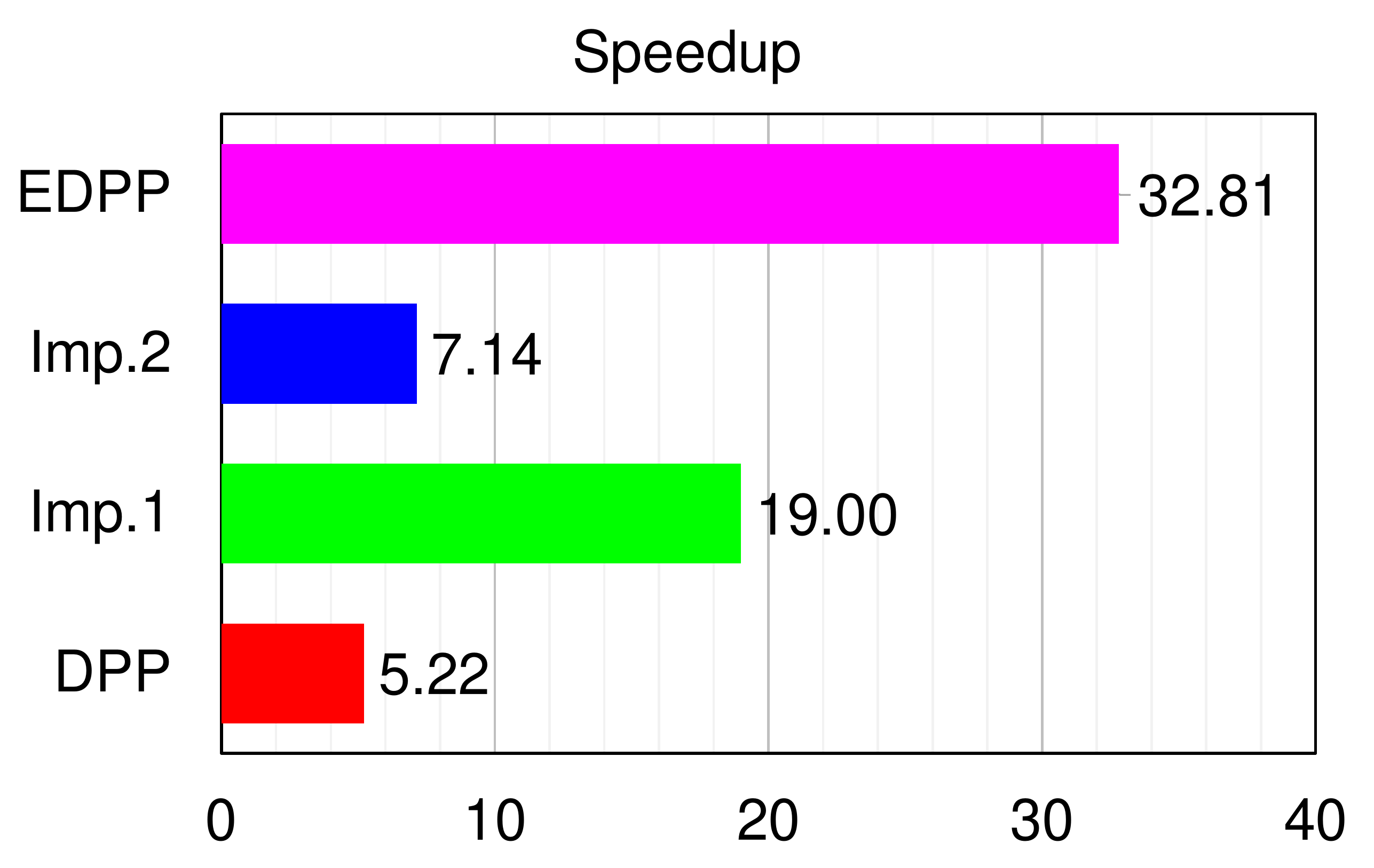}
		}
		\subfigure[PIE, ${\bf X}\in\mathbb{R}^{1024\times 11553}$] { \label{fig:mnist_e1}
			\includegraphics[width=0.31\columnwidth]{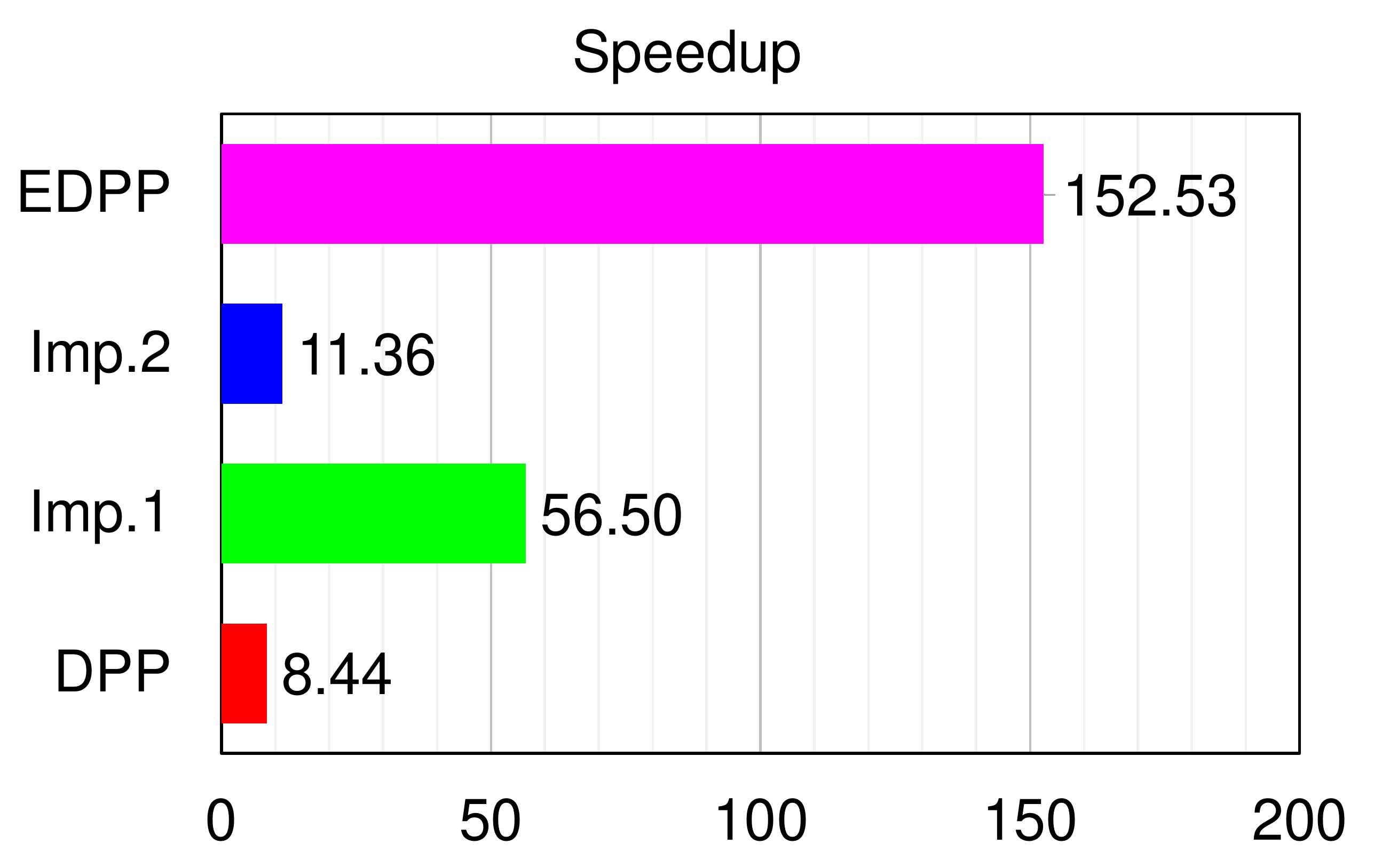}
		}
		\subfigure[MNIST, ${\bf X}\in\mathbb{R}^{784\times 50000}$] { \label{fig:pie_e1}
			\includegraphics[width=0.31\columnwidth]{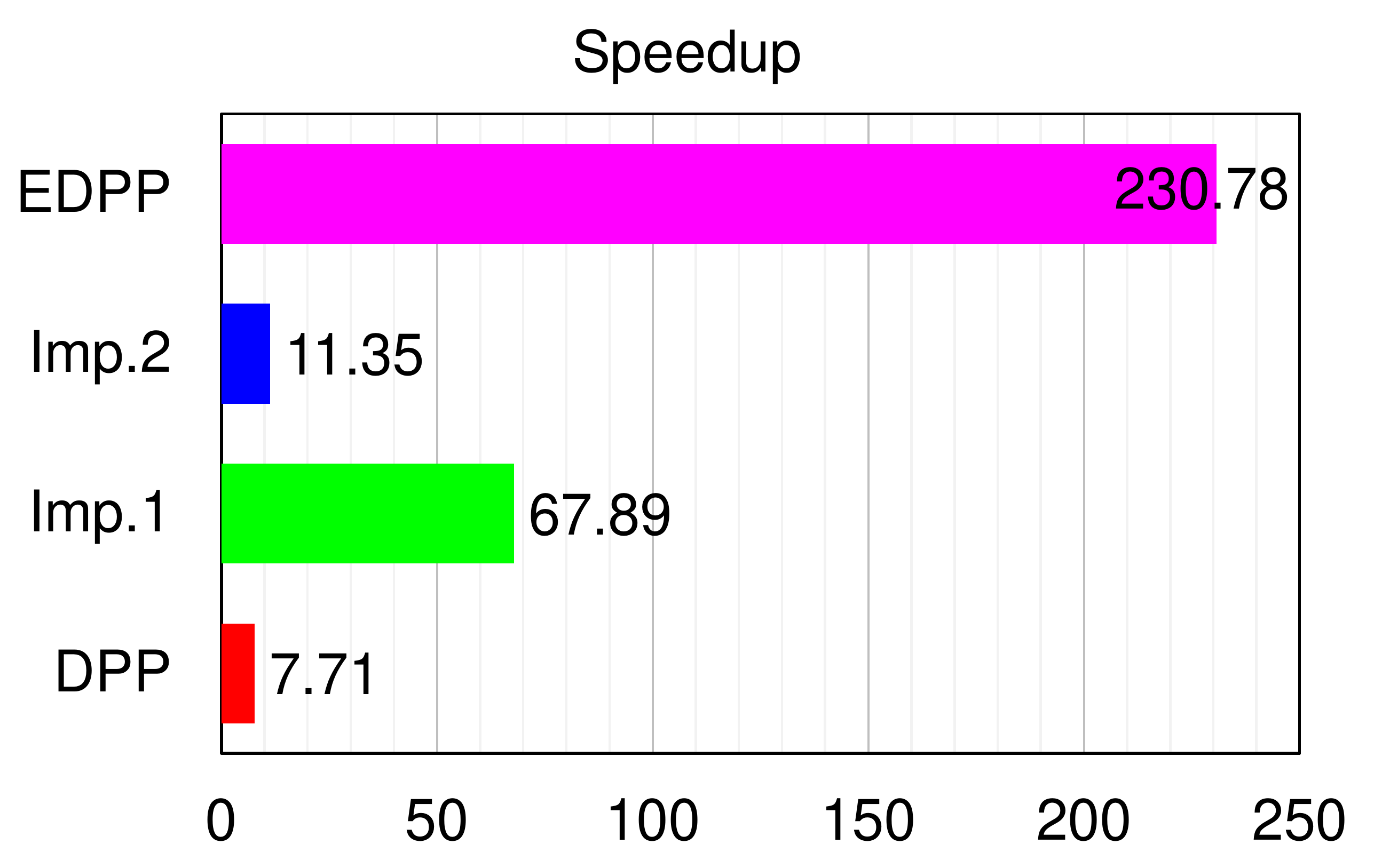}
		}}
		\caption{Comparison of the family of DPP rules on three real data sets: Prostate Cancer digit data set (left), PIE data set (middle) and MNIST image data set (right). The first row shows the rejection ratios of DPP, Improvement 1, Improvement 2 and EDPP. The second row presents the speedup gained by these four methods.}
		\label{fig:lasso_DPP}
	\end{figure*}
	
	\vspace{5mm}
	{\bf Comparisons of the Family of DPP rules} We evaluate the performance of the family of DPP screening rules, i.e., DPP, Improvement 1, Improvement 2 and EDPP, on three real data sets: a) the Prostate Cancer \citep{Petricoin2002}; b) the PIE face image data set \citep{Sim2003}; c) the MNIST handwritten digit data set \citep{Lecun1998}. To measure the performance of the screening rules, we compute the following two quantities:
	\begin{enumerate}
		\item the {\it rejection ratio}, i.e., the ratio of
		the number of features discarded by screening rules to the actual number of zero features in the ground truth;
		\item the {\it speedup}, i.e., the ratio of the running time of the solver with screening rules to the running time of the solver without screening.
	\end{enumerate}
	For each data set, we run the solver with or without the screening rules to solve the Lasso problem along a sequence of $100$ parameter values equally spaced on the $\lambda/\lambda_{\rm max}$ scale from $0.05$ to $1.0$. \figref{fig:lasso_DPP} presents the rejection ratios and speedup by the family of DPP screening rules. Table \ref{table:lasso_DPP_time} reports the running time of the solver with or without the screening rules for solving the $100$ Lasso problems, as well as the time for running  the screening rules.
	
	\setlength{\tabcolsep}{.18em}
	\begin{table}
		\begin{center}
			\begin{footnotesize}
				\def\arraystretch{1.25}
				\begin{tabular}{ |l||c||c|c|c|c||c|c|c|c| }
					\hline
					Data & solver & DPP+solver &  Imp.1+solver & Imp.2+solver & EDPP+solver & DPP & Imp.1 & Imp.2 & EDPP  \\
					\hline\hline
					Prostate Cancer  & 121.41 & 23.36 & 6.39 & 17.00 & 3.70 & 0.30 & 0.27 & 0.28 & 0.23\\\hline
					PIE  & 629.94 & 74.66 & 11.15 & 55.45 & 4.13 & 1.63 & 1.34 & 1.54 & 1.33 \\\hline
					MNIST  & 2566.26 & 332.87 & 37.80 & 226.02 & 11.12 & 5.28 & 4.36 & 4.94 & 4.19 \\\hline
				\end{tabular}
			\end{footnotesize}
		\end{center}
		\caption{Running time (in seconds) for solving the Lasso problems along a sequence of $100$ tuning parameter values equally spaced on the scale of ${\lambda}/{\lambda_{\rm max}}$ from $0.05$ to $1$ by (a): the solver \citep{SLEP} (reported in the second column) without screening; (b): the solver combined with different screening methods (reported in the $3^{rd}$ to the $6^{th}$ columns).
			The last four columns report the total running time (in seconds) for the screening methods.
		}
		\label{table:lasso_DPP_time}
	\end{table}
	
	\vspace{2mm}
	{\bf The Prostate Cancer Data Set} The Prostate Cancer data set \citep{Petricoin2002} is obtained by protein mass spectrometry. The features are indexed by time-of-flight values, which are related to the mass over charge ratios of the constituent proteins in the blood. The data set has $15154$ measurements of $132$ patients. $69$ of the patients have prostate cancer and the rest are healthy. Therefore, the data matrix ${\bf X}$ is of size $132\times 15154$, and the response vector ${\bf y}\in\{1,-1\}^{132}$ contains the binary labels of the patients.
	
	{\bf The PIE Face Image Data Set} The PIE face image data set used in this experiment\footnote{\url{http://www.cad.zju.edu.cn/home/dengcai/Data/FaceData.html}} \citep{Cai2007} contains $11554$ gray face images of $68$ people, taken under different poses,  illumination conditions and expressions. Each of the images has $32\times 32$ pixels. Therefore, in each trial, we first randomly pick an image as the response ${\bf y}\in\mathbb{R}^{1024}$, and then use the remaining images to form the data matrix ${\bf X}\in\mathbb{R}^{1024\times 11553}$. We run $100$ trials and report the average performance of the screening rules.
	
	{\bf The MNIST Handwritten Digit Data Set} This data set contains grey images of scanned handwritten digits, including $60,000$ for training and $10,000$ for testing. The dimension of each image is $28\times 28$. We first randomly select $5000$ images for each digit from the training set (and in total we have $50000$ images) and get a data matrix ${\bf X}\in\mathbb{R}^{784\times 50000}$. Then in each trial, we randomly select an image from the testing set as the response ${\bf y}\in\mathbb{R}^{784}$. We run $100$ trials and report the average performance of the screening rules.
	
	\vspace{2mm}
	From \figref{fig:lasso_DPP}, we can see that both Improvements 1 and 2 are able to discard more inactive features than DPP, and thus lead to a higher speedup. Compared to Improvement 2, we can also observe that Improvement 1 is more effective in discarding the inactive features. For the three data sets, the second row of \figref{fig:lasso_DPP} shows that Improvement 1 leads to about $20$, $60$, $70$ times speedup respectively, which are much higher than the ones gained by Improvement 1 (roughly $10$ times for all the three cases).
	
	Moreover, the EDPP rule, which combines the ideas of both Improvements 1 and 2, is even more effective in discarding the inactive features than Improvement 1. We can see that, for all of the three data sets and most of the $100$ parameter values, the rejection ratios of EDPP are very close to $100\%$. In other words, EDPP is able to discard almost all of the inactive features. Thus, the resulting speedup of EDPP is significantly better than the ones gained by the other three DPP rules. For the PIE and MNIST data sets, we can see that the speedup gained EDPP is about $150$ and $230$ times, which are two orders of magnitude. In view of Table \ref{table:lasso_DPP_time}, for the MNIST data set, the solver without screening needs about $2566.26$ seconds to solve the $100$ Lasso problems. In contrast, the solver with EDPP only needs $11.12$ seconds, leading to substantial savings in the computational cost. Moreover, from the last four columns of Table \ref{table:lasso_DPP_time}, we can also observe that the computational cost of the family of DPP rules are very low. Compared to that of the solver without screening, the computational cost of the family of DPP rules is negligible.
	
	In Section \ref{section:experiments}, we will only compare the performance of EDPP against several other state-of-the-art screening rules.
	
	\section{Extensions to Group Lasso}\label{section:glasso}
	
	To demonstrate the flexibility of the family of DPP rules, we extend the idea of EDPP to the group Lasso problem \citep{Yuan2006} in this section. Although the Lasso and group Lasso problems are very different from each other, we will see that their dual problems share a lot of similarities. For example, both of the dual problems can be formulated as looking for projections onto nonempty closed convex subsets of a Hilbert space. Recall that, the EDPP rule for the Lasso problem is entirely based on the properties of the projection operators. Therefore, the framework of the EDPP screening rule we developed for Lasso is also applicable for the group Lasso problem. In Section \ref{subsection:basics_groupLasso}, we briefly review some basics of the group Lasso problem and explore the geometric properties of its dual problem. In Section \ref{subsection:EDPP_groupLasso}, we develop the EDPP rule for the group Lasso problem.
	
	\subsection{Basics}\label{subsection:basics_groupLasso}
	With the group information available, the group Lasso problem takes the form of:
	\begin{equation}\label{problem:group_lasso_primal_mt}
		\inf_{\beta\in\mathbb{R}^{p}}\frac{1}{2}\left\|{\bf y}-\sum\nolimits_{g=1}^{G}{\bf X}_g\beta_g\right\|_2^2+\lambda\sum\nolimits_{g=1}^{G}\sqrt{n_g}\|\beta_g\|_2,
	\end{equation}
	where ${\bf X}_g\in\mathbb{R}^{N\times n_g}$ is the data matrix for the $g^{th}$ group and $p=\sum_{g=1}^{G}n_g$.
	The  dual problem of (\ref{problem:group_lasso_primal_mt}) is (see detailed derivation in the appendix):
	\begin{align}\label{problem:group_lasso_dual_mt}
		\sup_{\theta}\quad \left\{\frac{1}{2}\|{\bf y}\|_2^2 - \frac{\lambda^2}{2}\left\|\theta - \frac{{\bf y}}{\lambda}\right\|_2^2:\,\, \|{\bf X}_g^{T}\theta\|_2\leq \sqrt{n_g},\,g=1,2,\ldots,G\right\}
	\end{align}
	The KKT conditions are given by
	\begin{align}\label{equation:group_lasso_primal_dual}
		{\bf y}&=\sum\nolimits_{g=1}^{G}{\bf X}_g\beta_g^{\ast}(\lambda)+\lambda\theta^{\ast}(\lambda),\\ \label{equation:group_lasso_KKT}
		(\theta^{\ast}(\lambda))^{T}{\bf X}_g&\in
		\begin{cases}
			\sqrt{n_g}\frac{\beta_g^{\ast}(\lambda)}{\|\beta_g^{\ast}(\lambda)\|_2},\hspace{9mm}{\rm if}\beta_g^{\ast}(\lambda)\neq 0,  \\
			\sqrt{n_g}{\bf u},\,\|{\bf u}\|_2\leq 1,\hspace{2mm}{\rm if}\beta_g^{\ast}(\lambda)= 0.   \\
		\end{cases}
	\end{align}
	for $g=1,2,\ldots,G$.
	Clearly, in view of \eqref{equation:group_lasso_KKT}, we can see that
	\begin{align}\tag{R2}\label{rule2}
		\|(\theta^{\ast}(\lambda))^{T}{\bf X}_g\|_2<\sqrt{n_g}\Rightarrow \beta_g^{\ast}(\lambda)=0
	\end{align}
	However, since $\theta^*(\lambda)$ is generally unknown, (\ref{rule2}) is not applicable to identify the {\it inactive groups}, i.e., the groups which have $0$ coefficients in the solution vector, for the group Lasso problem. Therefore, similar to the Lasso problem, we can first find a region $\overline{\bf \Theta}$ which contains $\theta^*(\lambda)$, and then (\ref{rule2}) can be relaxed as follows:
	\begin{align}\tag{R2$'$}\label{rule2'}
		\sup_{\theta\in\overline{\bf \Theta}}\|(\theta)^{T}{\bf X}_g\|_2<\sqrt{n_g}\Rightarrow \beta_g^{\ast}(\lambda)=0.
	\end{align}
	Therefore, to develop screening rules for the group Lasso problem, we only need to estimate the region $\overline{\bf \Theta}$ which contains $\theta^*(\lambda)$, solve the maximization problem in (\ref{rule2'}), and plug it into (\ref{rule2'}). In other words, the three steps proposed in Section \ref{subsection:fundamental_DPP} can also be applied to develop screening rules for the group Lasso problem. Moreover, (\ref{rule2'}) also implies that the smaller the region $\overline{\bf \Theta}$ is, the more accurate the estimation of the dual optimal solution is. As a result, the more effective the resulting screening rule is in discarding the inactive features.
	
	{\bf Geometric Interpretations}
	For notational convenience, let $\overline{F}$ be the feasible set of problem (\ref{problem:group_lasso_dual_mt}). Similar to the case of Lasso, problem (\ref{problem:group_lasso_dual_mt})implies that the dual optimal $\theta^{\ast}(\lambda)$ is the projection of ${\bf y}/{\lambda}$ onto the feasible set $\overline{F}$,
	i.e.,
	\begin{align}\label{eqn:groupLasso_dual_projection}
		\theta^*(\lambda)=P_{\overline{F}}\left(\frac{\bf y}{\lambda}\right),\hspace{2mm}\forall \hspace{1mm}\lambda>0.
	\end{align}
	Compared to \eqref{eqn:Lasso_dual_projection}, the only difference in \eqref{eqn:groupLasso_dual_projection} is that the feasible set $\overline{F}$ is the intersection of a set of ellipsoids, and thus not a polytope. However, similar to $F$, $\overline{F}$ is also a nonempty closed and convex (notice that $0$ is a feasible point). Therefore, we can make use of all the aforementioned properties of the projection operators, e.g., Lemmas \ref{lemma:projection_ray} and \ref{lemma:projection_ray_lambdamx}, Theorems \ref{thm:projection_normal_cone} and \ref{thm:firmly_nonexpansive}, to develop screening rules for the group Lasso problem.
	
	Moreover, similar to the case of Lasso, we also have a specific parameter value \citep{Tibshirani11} for the group Lasso problem, i.e.,
	\begin{align}\label{eqn:gLasso_lambdamx}
		\overline{\lambda}_{\rm max}=\max_g \frac{\|{\bf X}_g^T{\bf y}\|_2}{\sqrt{n_g}}.
	\end{align}
	Indeed, $\overline{\lambda}_{\rm max}$ is the smallest parameter value such that the optimal solution of problem (\ref{problem:group_lasso_primal_mt}) is 0. More specifically, we have:
	\begin{align}\label{eqn:groupLasso_beta_0}
		\beta^*(\lambda)=0,\hspace{2mm}\forall\hspace{1mm}\lambda\in[\overline{\lambda}_{\rm max},\infty).
	\end{align}
	Combining the result in (\ref{eqn:groupLasso_beta_0}) and \eqref{equation:group_lasso_primal_dual}, we immediately have
	\begin{align}\label{eqn:groupLasso_theta_closed_form}
		\theta^*(\lambda)=\frac{\bf y}{\lambda},\hspace{2mm}\forall\hspace{1mm}\lambda\in[\overline{\lambda}_{\rm max},\infty).
	\end{align}
	Therefore, all through the subsequent sections, we will focus on the cases with $\lambda\in(0,\overline{\lambda}_{\rm max})$.

	\subsection{Enhanced DPP rule for Group Lasso}\label{subsection:EDPP_groupLasso}
	
	In view of (\ref{rule2'}), we can see that the estimation of the dual optimal solution is the key step to develop a screening rule for the group Lasso problem. Because $\theta^*(\lambda)$ is the projection of ${\bf y}/\lambda$ onto the nonempty closed convex set $\overline{F}$ [please refer to \eqref{eqn:groupLasso_dual_projection}], we can make use of all the properties of projection operators, e.g., Lemmas \ref{lemma:projection_ray} and \ref{lemma:projection_ray_lambdamx}, Theorems \ref{thm:projection_normal_cone} and \ref{thm:firmly_nonexpansive}, to estimate the dual optimal solution. First, let us develop a useful technical result as follows.
	
	\begin{lemma}\label{lemma:gLasso_projection_ray}
		For the group Lasso problem, let $\overline{\lambda}_{\rm max}$ be given by \eqref{eqn:gLasso_lambdamx} and
		\begin{align}\label{eqn:groupLasso_Xstar}
			{\bf X}_*:={\rm argmax}_{{\bf X}_g}\frac{\|{\bf X}_g^T{\bf y}\|_2}{\sqrt{n_g}}.
		\end{align}
		Suppose the dual optimal solution $\theta^*(\cdot)$ is known at $\lambda_0\in(0,\overline{\lambda}_{\rm max}]$, let us define
		\begin{align}\label{eqn:v1_gLasso}
			\overline{\bf v}_1(\lambda_0)&=
			\begin{cases}
				\frac{\bf y}{\lambda_0}-\theta^*(\lambda_0),\hspace{2mm}{\rm if}\hspace{2mm}\lambda_0\in(0,\overline{\lambda}_{\rm max}),\\
				{\bf X}_*{\bf X}_*^T{\bf y},\hspace{8.5mm}{\rm if}\hspace{2mm}\lambda_0=\overline{\lambda}_{\rm max}.\\
			\end{cases}\\
			\overline{\theta}_{\lambda_0}(t)&=\theta^*(\lambda_0)+t\overline{\bf v}_1(\lambda_0),\hspace{2mm}t\geq0.
		\end{align}
		Then, we have the following result holds
		\begin{align}\label{eqn:groupLasso_projection_ray}
			P_{\overline{F}}(\overline{\theta}_{\lambda_0}(t))=\theta^*(\lambda_0), \hspace{1mm} \forall \,\,t\geq0.
		\end{align}
	\end{lemma}
	
	\begin{proof}
		Let us first consider the cases with $\lambda_0\in(0,\overline{\lambda}_{\rm max})$. In view of the definition of $\overline{\lambda}_{\rm max}$, it is easy to see that ${\bf y}/\lambda_0\notin \overline{F}$. Therefore, in view of \eqref{eqn:groupLasso_dual_projection} and Lemma \ref{lemma:projection_ray}, the statement in \eqref{eqn:groupLasso_projection_ray} follows immediately.
		
		We next consider the case with $\lambda_0=\overline{\lambda}_{\rm max}$. By Theorem \ref{thm:projection_normal_cone}, we only need to check if
		\begin{align}\label{inclusion:norm_cone_gLasso}
			\overline{\bf v}_1(\overline{\lambda}_{\rm max})\in N_{\overline{F}}(\theta^*(\overline{\lambda}_{\rm max}))\Leftrightarrow \left\langle\overline{\bf v}_1(\overline{\lambda}_{\rm max}),\theta-\theta^*(\overline{\lambda}_{\rm max})\right\rangle\leq0, \hspace{2mm}\forall\hspace{1mm}\theta\in\overline{F}.
		\end{align}
		Indeed, in view of \eqref{eqn:gLasso_lambdamx} and \eqref{eqn:groupLasso_theta_closed_form}, we can see that
		\begin{align}\label{eqn:groupLasso_v1thetamx_ip}
			\langle\overline{\bf v}_1(\overline{\lambda}_{\rm max}),\theta^*(\overline{\lambda}_{\rm max})\rangle=\left\langle{\bf X}_*{\bf X}_*^T{\bf y},\frac{\bf y}{\overline{\lambda}_{\rm max}}\right\rangle=\frac{\|{\bf X}_*^T{\bf y}\|_2^2}{\overline{\lambda}_{\rm max}}.
		\end{align}
		On the other hand, by \eqref{eqn:gLasso_lambdamx} and \eqref{eqn:groupLasso_Xstar}, we can see that
		\begin{align}\label{eqn:groupLasso_norm_Xstary}
			{\|{\bf X}_*^T{\bf y}\|_2}=\overline{\lambda}_{\rm max}\sqrt{n_*},
		\end{align}
		where $n_*$ is the number of columns of ${\bf X}_*$.
		By plugging \eqref{eqn:groupLasso_norm_Xstary} into \eqref{eqn:groupLasso_v1thetamx_ip}, we have
		\begin{align}
			\langle\overline{\bf v}_1(\overline{\lambda}_{\rm max}),\theta^*(\overline{\lambda}_{\rm max})\rangle=\overline{\lambda}_{\rm max}\cdot n_*.
		\end{align}
		Moreover, for any feasible point $\theta\in\overline{F}$, we can see that
		\begin{align}\label{ineqn:groupLasso_Xstar_F}
			\|{\bf X}_*^T\theta\|_2\leq\sqrt{n_*}.
		\end{align}
		In view of the result in (\ref{ineqn:groupLasso_Xstar_F}) and \eqref{eqn:groupLasso_norm_Xstary}, it is easy to see that
		\begin{align}\label{ineqn:groupLasso_v1F_ip}
			\left\langle\overline{\bf v}_1(\overline{\lambda}_{\rm max}),\theta\right\rangle=\left\langle{\bf X}_*{\bf X}_*^T{\bf y},\theta\right\rangle=\left\langle{\bf X}_*^T{\bf y},{\bf X}_*^T\theta\right\rangle\leq\|{\bf X}_*^T{\bf y}\|_2\|{\bf X}_*^T\theta\|_2=\overline{\lambda}_{\rm max}\cdot n_*.
		\end{align}
		Combining the result in \eqref{eqn:groupLasso_v1thetamx_ip} and (\ref{ineqn:groupLasso_v1F_ip}), it is easy to see that the inequality in (\ref{inclusion:norm_cone_gLasso}) holds for all $\theta\in\overline{F}$, which completes the proof.
	\end{proof}
	
	By Lemma \ref{lemma:gLasso_projection_ray}, we can accurately estimate the dual optimal solution of the group Lasso problem in the following theorem. It is easy to see that the result in Theorem \ref{thm:gLasso_estimation} is very similar to the one in Theorem \ref{thm:Lasso_EDPP_estimation} for the Lasso problem.
	
	\begin{theorem}\label{thm:gLasso_estimation}
		For the group Lasso problem, suppose the dual optimal solution $\theta^*(\cdot)$ at $\theta_0\in(0,\overline{\lambda}_{\rm max}]$ is known, and $\overline{\bf v}_1(\lambda_0)$ is given by \eqref{eqn:v1_gLasso}. For any $\lambda\in(0,\lambda_0]$, let us define
		\begin{align}\label{eqn:groupLasso_v2}
			\overline{\bf v}_2(\lambda,\lambda_0)=\frac{\bf y}{\lambda}-\theta^*(\lambda_0),
		\end{align}
		\begin{align}\label{eqn:groupLasso_v2_perp}
			\overline{\bf v}_2^{\perp}(\lambda,\lambda_0)=\overline{\bf v}_2(\lambda,\lambda_0)-\frac{\langle\overline{\bf v}_1(\lambda_0),\overline{\bf v}_2(\lambda,\lambda_0)\rangle}{\|\overline{\bf v}_1(\lambda_0)\|_2^2}\overline{\bf v}_1(\lambda_0).
		\end{align}
		Then, the dual optimal solution $\theta^*(\lambda)$ can be estimated as follows:
		\begin{align}
			\left\|\theta^*(\lambda)-\left(\theta^*(\lambda_0)+\frac{1}{2}\overline{\bf v}_2^{\perp}(\lambda,\lambda_0)\right)\right\|_2\leq\frac{1}{2}\|\overline{\bf v}_2^{\perp}(\lambda,\lambda_0)\|_2.
		\end{align}
	\end{theorem}
	
	We omit the proof of Theorem \ref{thm:gLasso_estimation} since it is exactly the same as the one of Theorem \ref{thm:Lasso_EDPP_estimation}.
	Indeed, Theorem \ref{thm:gLasso_estimation} is equivalent to estimating $\theta^*(\lambda)$ in a ball as follows:
	\begin{align}\label{eqn:gLasso_EDPP_ball}
		\theta^*(\lambda)\in B\left(\theta^*(\lambda_0)+\frac{1}{2}\overline{\bf v}_2^{\perp}(\lambda,\lambda_0),\frac{1}{2}\|\overline{\bf v}_2^{\perp}(\lambda,\lambda_0)\|_2\right).
	\end{align}
	Based on this estimation and (\ref{rule2'}), we immediately have the following result.
	
	\begin{theorem}\label{theorem:fundamental_rule_glasso}
		For the group Lasso problem, assume the dual optimal solution $\theta^*(\cdot)$ is known at $\lambda_0\in(0,\overline{\lambda}_{\rm max}]$, and $\lambda\in(0,\lambda_0]$. Then $\beta_g^{\ast}(\lambda)=0$ if the following holds
		\begin{equation}\label{equation:gDPP}
			\left\|{\bf X}_g^T\left(\theta^*(\lambda_0)+\frac{1}{2}\overline{\bf v}_2^{\perp}(\lambda,\lambda_0)\right)\right\|_2<\sqrt{n_g}-\frac{1}{2}\|\overline{\bf v}_2^{\perp}(\lambda,\lambda_0)\|_2\|{\bf X}_g\|_2.
		\end{equation}
	\end{theorem}
	
	\begin{proof}
		In view of (\ref{rule2'}), we only need to check if
		\begin{align*}
			\left\|{\bf X}_g^T\theta^{\ast}(\lambda)\right\|_2<\sqrt{n_g}.
		\end{align*}
		To simplify notations, let
		$$
		{\bf o} = \theta^*(\lambda_0)+\frac{1}{2}\overline{\bf v}_2^{\perp}(\lambda,\lambda_0),\hspace{2mm}r = \frac{1}{2}\|\overline{\bf v}_2^{\perp}(\lambda,\lambda_0)\|_2.
		$$
		It is easy to see that
		\begin{align}\label{equation:proof_g_fundamental}
			\left\|{\bf X}_g^T\theta^*(\lambda)\right\|_2&\leq\|{\bf X}_g^T(\theta^{\ast}(\lambda)-{\bf o})\|_2+\|{\bf X}_g^T{\bf o}\|_2\\ \nonumber
			&<\|{\bf X}_g\|_2\|\theta^{\ast}(\lambda)-{\bf o}\|_2 + \sqrt{n_g}-r\|{\bf X}_g\|_2\\ \nonumber
			&\leq r\|{\bf X}_g\|_2 + \sqrt{n_g}-r\|{\bf X}_g\|_2=\sqrt{n_g},
		\end{align}
		which completes the proof. The second and third inequalities in (\ref{equation:proof_g_fundamental}) are due to (\ref{equation:gDPP}) and Theorem \ref{thm:gLasso_estimation}, respectively.
	\end{proof}
	
	In view of \eqref{equation:group_lasso_primal_dual} and Theorem \ref{theorem:fundamental_rule_glasso}, we can derive the EDPP rule to discard the inactive groups for the group Lasso problem as follows.
	
	
	\begin{corollary}
		{\bf EDPP}: For the group Lasso problem (\ref{problem:group_lasso_primal_mt}), suppose we are given a sequence of parameter values $\overline{\lambda}_{\rm max}=\lambda_0>\lambda_1>\ldots>\lambda_{\mathcal{K}}$. For any integer $0\leq k<\mathcal{K}$, we have $\beta_g^{\ast}(\lambda_{k+1})=0$ if $\beta^{\ast}(\lambda_k)$ is known and the following holds:
		\begin{align*}
			\left\|{\bf X}_g^T\left(\frac{{\bf y}-\sum_{g=1}^{G}{\bf X}_g\beta_g^{\ast}(\lambda_k)}{\lambda_{k}}+\frac{1}{2}\overline{\bf v}_2^{\perp}(\lambda_{k+1},\lambda_k)\right)\right\|_2 <\sqrt{n_g}-\frac{1}{2}\|\overline{\bf v}_2^{\perp}(\lambda_{k+1},\lambda_k)\|_2\|{\bf X}_g\|_2.
		\end{align*}
	\end{corollary}

	\section{Experiments}\label{section:experiments}
	
	In this section, we evaluate the proposed EDPP rules for Lasso and group Lasso on both synthetic and real data sets. To measure the performance of our screening rules, we compute the {\it rejection ratio} and {\it speedup} (please refer to Section \ref{subsubsection:EDPP} for details). Because the EDPP rule is safe, i.e., no active features/groups will be mistakenly discarded, the rejection ratio will be less than one.
	
	In Section \ref{subsection:experiment_lasso}, we conduct two sets of experiments to compare the performance of EDPP against several state-of-the-art screening methods. We first compare the performance of the basic versions of EDPP, DOME, SAFE, and strong rule. Then, we focus on the sequential versions of EDPP, SAFE, and strong rule. Notice that, SAFE and EDPP are safe. However, strong rule may mistakenly discard features with nonzero coefficients in the solution. Although DOME is also safe for the Lasso problem, it is unclear if there exists a sequential version of DOME. Recall that, real applications usually favor the sequential screening rules because we need to solve a sequence of of Lasso problems to determine an appropriate parameter value \citep{Tibshirani11}. Moreover, DOME assumes special structure on the data, i.e., each feature and the response vector should be normalized to have unit length.
	
	In Section \ref{subsection:experiment_glasso}, we compare EDPP with strong rule for the group Lasso problem on synthetic data sets. We are not aware of any {\it safe} screening rules for the group Lasso problem at this point. For SAFE and Dome, it is not straightforward to extend them to the group Lasso problem.

	\subsection{EDPP for the Lasso Problem}\label{subsection:experiment_lasso}
	
	For the Lasso problem, we first compare the performance of the basic versions of EDPP, DOME, SAFE and strong rule in Section  \ref{subsubsection:exp_basic_Lasso}. Then, we compare the performance of the sequential versions of EDPP, SAFE and strong rule in Section \ref{subsubsection:exp_seq_Lasso}.
	
	\subsubsection{Evaluation of the Basic EDPP Rule}\label{subsubsection:exp_basic_Lasso}
	
	In this section, we perform experiments on six real data sets to compare the performance of the basic versions of SAFE, DOME, strong rule and EDPP. Briefly speaking, suppose that we are given a parameter value $\lambda$. Basic versions of the aforementioned screening rules always make use of $\beta^*(\lambda_{\rm max})$ to identify the zero components of $\beta^*(\lambda)$. Take EDPP for example. The basic version of EDPP can be obtained by replacing $\beta^*(\lambda_k)$ and ${\bf v}_2^{\perp}(\lambda_{k+1},\lambda_k)$ with $\beta^*(\lambda_0)$ and ${\bf v}_2^{\perp}(\lambda_{k},\lambda_0)$, respectively, in (\ref{ineqn:Lasso_EDPPs}) for all $k=1,\ldots,\mathcal{K}$.
	
	In this experiment, we report the rejection ratios of the basic SAFE, DOME, strong rule and EDPP along a sequence of $100$ parameter values equally spaced on the $\lambda/\lambda_{\rm max}$ scale from $0.05$ to $1.0$. We note that DOME requires that all features of the data sets have unit length. Therefore, to compare the performance of DOME with SAFE, strong rule and EDPP, we normalize the features of all the data sets used in this section. However, it is worthwhile to mention that SAFE, strong rule and EDPP do not assume any specific structures on the data set.
	The data sets used in this section are listed as follows:
	\begin{enumerate}
		\item[a)] Colon Cancer data set \citep{Alon1999};
		\item[b)] Lung Cancer data set \citep{Bhattacharjee2001};
		\item[c)] Prostate Cancer data set \citep{Petricoin2002};
		\item[d)] PIE face image data set \citep{Sim2003,Cai2007};
		\item[e)] MNIST handwritten digit data set \citep{Lecun1998};
		\item[f)] COIL-100 image data set \citep{Nene1996a,Cai2011}.
	\end{enumerate}
	
	\begin{figure*}[ht]
		\centering{
			\subfigure[Colon Cancer, ${\bf X}\in\mathbb{R}^{62\times 2000}$] { \label{fig:ColonCancer}
				\includegraphics[width=0.31\columnwidth]{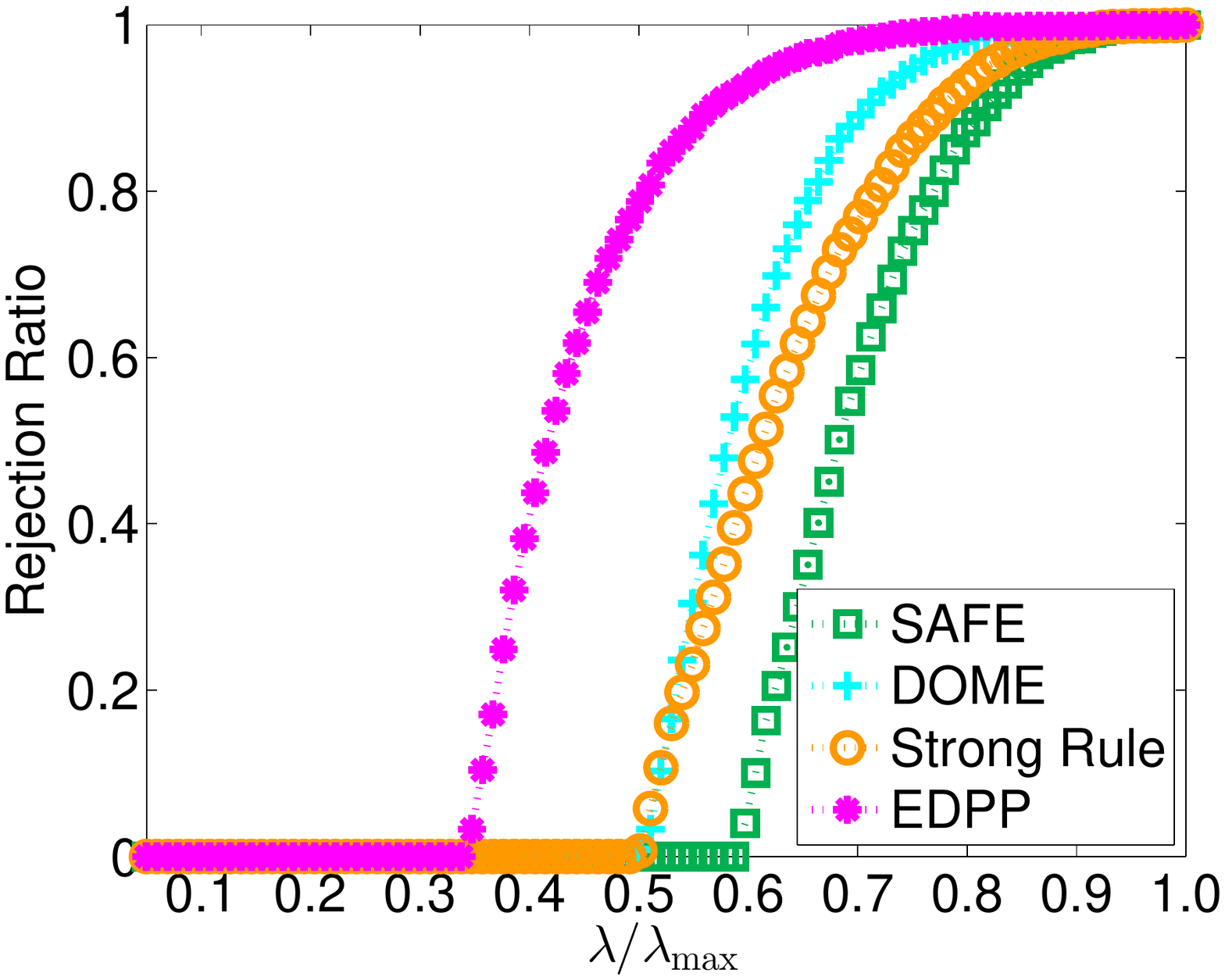}
			}
			\subfigure[Lung Cancer, ${\bf X}\in\mathbb{R}^{203\times 12600}$] { \label{fig:LungCancer}
				\includegraphics[width=0.31\columnwidth]{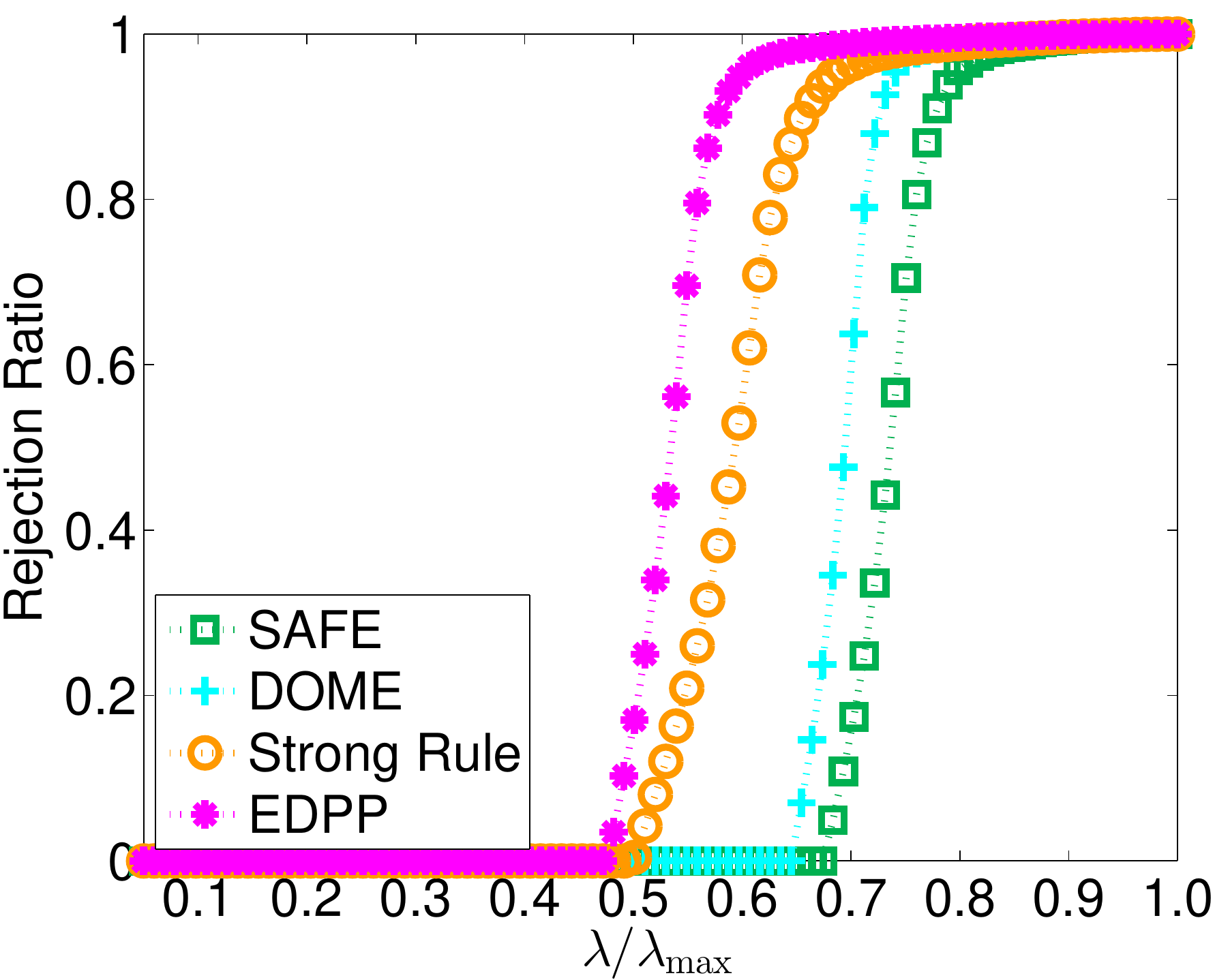}
			}
			\subfigure[{\scriptsize Prostate Cancer}, ${\bf X}\in\mathbb{R}^{132\times 15154}$] { \label{fig:prostatecancer_basic}
				\includegraphics[width=0.31\columnwidth]{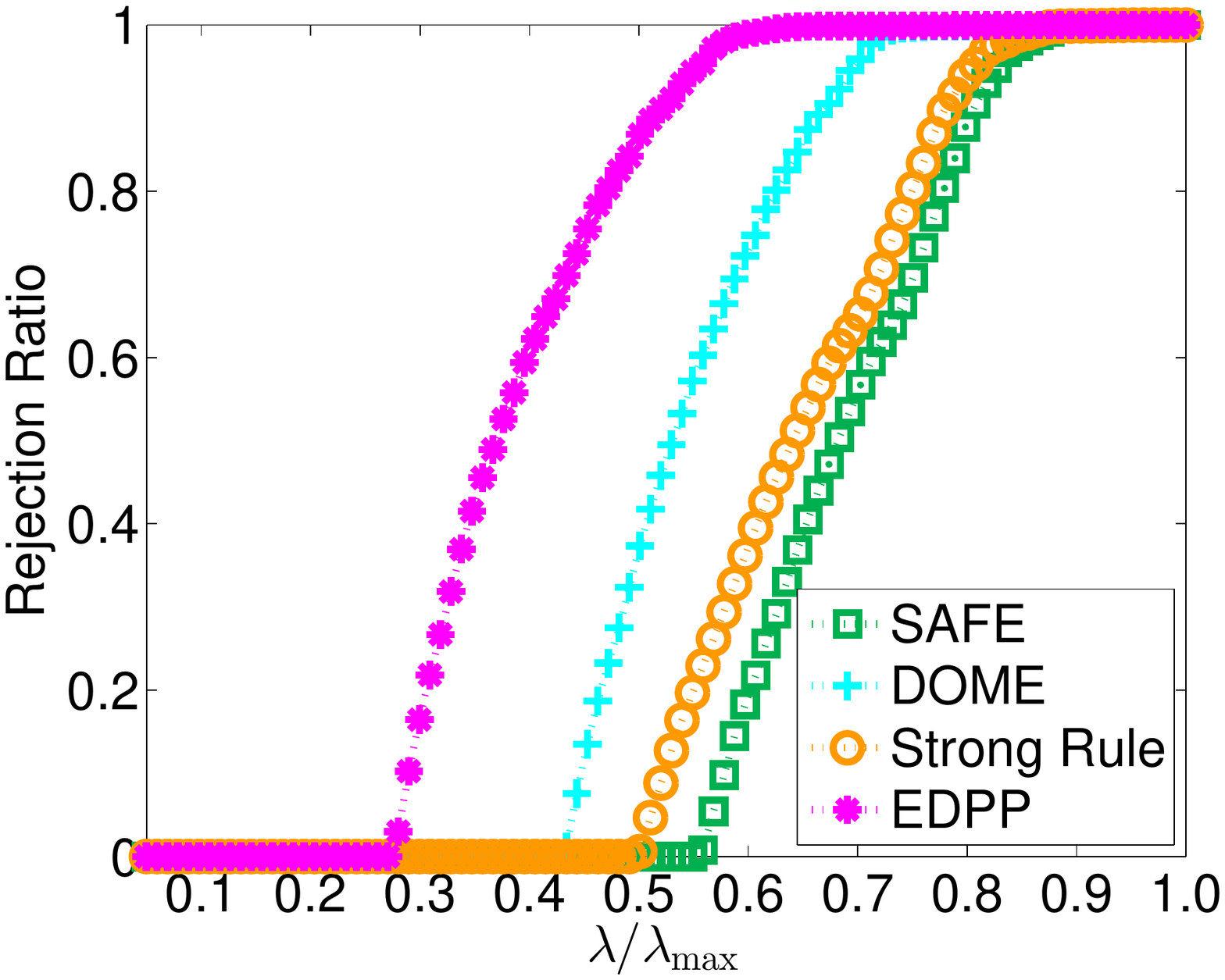}
			}\\
			\vspace{5mm}
			\subfigure[PIE, ${\bf X}\in\mathbb{R}^{1024\times 11553}$] { \label{fig:pie_basic}
				\includegraphics[width=0.31\columnwidth]{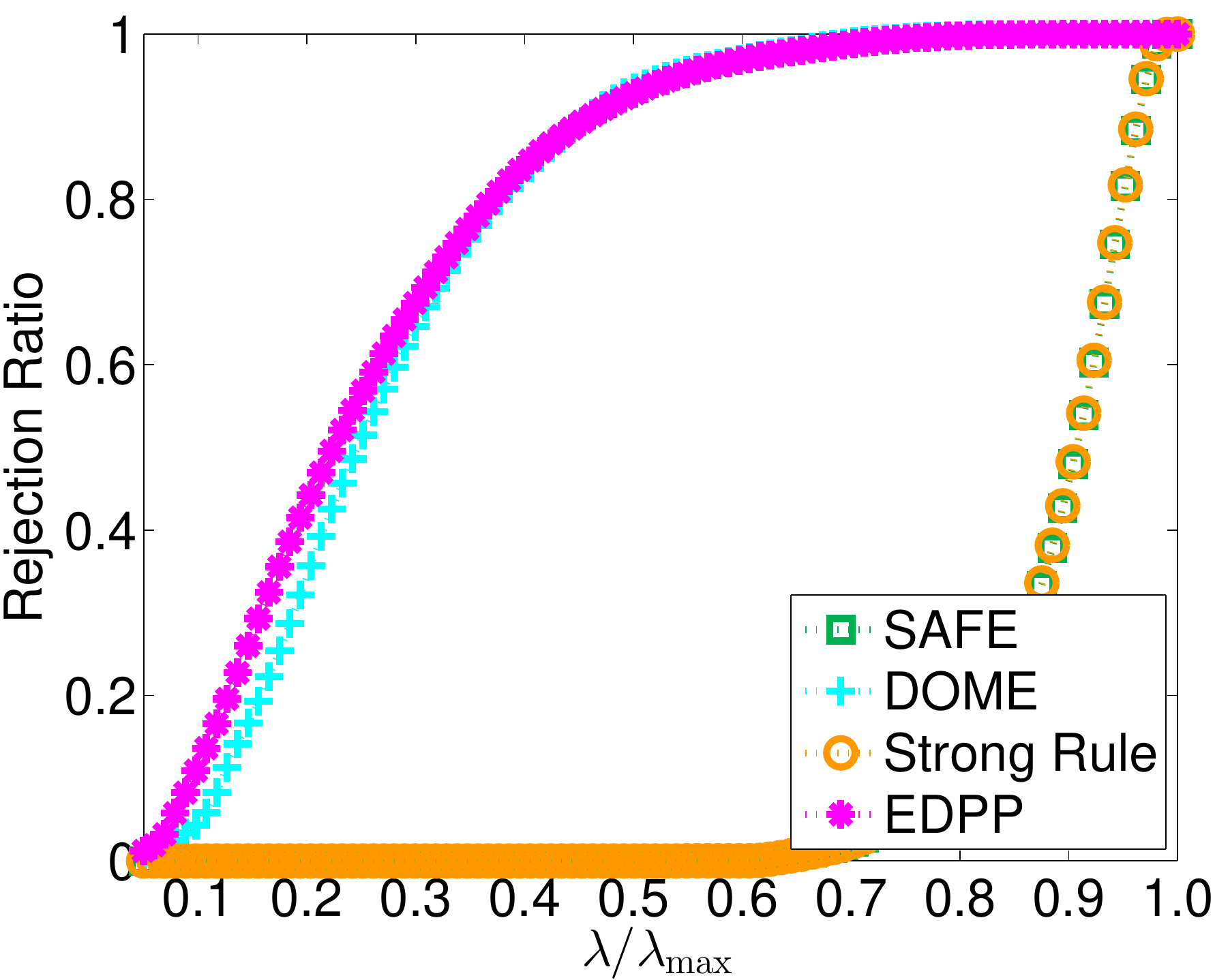}
			}
			\subfigure[MNIST, ${\bf X}\in\mathbb{R}^{784\times 50000}$] { \label{fig:mnist_basic}
				\includegraphics[width=0.31\columnwidth]{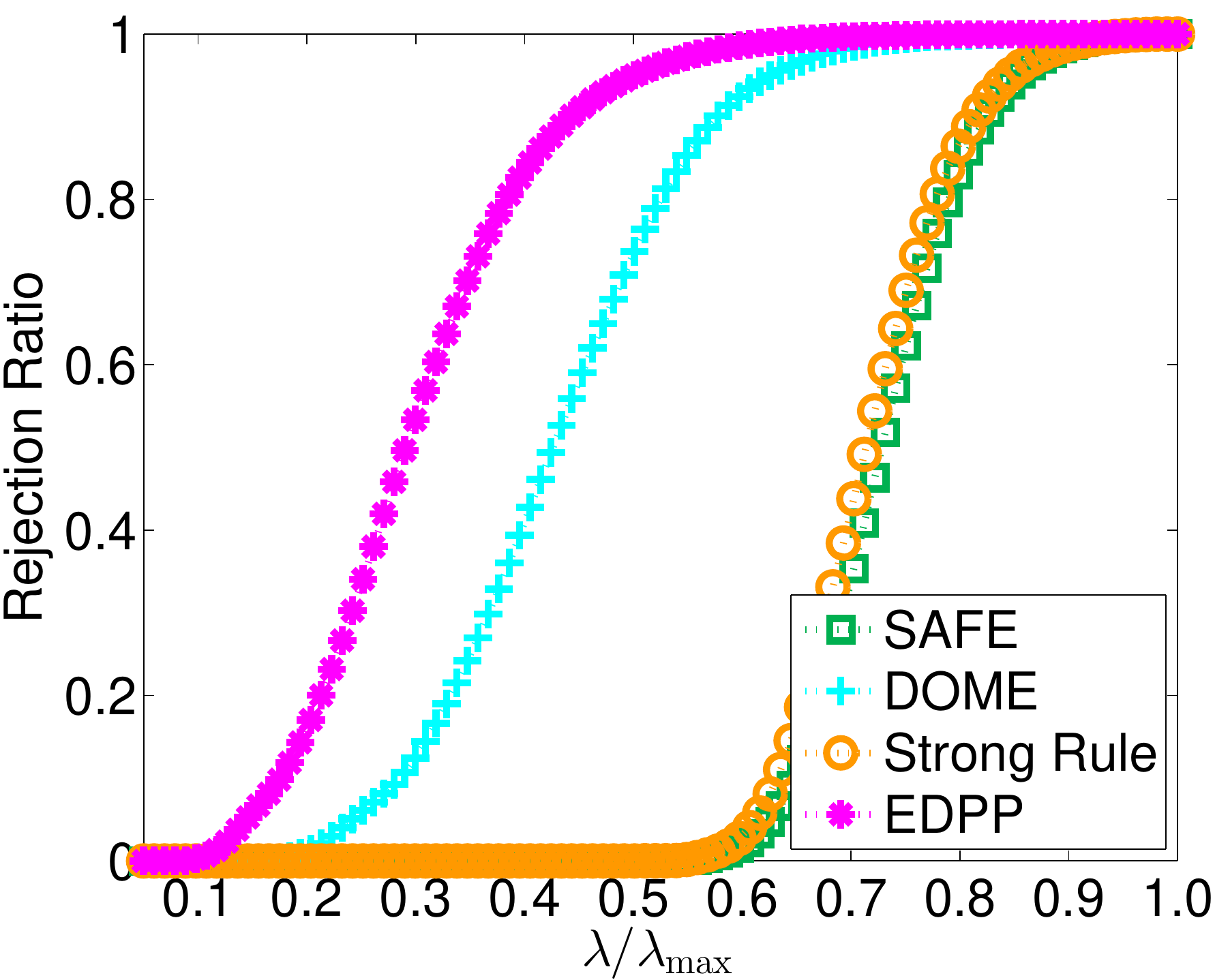}
			}
			\subfigure[COIL-100, ${\bf X}\in\mathbb{R}^{1024\times 7199}$] { \label{fig:olivettifaces}
				\includegraphics[width=0.31\columnwidth]{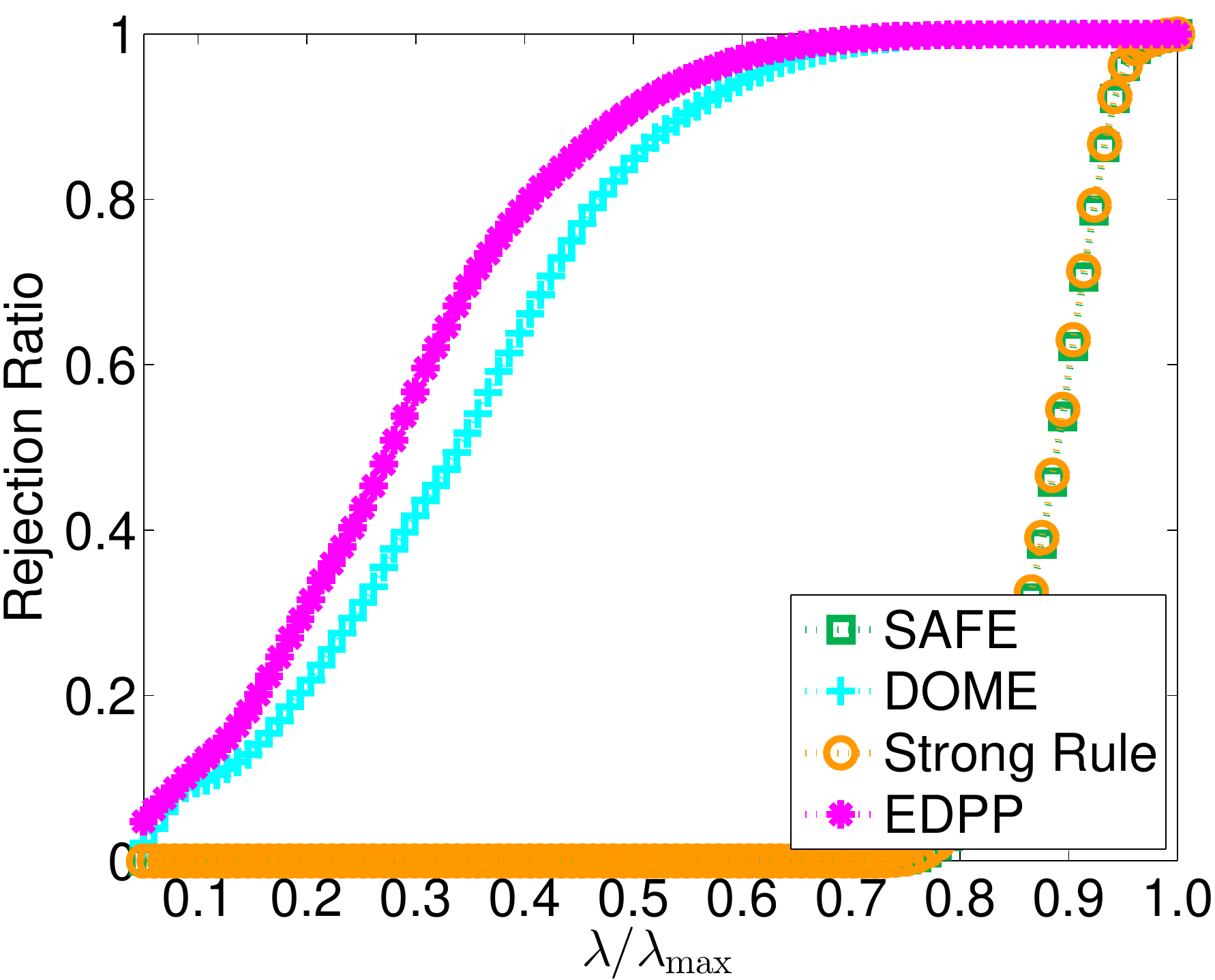}
			}
		}
		\caption{Comparison of basic versions of SAFE, DOME, Strong Rule and EDPP on six real data sets.  }
		\label{fig:lasso_basic}
	\end{figure*}

	\textbf{The Colon Cancer Data Set} This data set contains gene expression information of 22 normal tissues and 40 colon cancer tissues, and each has 2000 gene expression values.
	
	\vspace{2mm}
	\textbf{The Lung Cancer Data Set} This data set contains gene expression information of 186 lung tumors and 17 normal lung specimens. Each specimen has 12600 expression values.
	
	\vspace{2mm}
	\textbf{The COIL-100 Image Data Set} The data set consists of images of 100 objects. The images of each
	object are taken every 5 degree by rotating the object, yielding 72 images per object. The dimension of each image is $32\times 32$. In each trial, we randomly select one image as the response vector and use the remaining ones as the data matrix. We run 100 trials and report the average performance of the screening rules.
	
	\vspace{2mm}
	The description and the experimental settings for the Prostate Cancer data set, the PIE face image data set and the MNIST handwritten digit data set are given in Section \ref{subsubsection:EDPP}.
	
	\figref{fig:lasso_basic} reports the rejection ratios of the basic versions of SAFE, DOME, strong rule and EDPP. We can see that EDPP significantly outperforms the other three screening methods on five of the six data sets, i.e., the Colon Cancer, Lung Cancer, Prostate Cancer, MNIST, and COIL-100 data sets. On the PIE face image data set, EDPP and DOME provide similar performance and both significantly outperform SAFE and strong rule.
	
	However, as pointed out by \citet{Tibshirani11}, the real strength of screening methods stems from their sequential versions. The reason is because the optimal parameter value is unknown in real applications. Typical approaches for model selection usually involve solving the Lasso problems many times along a sequence of parameter values. Thus, the sequential screening methods are more suitable in facilitating the aforementioned scenario and more useful than their basic-version counterparts in practice \citep{Tibshirani11}.

	\subsubsection{Evaluation of the Sequential EDPP Rule}\label{subsubsection:exp_seq_Lasso}
	In this section, we compare the performance of the sequential versions of SAFE, strong rule and EDPP by the rejection ratio and speedup. We first perform experiments on two synthetic data sets. We then apply the three screening rules to six real data sets.
	
	\vspace{2mm}
	\noindent\textbf{Synthetic Data Sets}
	
	First, we perform experiments on several synthetic problems, which have been commonly used in the sparse learning literature \citep{Bondell2008,Zou2005,Tibshirani1996}. We simulate data from the true model
	\begin{align}\label{eqn:regression_model}
		{\bf y}={\bf X}\beta^*+\sigma\epsilon,\hspace{2mm}\epsilon\sim N(0,1).
	\end{align}
	We generate two data sets with $250\times 10000$ entries: Synthetic 1 and Synthetic 2. For Synthetic 1,  the entries of the data matrix ${\bf X}$ are i.i.d. standard Gaussian with
	pairwise correlation zero, i.e., ${\rm corr}({\bf x}_i,{\bf x}_i)=0$. For Synthetic 2, the entries of the data matrix ${\bf X}$ are drawn from i.i.d. standard Gaussian with pairwise correlation $0.5^{|i-j|}$, i.e., ${\rm corr}({\bf x}_i,{\bf x}_j)=0.5^{|i-j|}$. To generate the response vector ${\bf y}\in\mathbb{R}^{250}$ by the model in (\ref{eqn:regression_model}), we need to set the parameter $\sigma$ and construct the ground truth $\beta^*\in\mathbb{R}^{10000}$. Throughout this section, $\sigma$ is set to be $0.1$. To construct $\beta^*$, we randomly select $\overline{p}$ components which are populated from a uniform $[-1,1]$ distribution,  and set the remaining ones as $0$. After we generate the data matrix ${\bf X}$ and the response vector ${\bf y}$, we run the solver with or without screening rules to solve the Lasso problems along a sequence of $100$ parameter values equally spaced on the $\lambda/\lambda_{\rm max}$ scale from $0.05$ to $1.0$. We then run $100$ trials and report the average performance.
	
	\begin{figure*}[th!]
		\centering{
			\includegraphics[width=0.31\columnwidth]{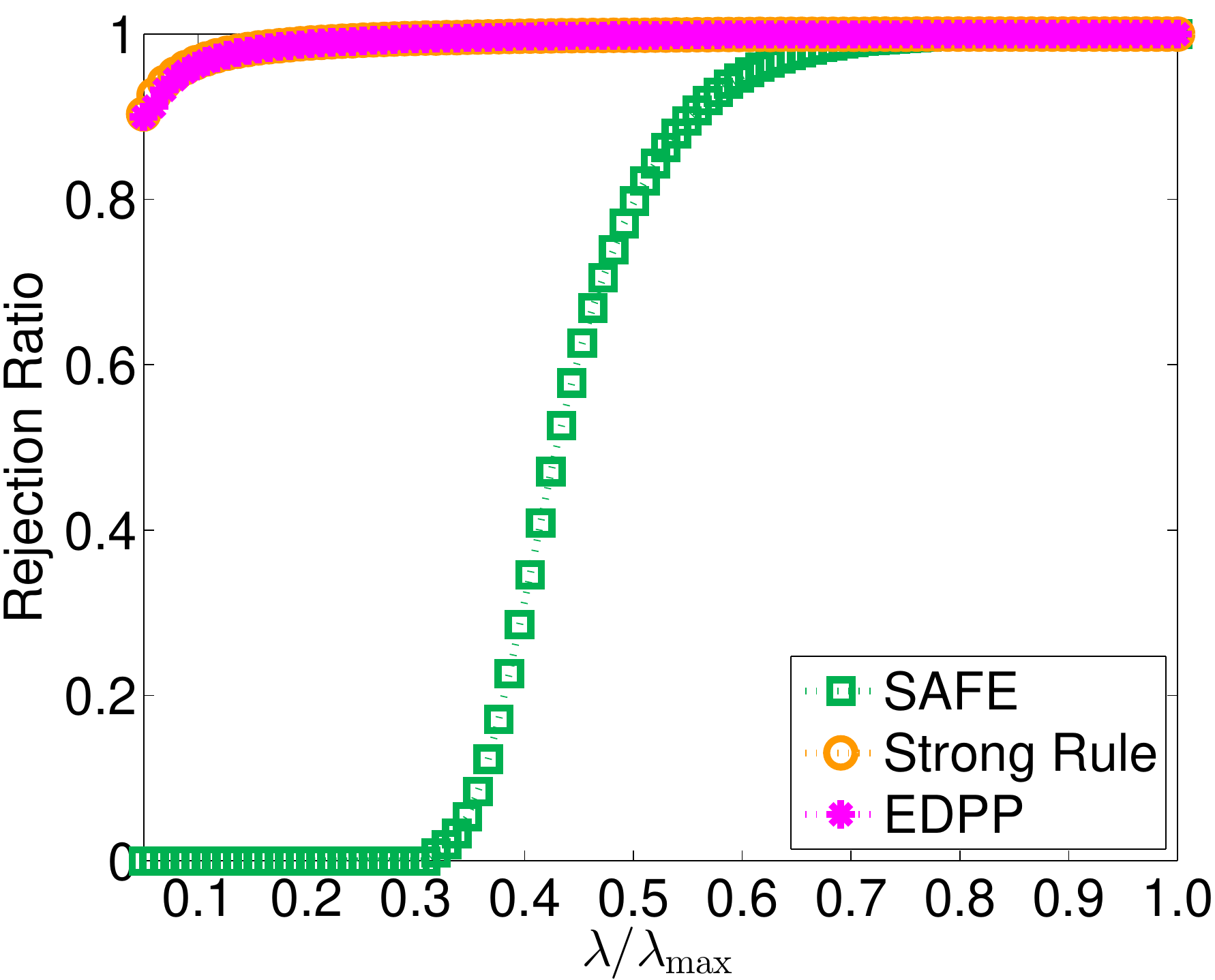}
			\includegraphics[width=0.31\columnwidth]{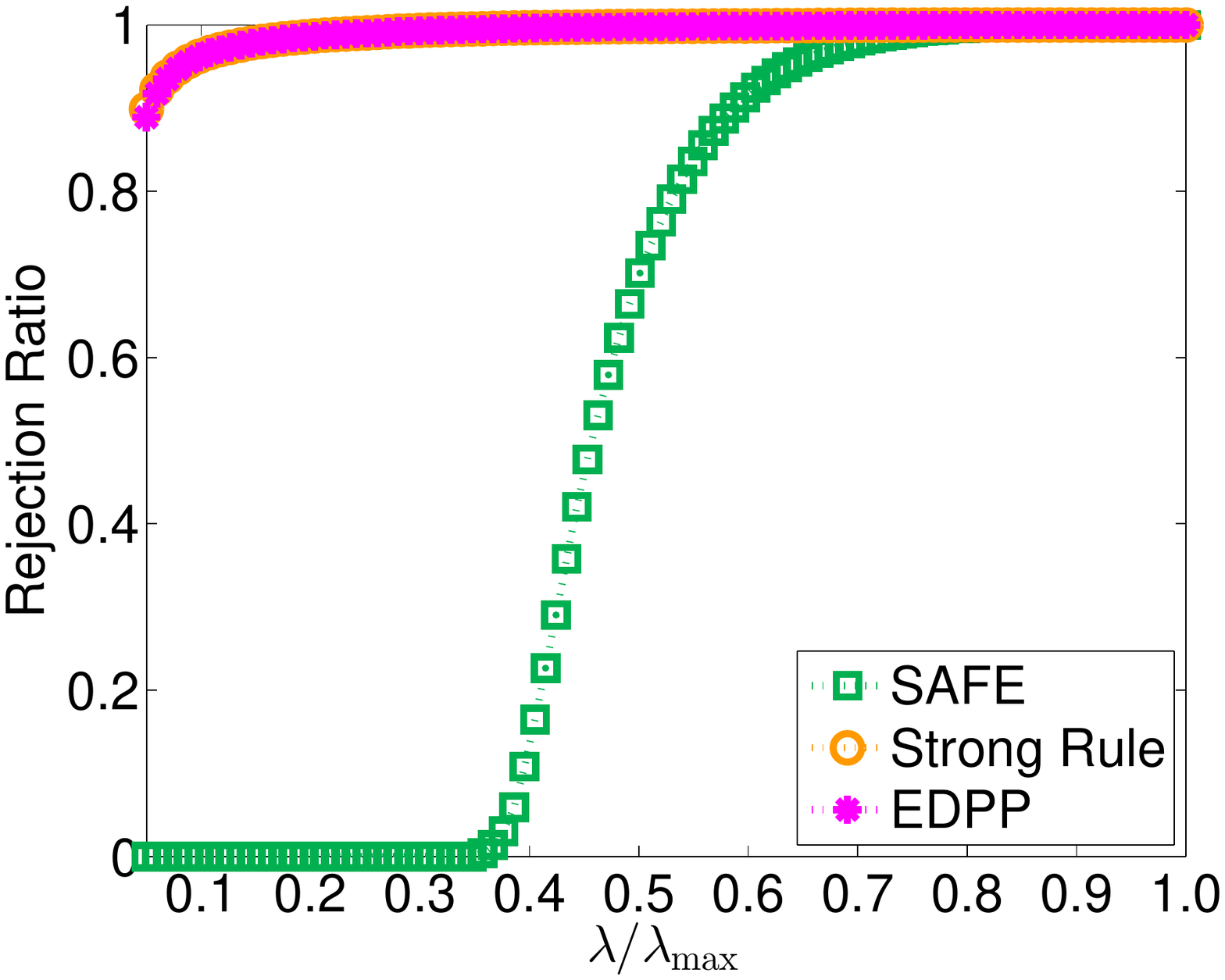}
			\includegraphics[width=0.31\columnwidth]{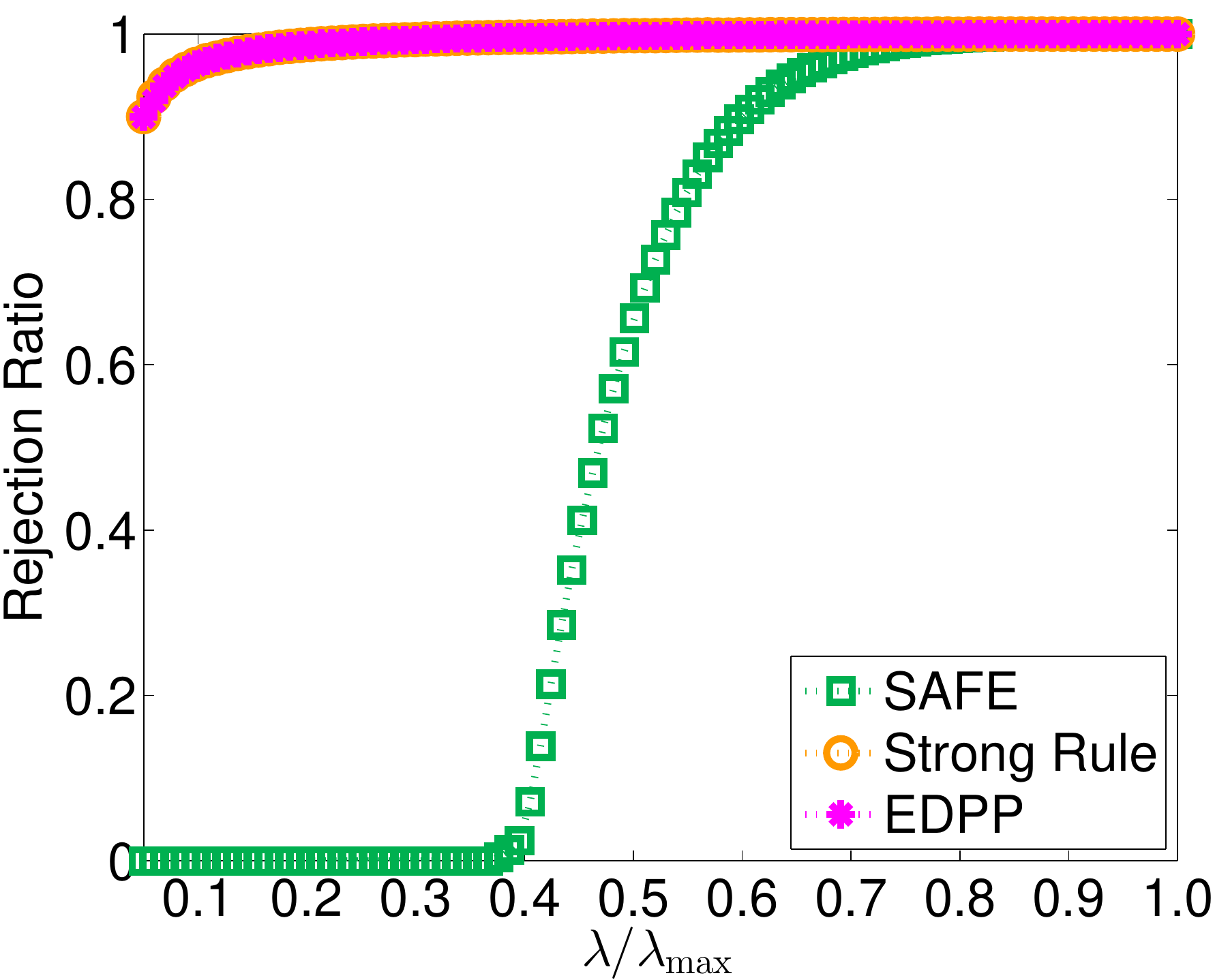}\\
			\subfigure[Synthetic 1, $\overline{p}=100$] { \label{fig:synthetic1_1}
				\includegraphics[width=0.31\columnwidth]{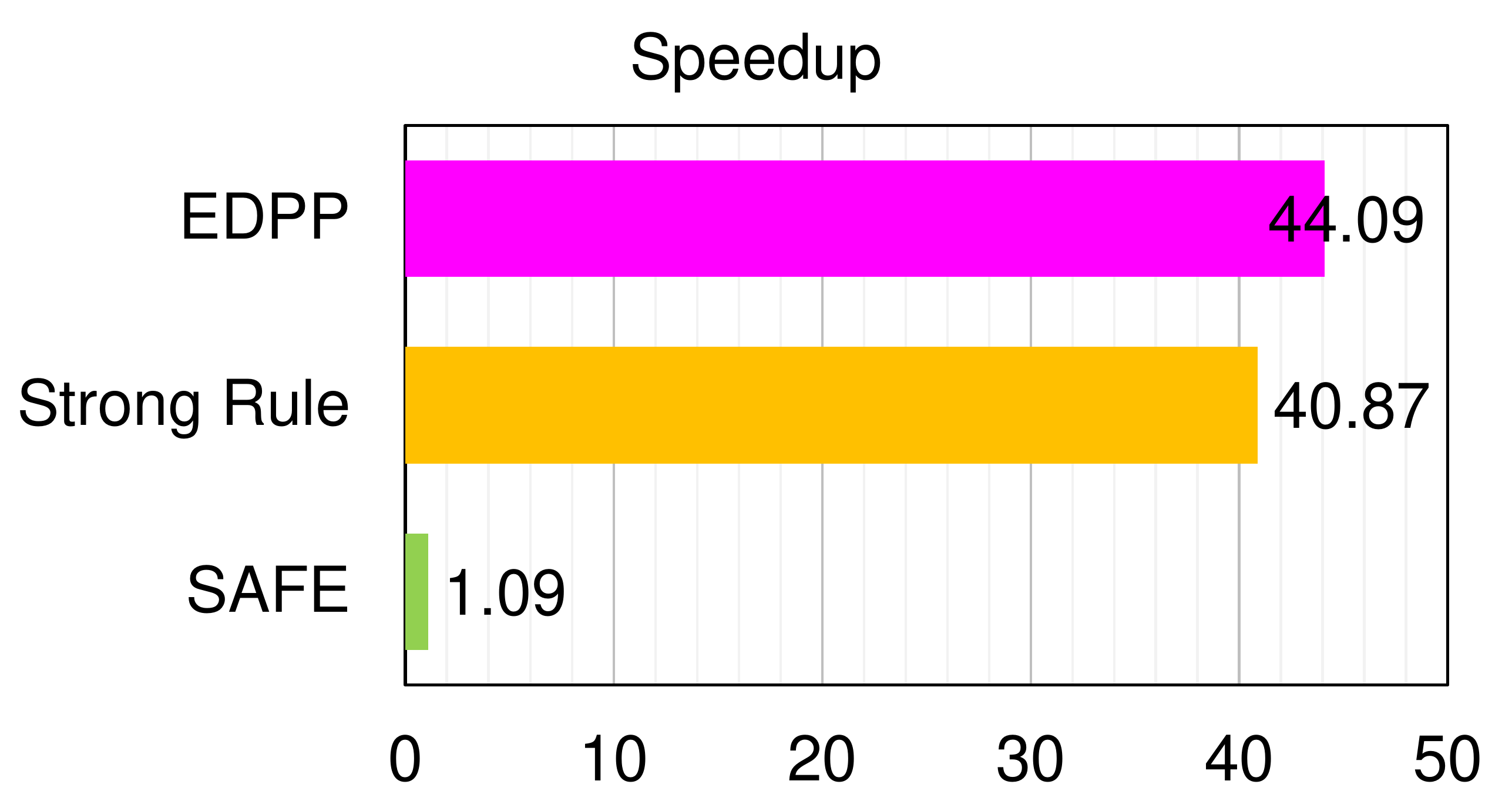}
			}
			\subfigure[Synthetic 1, $\overline{p}=1000$] { \label{fig:synthetic1_2}
				\includegraphics[width=0.31\columnwidth]{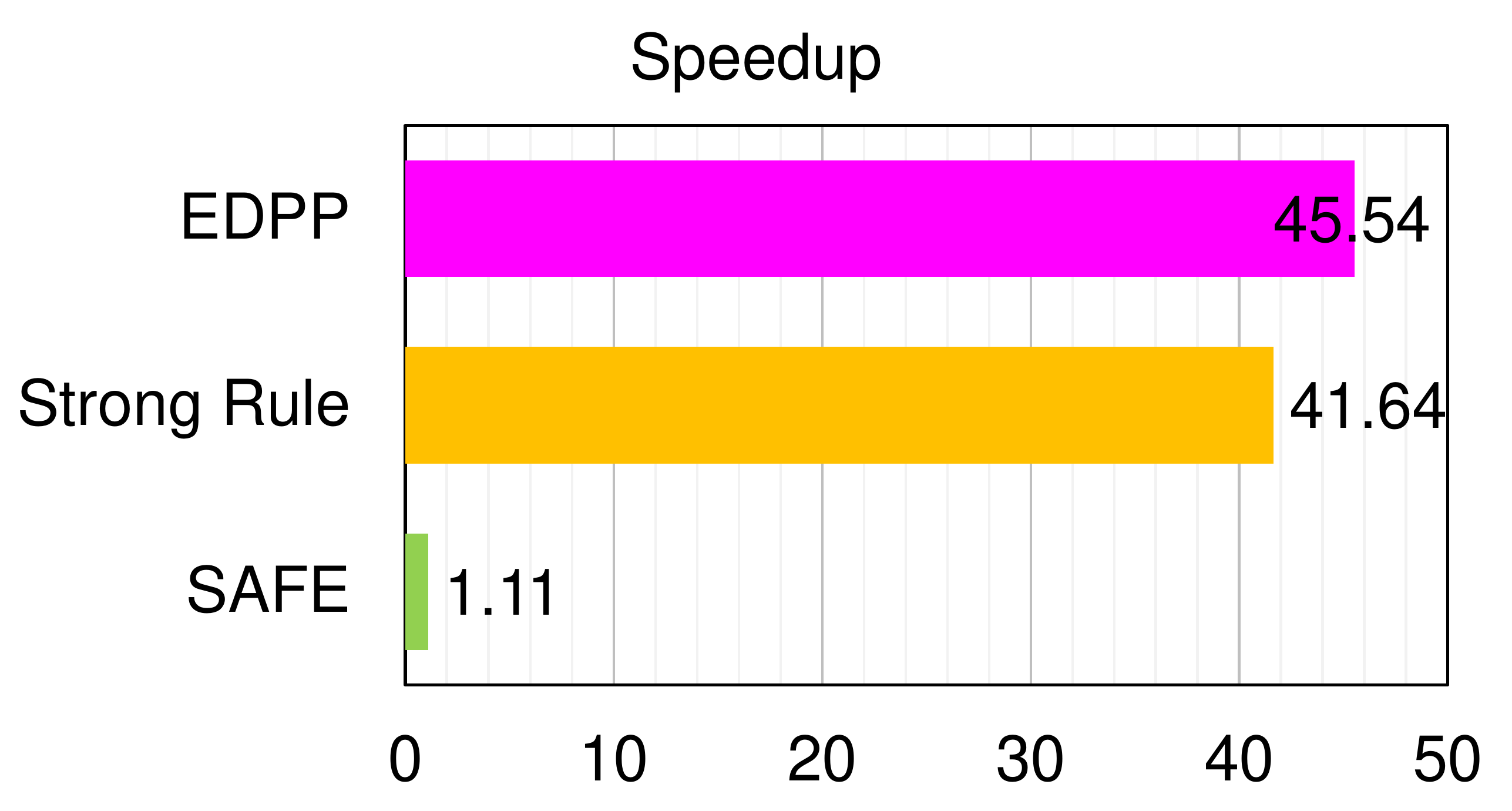}
			}
			\subfigure[Synthetic 1, $\overline{p}=5000$] { \label{fig:synthetic1_3}
				\includegraphics[width=0.31\columnwidth]{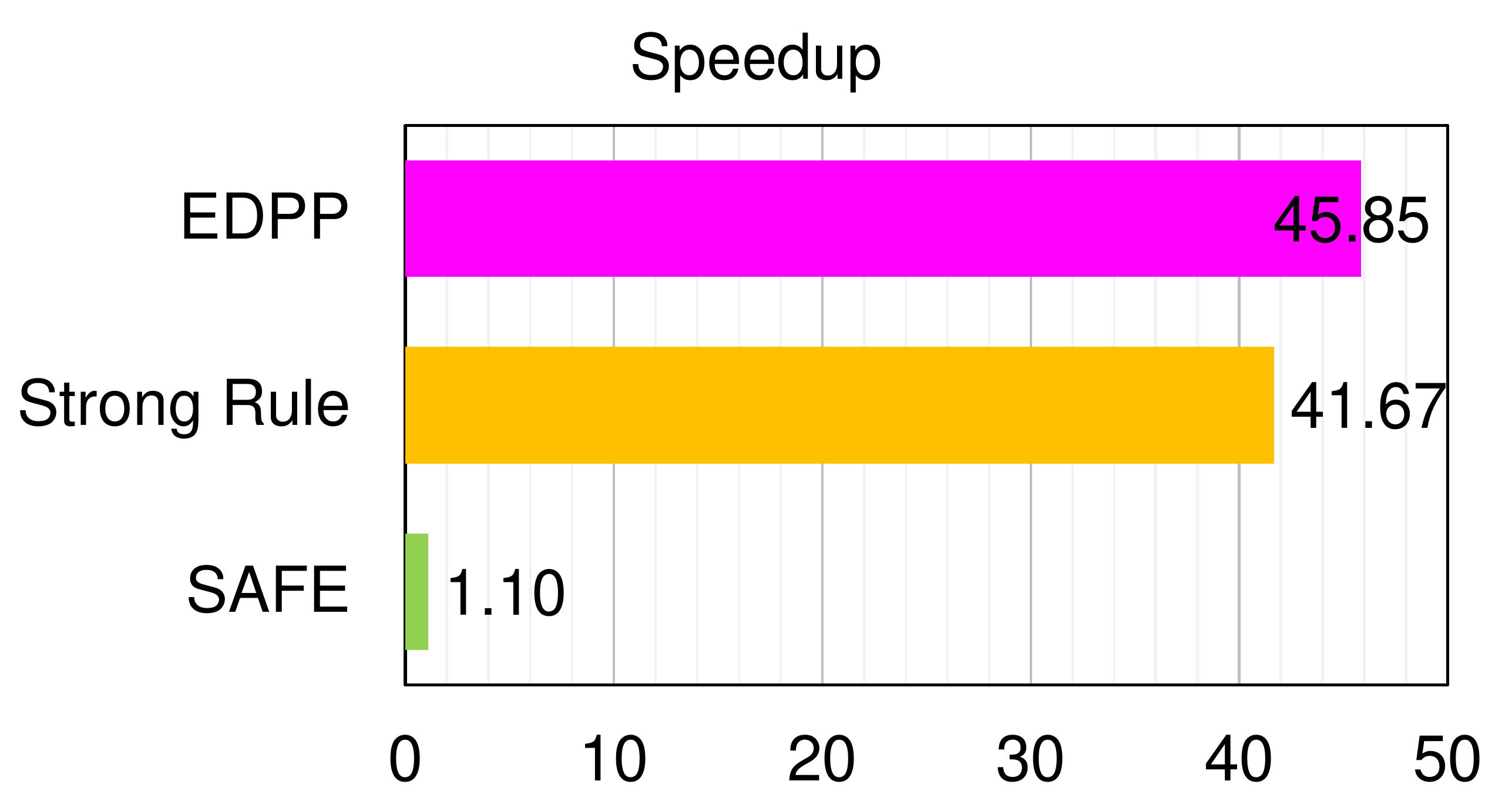}
			}\\
			\vspace{5mm}
			\includegraphics[width=0.31\columnwidth]{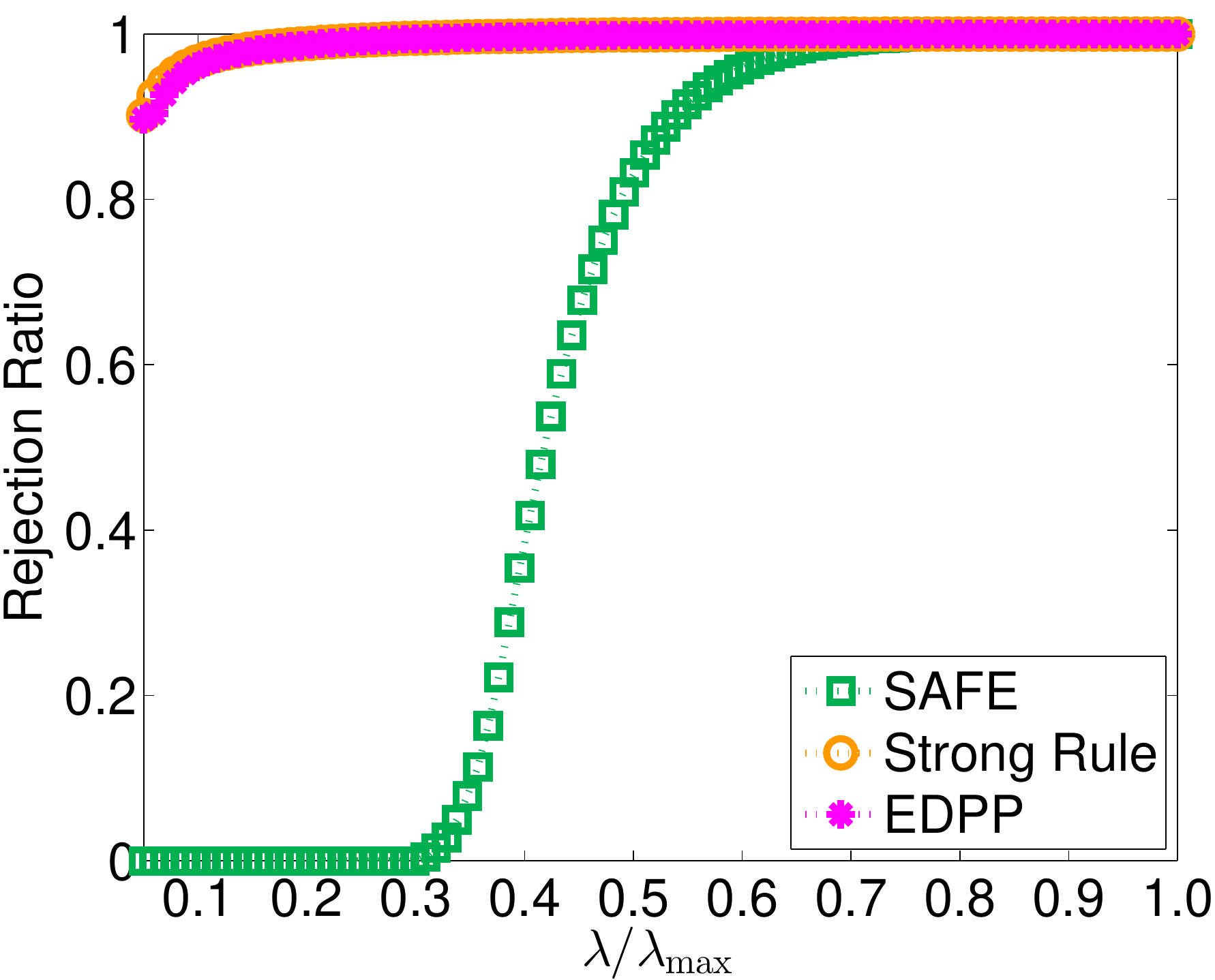}
			\includegraphics[width=0.31\columnwidth]{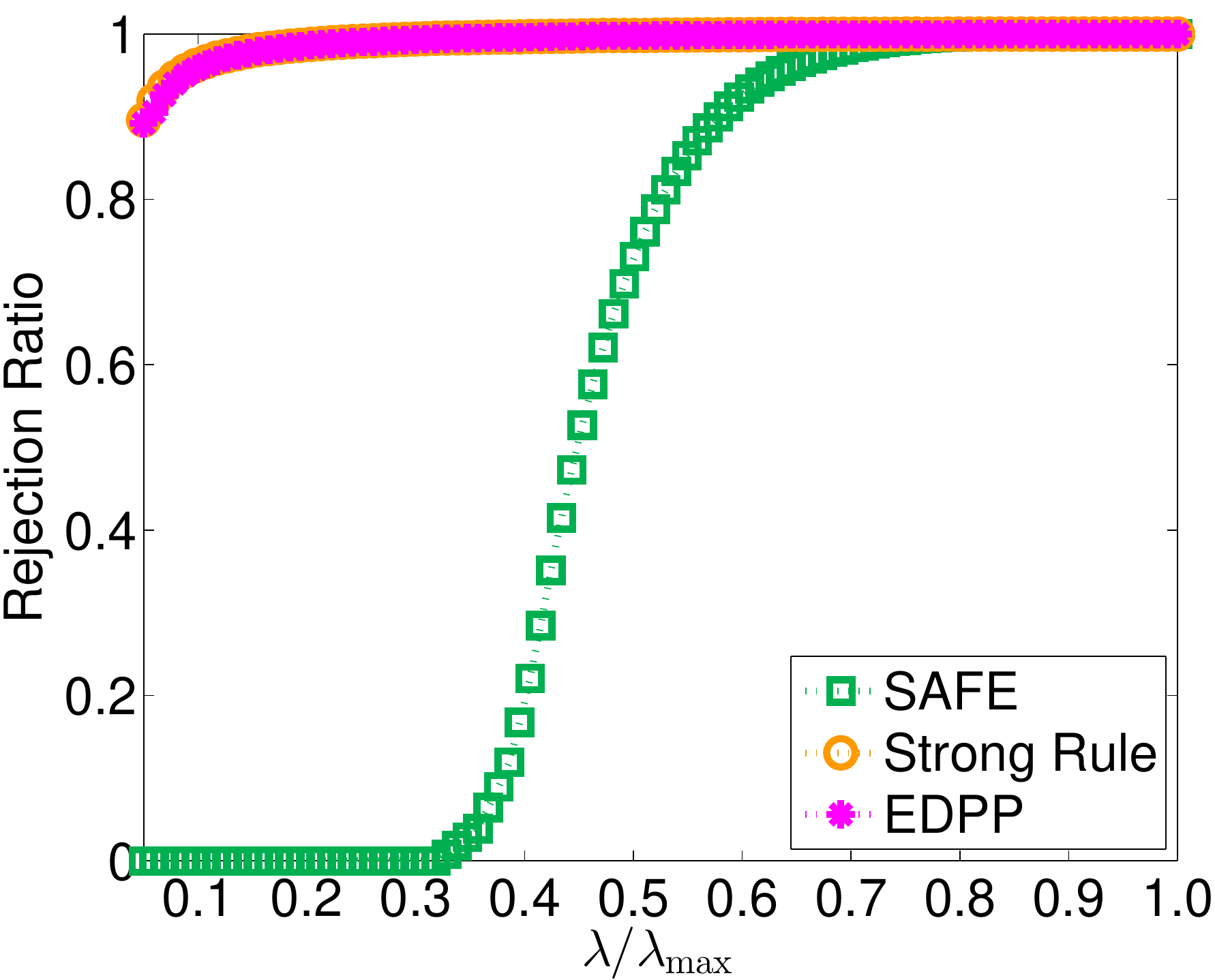}
			\includegraphics[width=0.31\columnwidth]{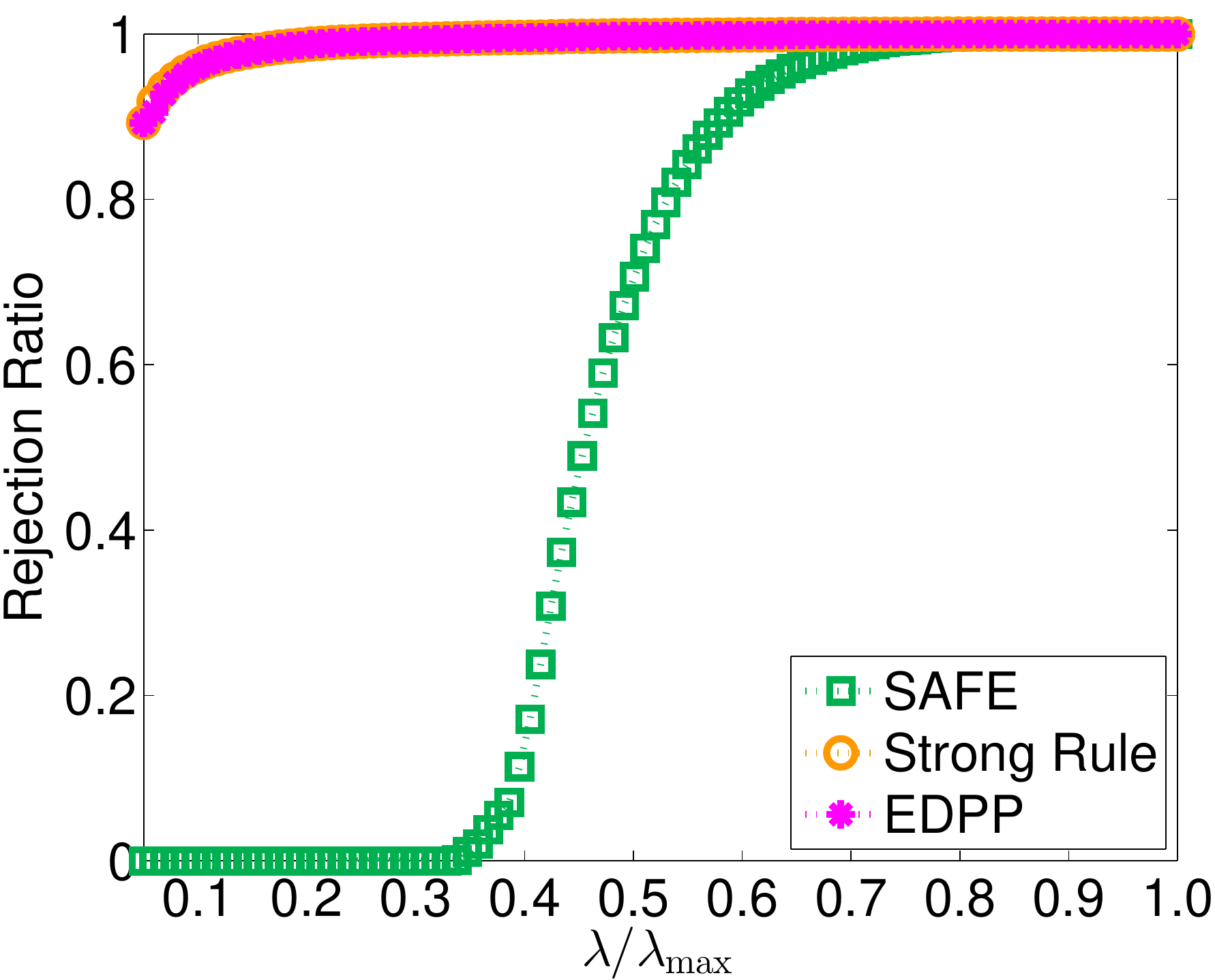}\\
			\subfigure[Synthetic 2, $\overline{p}=100$] { \label{fig:synthetic2_1}
				\includegraphics[width=0.31\columnwidth]{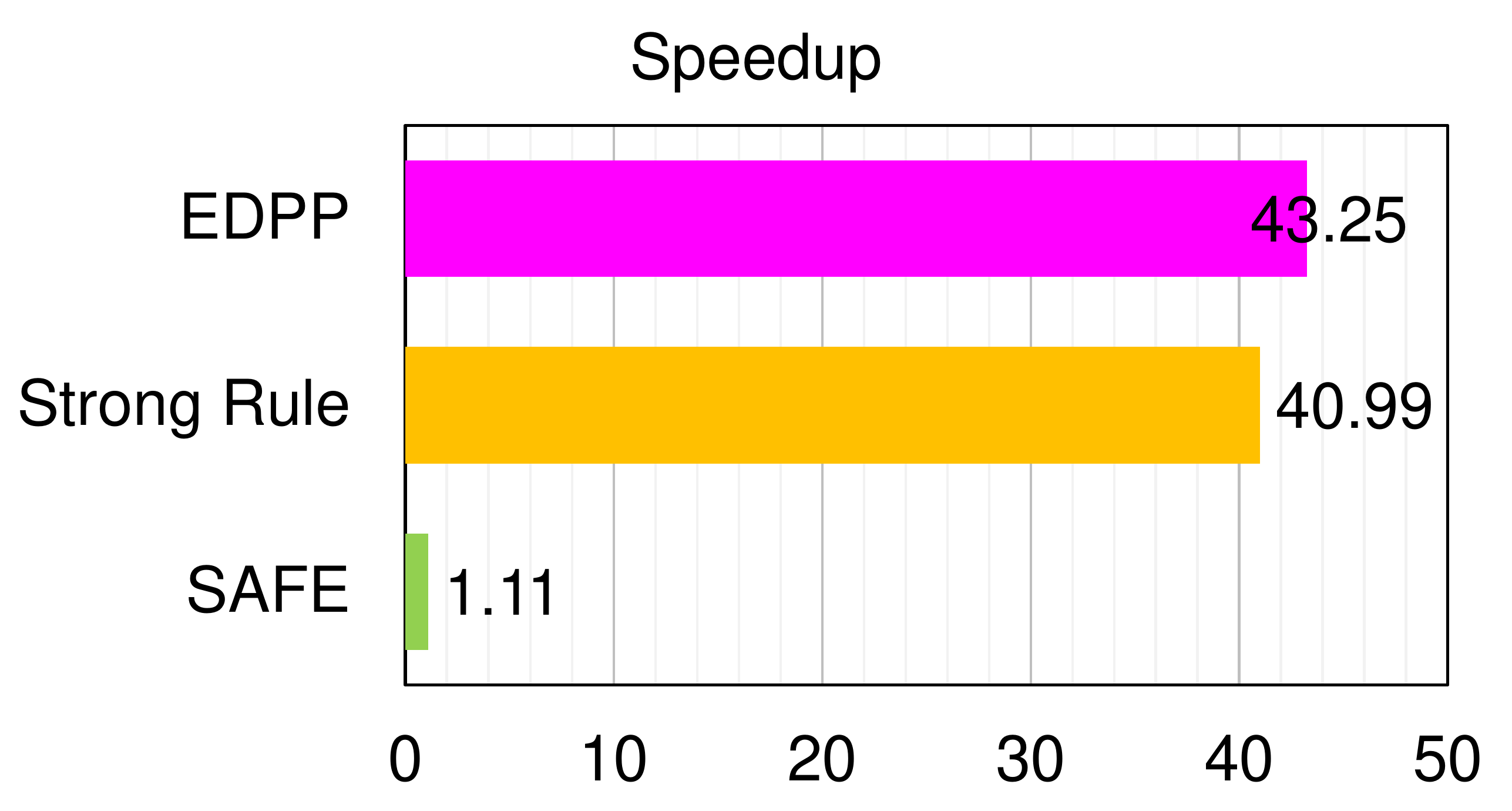}
			}
			\subfigure[Synthetic 2, $\overline{p}=1000$] { \label{fig:synthetic2_2}
				\includegraphics[width=0.31\columnwidth]{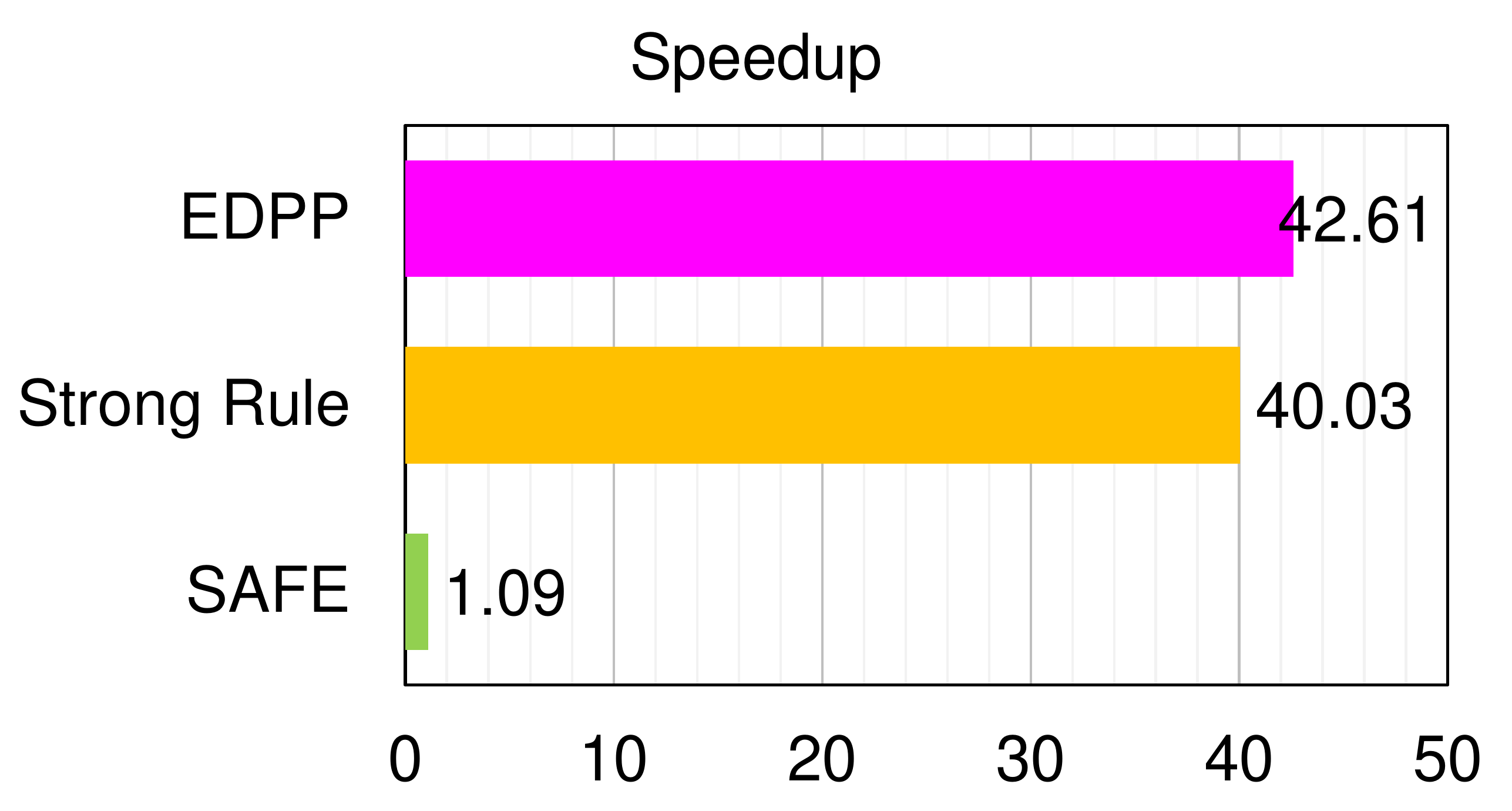}
			}
			\subfigure[Synthetic 2, $\overline{p}=5000$] { \label{fig:synthetic2_3}
				\includegraphics[width=0.31\columnwidth]{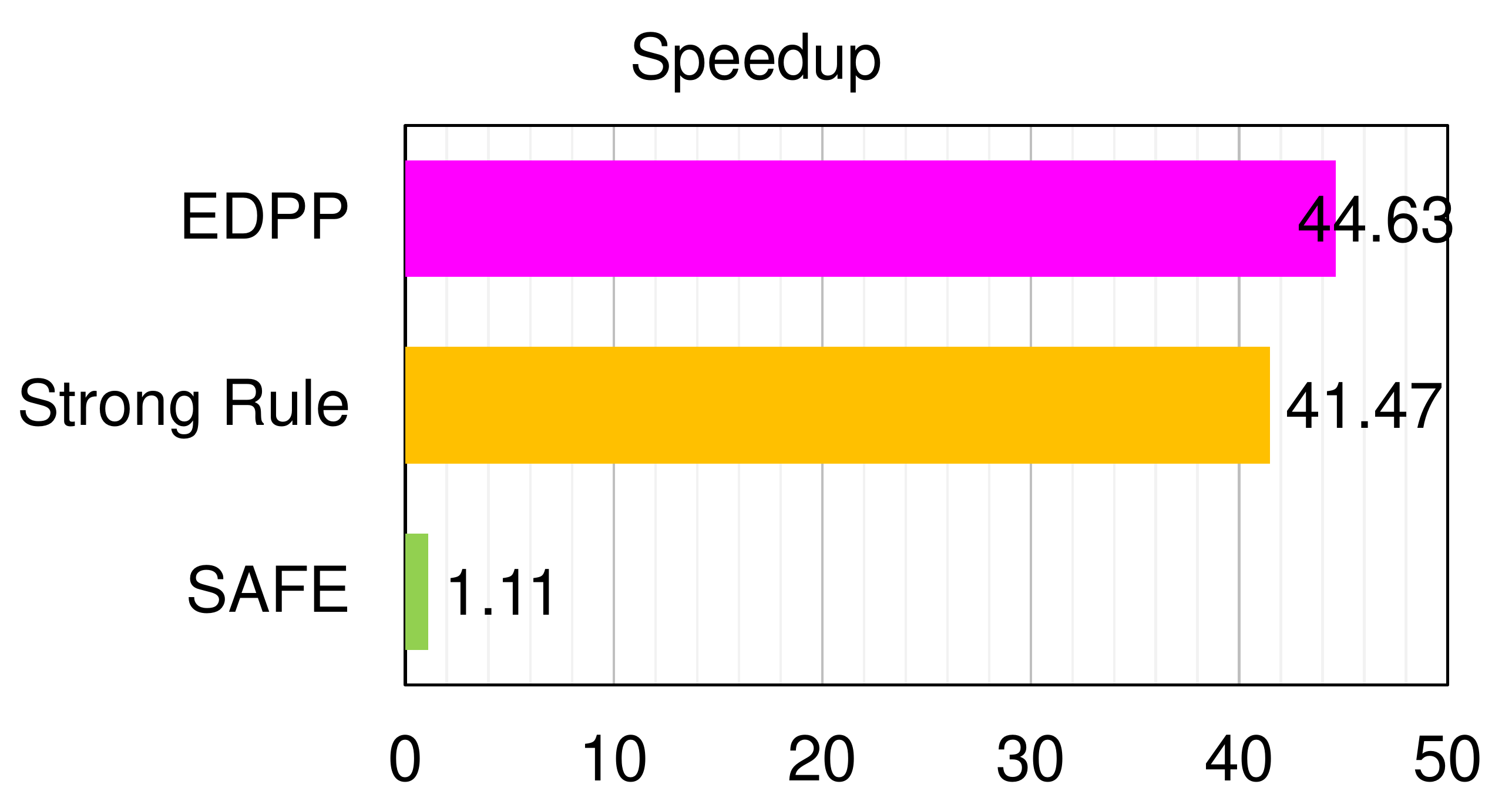}
			}
		}
		\caption{Comparison of SAFE, Strong Rule and EDPP on two synthetic datasets with different numbers of nonzero components of the groud truth.  }
		\label{fig:lasso_synthetic}
	\end{figure*}
	
	We first apply the screening rules, i.e., SAFE, strong rule and EDPP to Synthetic 1 with $\overline{p}=100,1000,5000$ respectively. \figref{fig:synthetic1_1}, \figref{fig:synthetic1_2} and \figref{fig:synthetic1_3} present the corresponding rejection ratios and speedup of SAFE, strong rule and EDPP. We can see that the rejection ratios of strong rule and EDPP are comparable to each other, and both of them are more effective in discarding inactive features than SAFE. In terms of the speedup, EDPP provides better performance than strong rule. The reason is because strong rule is a heuristic screening method, i.e., it may mistakenly discard active features which have nonzero components in the solution. Thus, strong rule needs to check the KKT conditions to ensure the correctness of the screening result. In contrast, the EDPP rule does not need to check the KKT conditions since the discarded features are guaranteed to be absent from the resulting sparse representation. From the last two columns of Table \ref{table:lasso_synthetic_time}, we can observe that the running time of strong rule is about twice of that of EDPP.
	
	\figref{fig:synthetic2_1}, \figref{fig:synthetic2_2} and \figref{fig:synthetic2_3} present the rejection ratios and speedup of SAFE, strong rule and EDPP on Synthetic 2 with $\overline{p}=100,1000,5000$ respectively. We can observe patterns similar to Synthetic 1. Clearly, our method, EDPP, is very robust to the variations of the intrinsic structures of the data sets and the sparsity of the ground truth.

	\setlength{\tabcolsep}{.18em}
	\begin{table}
		\begin{center}
			\begin{footnotesize}
				\def\arraystretch{1.25}
				\begin{tabular}{ |l||c||c||c|c|c||c|c|c| }
					\hline
					Data & $\overline{p}$ & solver & SAFE+solver &  Strong Rule+solver & EDPP+solver & SAFE & Strong Rule & EDPP  \\
					\hline\hline
					\multirow{3}{*}{Synthetic 1}  & 100 & 109.01 & 100.09 & 2.67 & 2.47 & 4.60 & 0.65 & 0.36  \\ \cline{2-9}
					& 1000 & 123.60 & 111.32 & 2.97 & 2.71 & 4.59 & 0.66 & 0.37  \\\cline{2-9}
					& 5000 & 124.92 & 113.09 & 3.00 & 2.72 & 4.57 &  0.65 & 0.36 \\\hline\hline
					\multirow{3}{*}{Synthetic 2}  & 100 & 107.50 & 96.94 & 2.62 & 2.49 & 4.61 & 0.67 & 0.37 \\\cline{2-9}
					& 1000 & 113.59 & 104.29 & 2.84 & 2.67 & 4.57 & 0.63 & 0.35  \\\cline{2-9}
					& 5000 & 125.25 & 113.35 & 3.02 & 2.81 & 4.62 & 0.65 &  0.36 \\\hline
				\end{tabular}
			\end{footnotesize}
		\end{center}
		\caption{Running time (in seconds) for solving the Lasso problems along a sequence of $100$ tuning parameter values equally spaced on the scale of ${\lambda}/{\lambda_{\rm max}}$ from $0.05$ to $1$ by (a): the solver \citep{SLEP} (reported in the third column) without screening; (b): the solver combined with different screening methods (reported in the $4^{th}$ to the $6^{th}$ columns).
			The last four columns report the total running time (in seconds) for the screening methods.
		}
		\label{table:lasso_synthetic_time}
	\end{table}
	
	\vspace{2mm}
	\noindent\textbf{Real Data Sets}
	
	In this section, we compare the performance of the EDPP rule with SAFE and strong rule on six real data sets along a sequence of $100$ parameter values equally spaced on the $\lambda/\lambda_{\rm max}$ scale from $0.05$ to $1.0$. The data sets are listed as follows:
	\begin{enumerate}
		\item[a)] Breast Cancer data set \citep{West2001,Shevade2003};
		\item[b)] Leukemia data set \citep{Armstrong2002};
		\item[c)] Prostate Cancer data set \citep{Petricoin2002};
		\item[d)] PIE face image data set \citep{Sim2003,Cai2007};
		\item[e)] MNIST handwritten digit data set \citep{Lecun1998};
		\item[f)] Street View House Number (SVHN) data set \citep{Netzer2001}.
	\end{enumerate}
	We present the rejection ratios and speedup of EDPP, SAFE and strong rule in \figref{fig:lasso_real}. Table \ref{table:lasso_real_time} reports the running time of the solver with or without screening for solving the $100$ Lasso problems, and that of the screening rules.
	
	\begin{figure*}[h!t]
		\centering{
			\includegraphics[width=0.31\columnwidth]{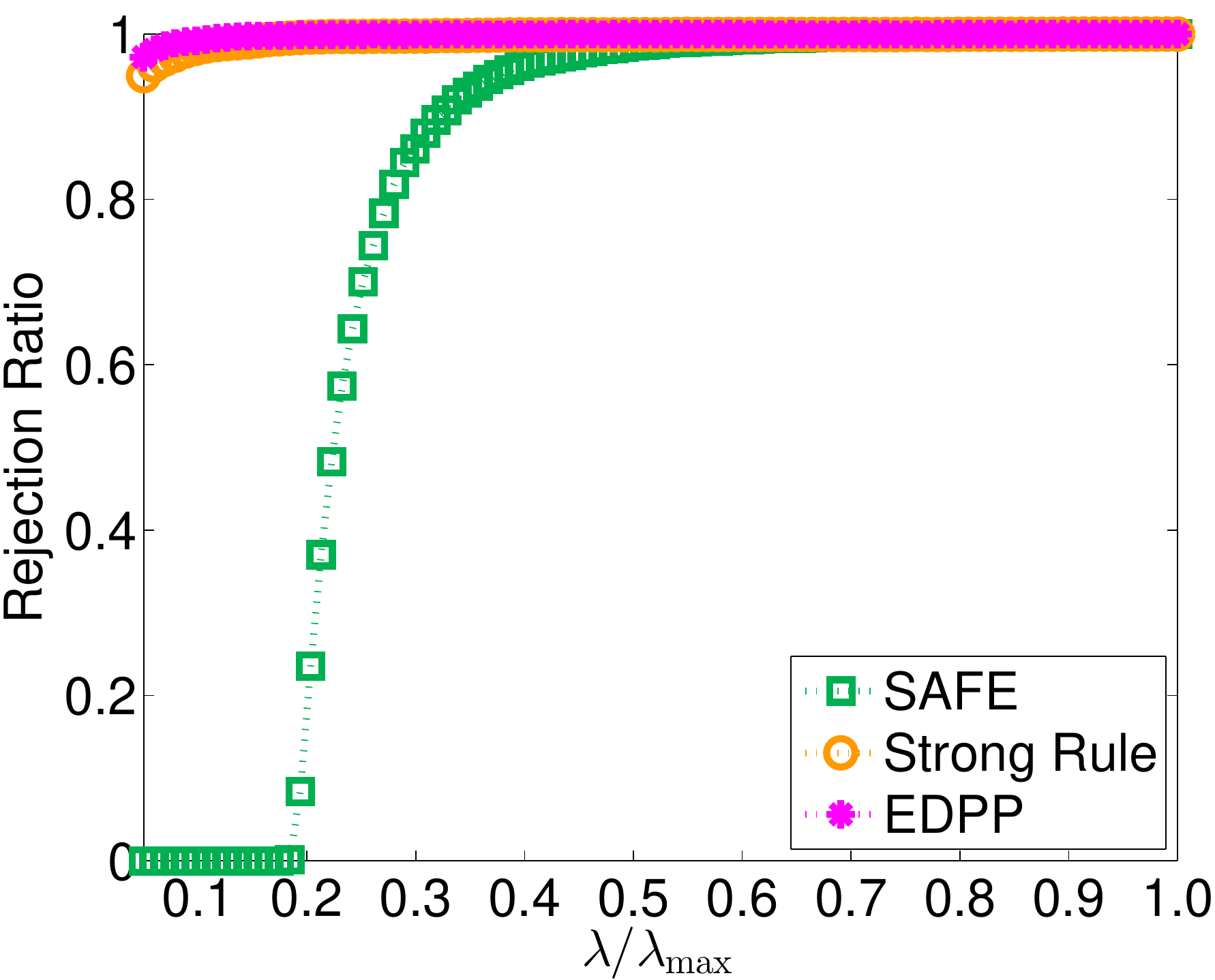}
			\includegraphics[width=0.31\columnwidth]{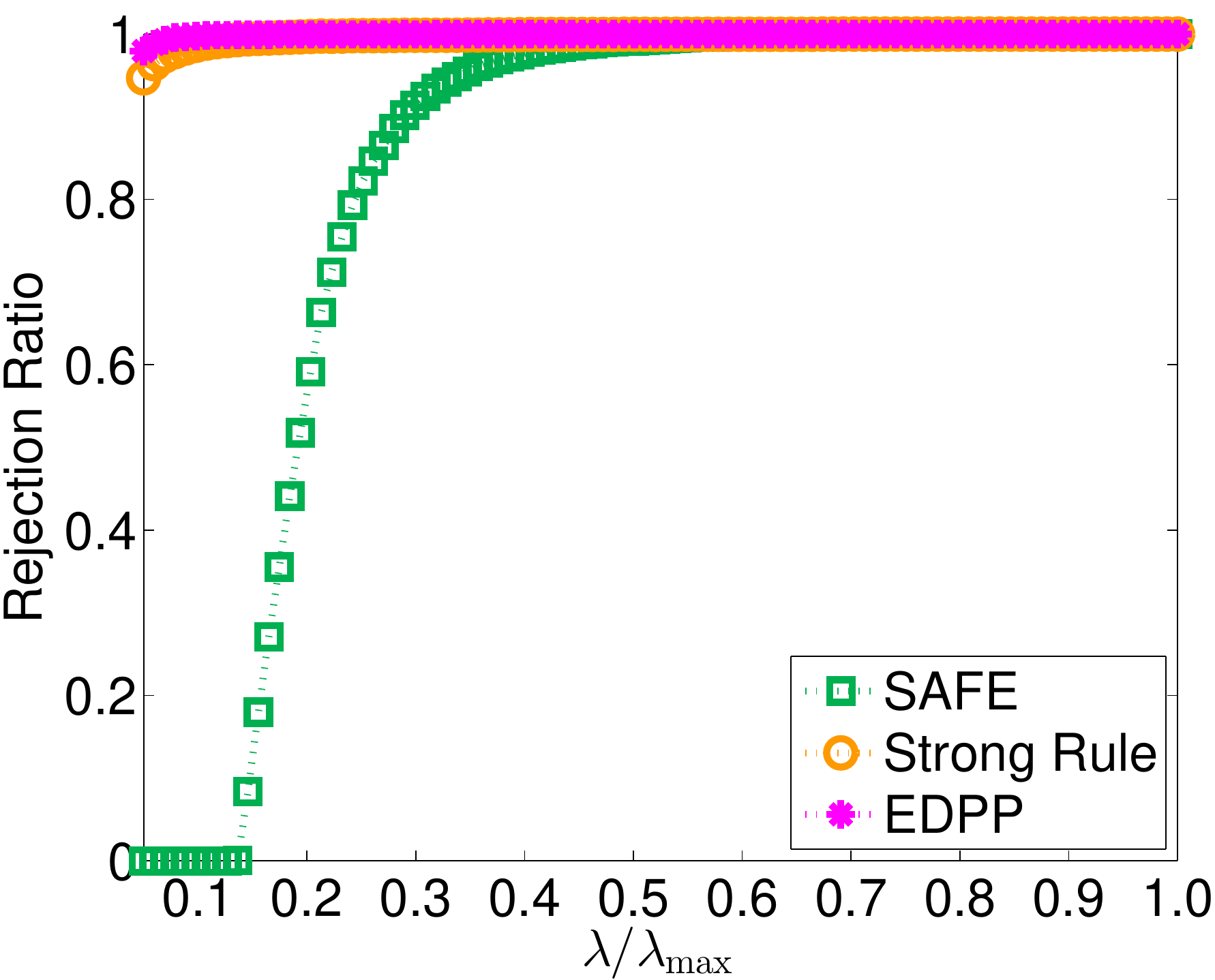}
			\includegraphics[width=0.31\columnwidth]{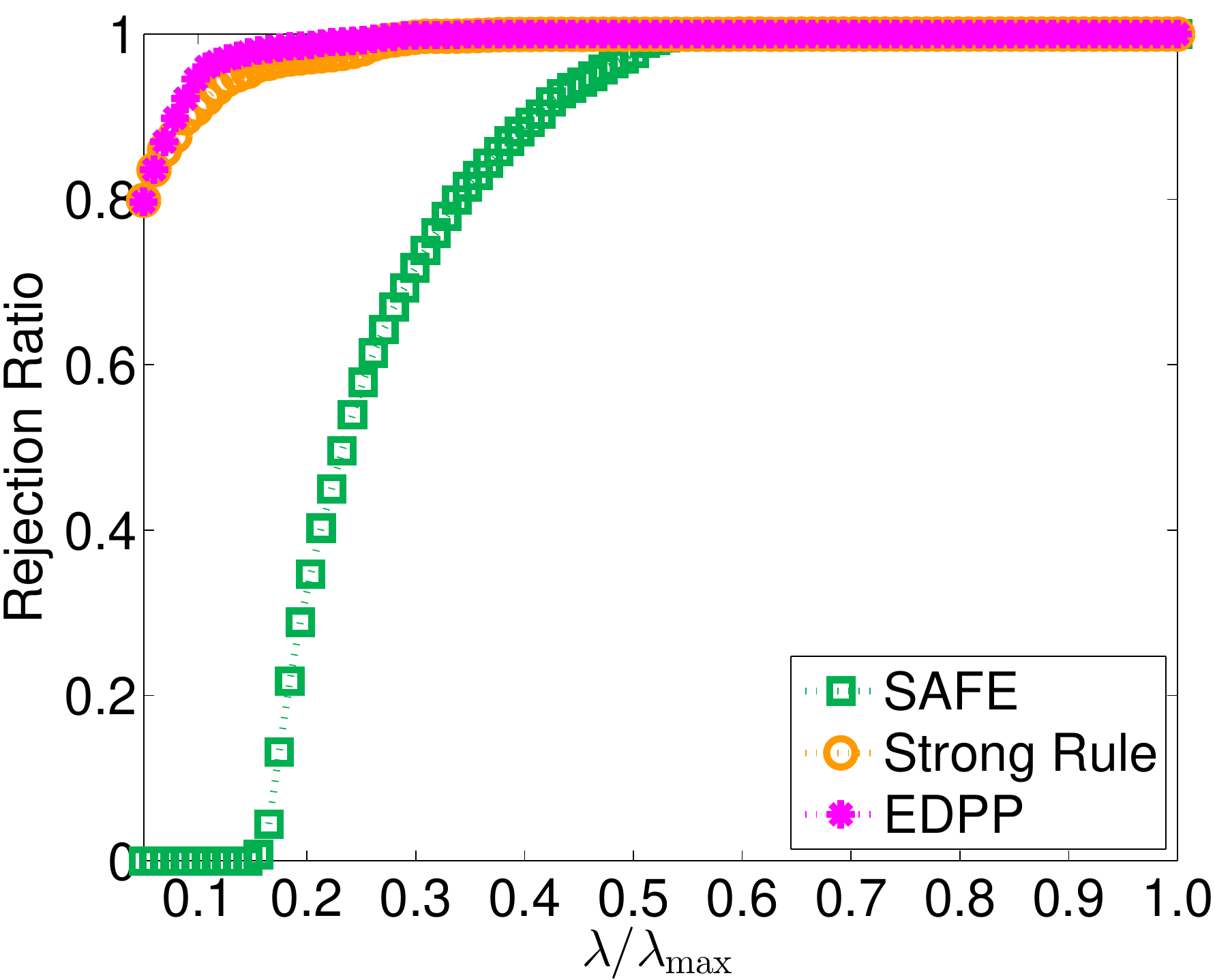}\\
			\subfigure[Breast Cancer, ${\bf X}\in\mathbb{R}^{44\times 7129}$] { \label{fig:breastcancer}
				\includegraphics[width=0.31\columnwidth]{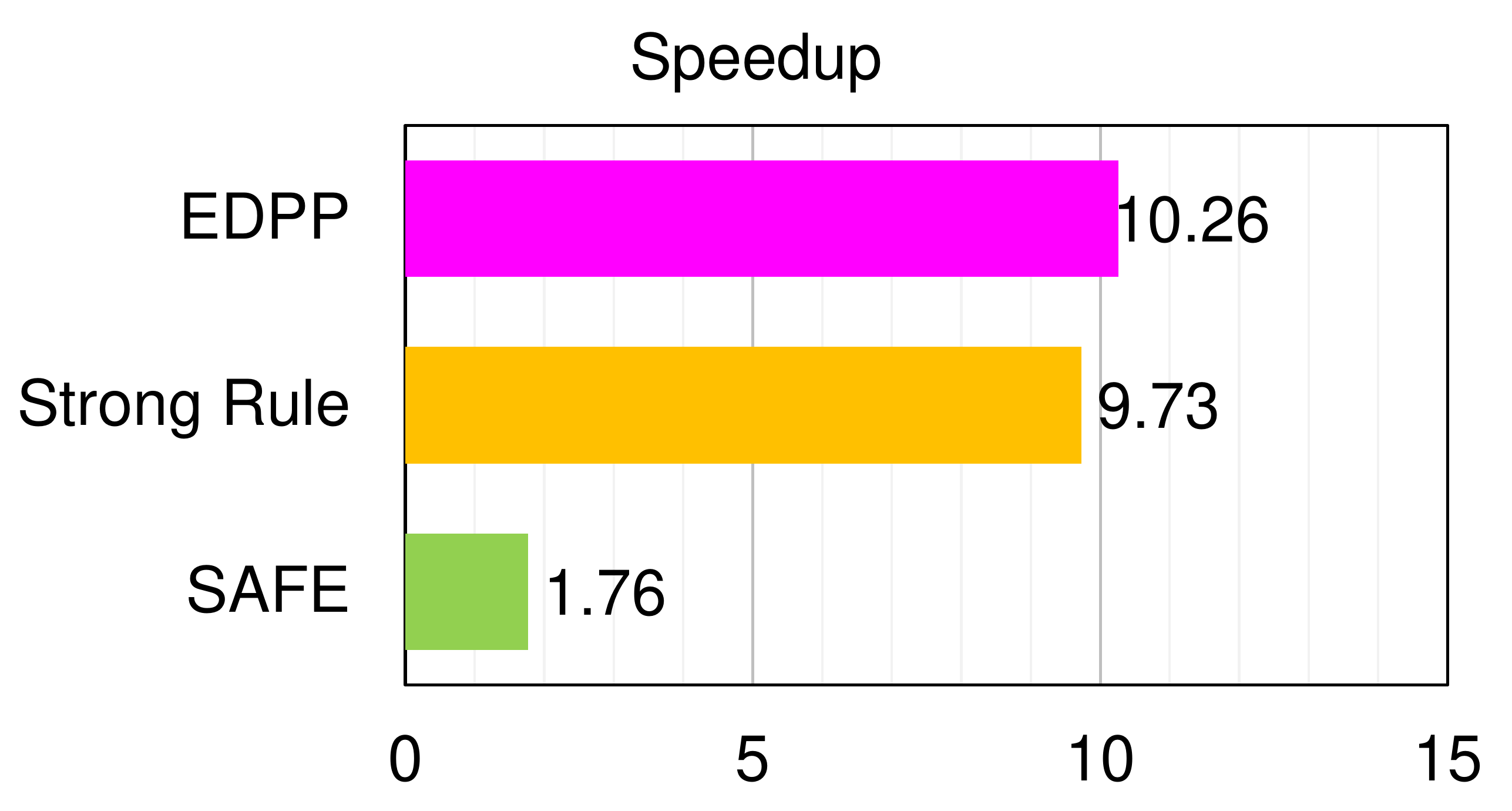}
			}
			\subfigure[Leukemia, ${\bf X}\in\mathbb{R}^{55\times 11225}$] { \label{fig:Leukemia}
				\includegraphics[width=0.31\columnwidth]{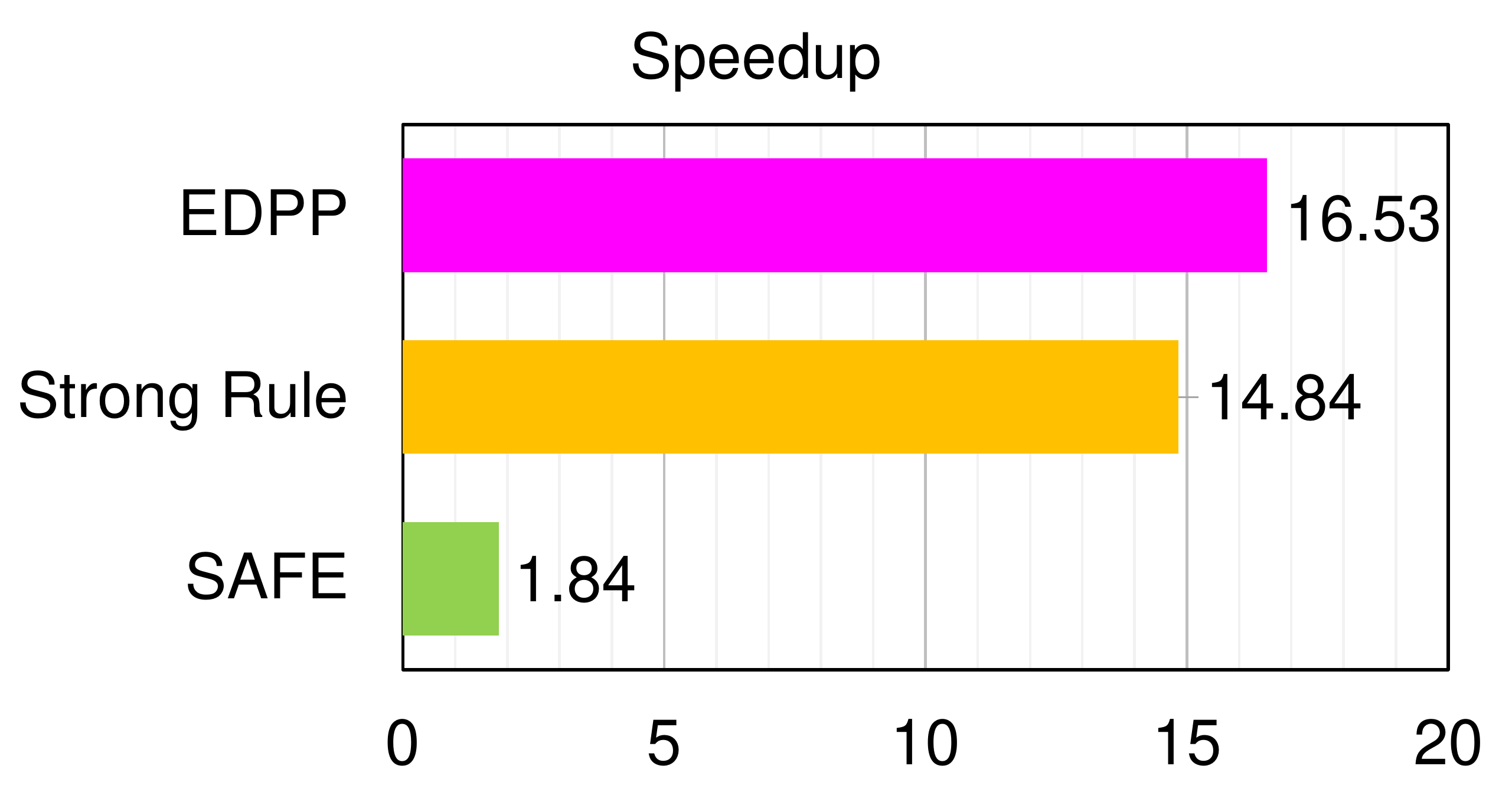}
			}
			\subfigure[{\scriptsize Prostate Cancer}, ${\bf X}\in\mathbb{R}^{132\times 15154}$] { \label{fig:prostatecancer_e2}
				\includegraphics[width=0.31\columnwidth]{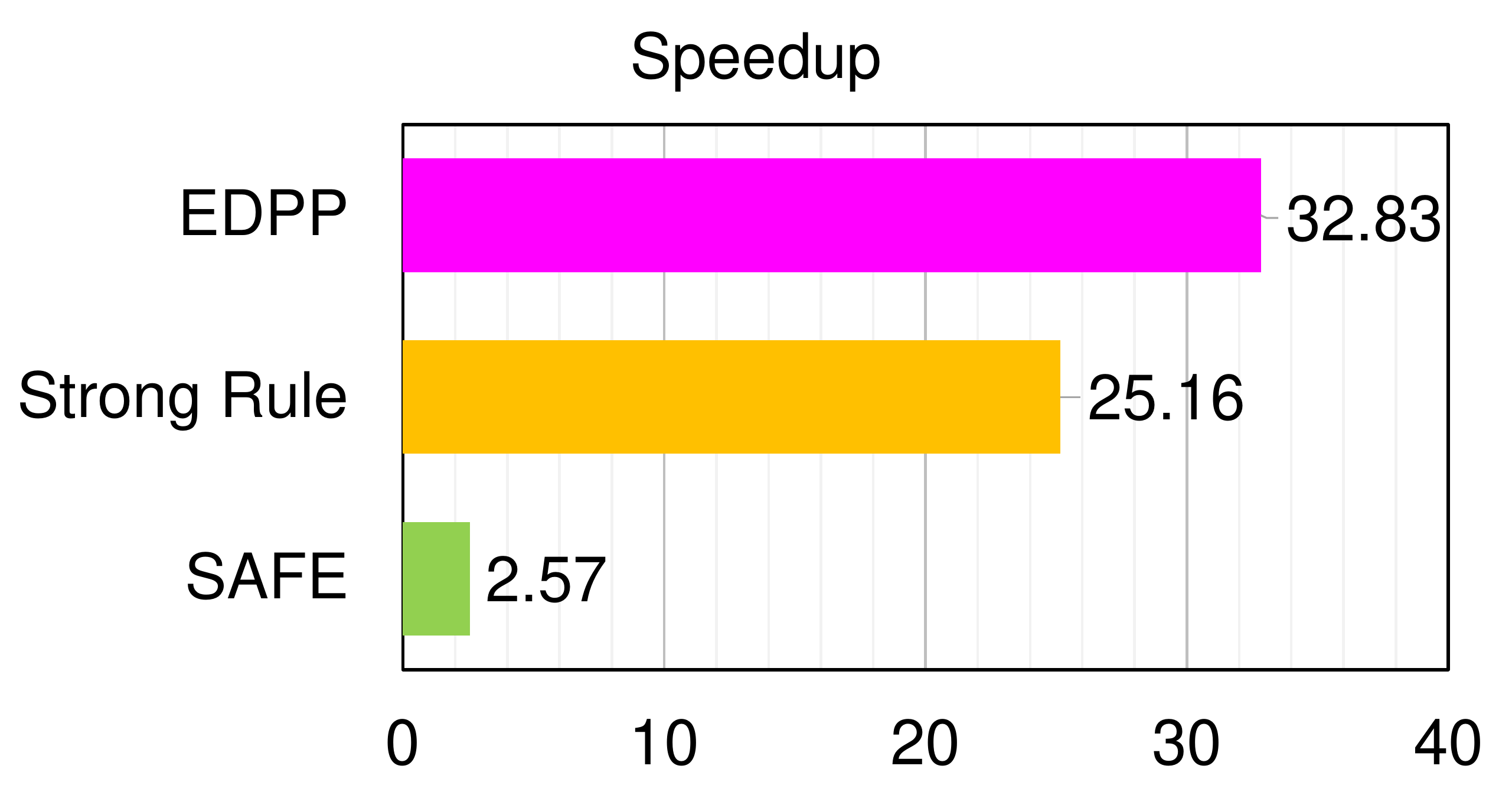}
			}\\
			\vspace{5mm}
			\includegraphics[width=0.31\columnwidth]{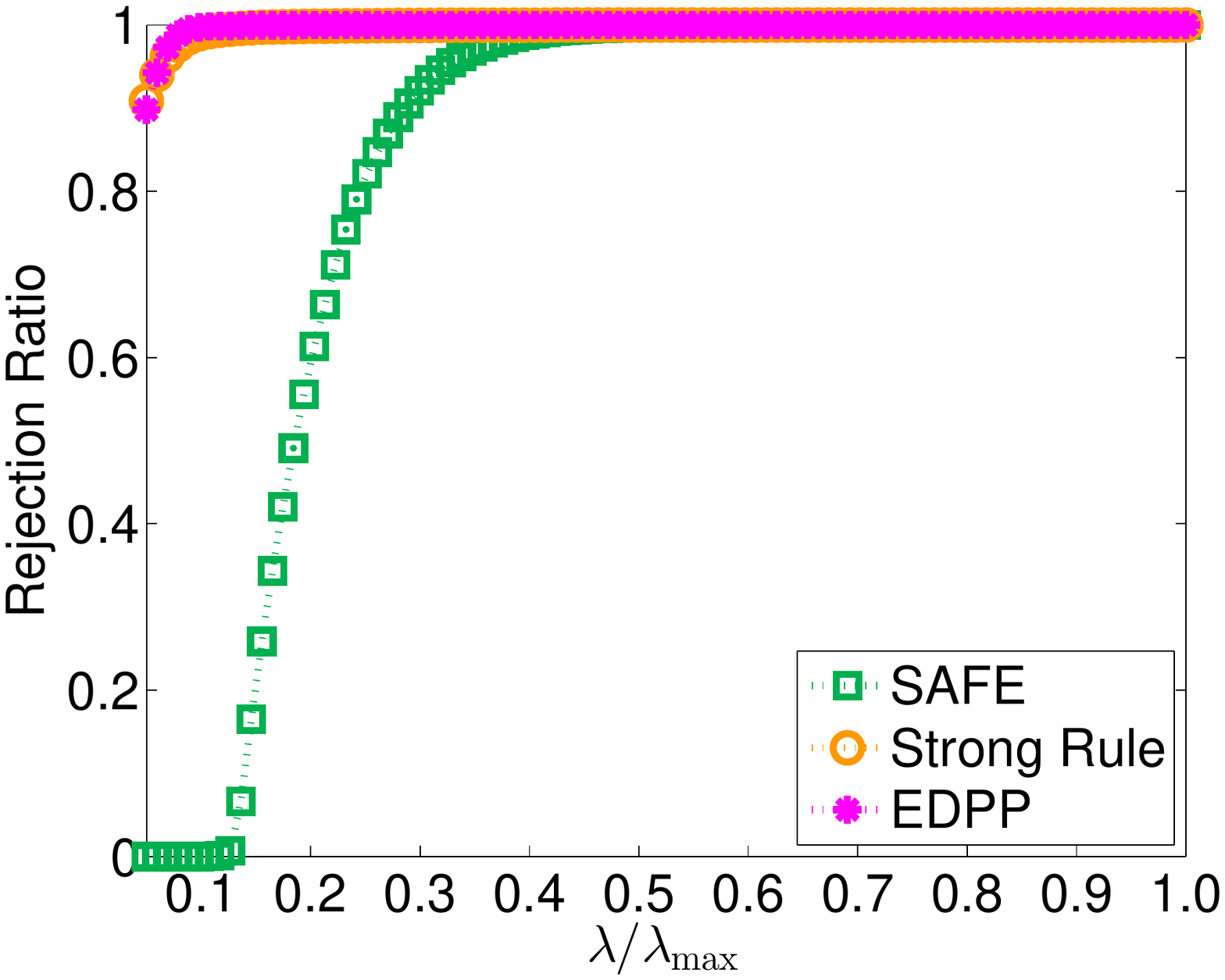}
			\includegraphics[width=0.31\columnwidth]{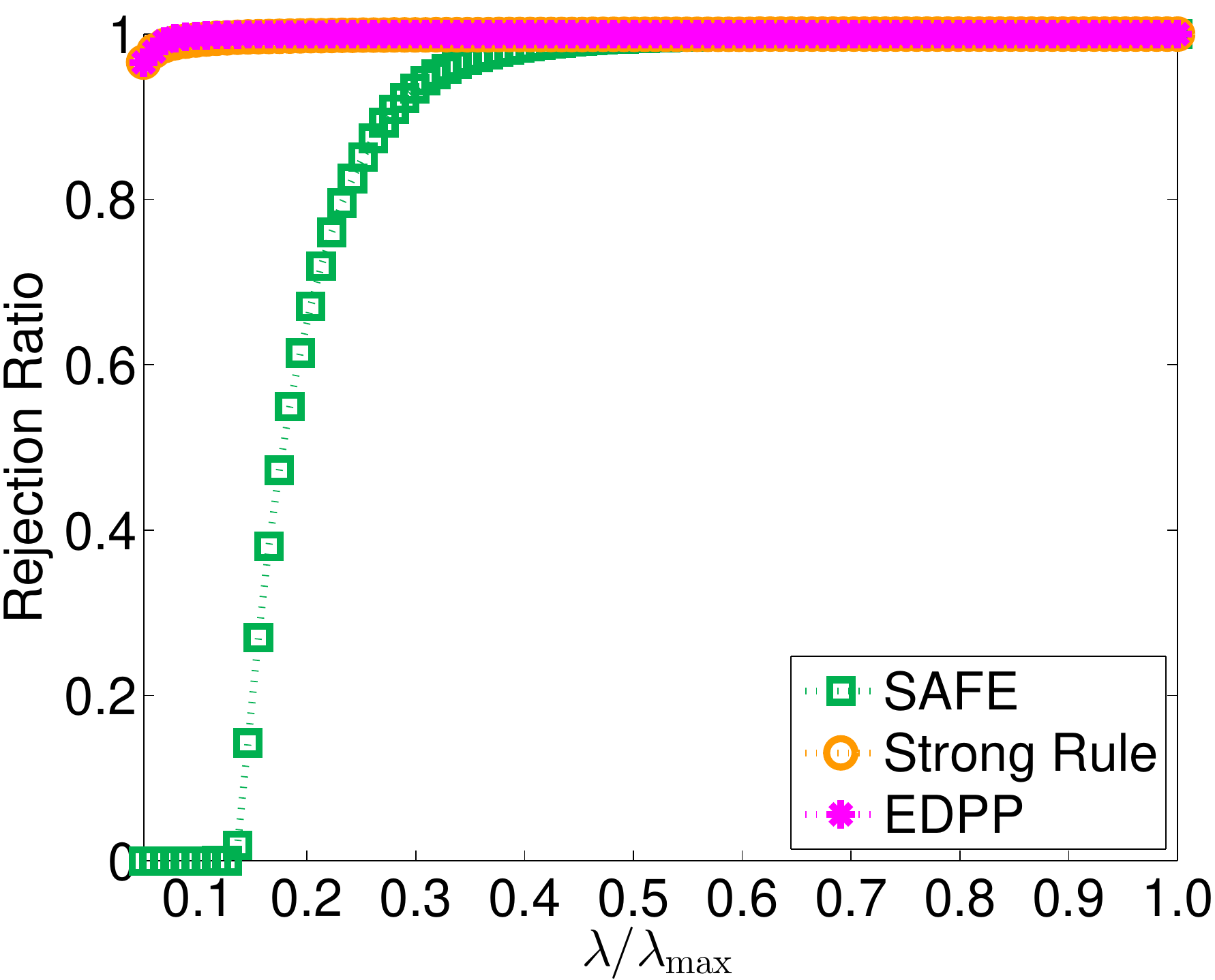}
			\includegraphics[width=0.31\columnwidth]{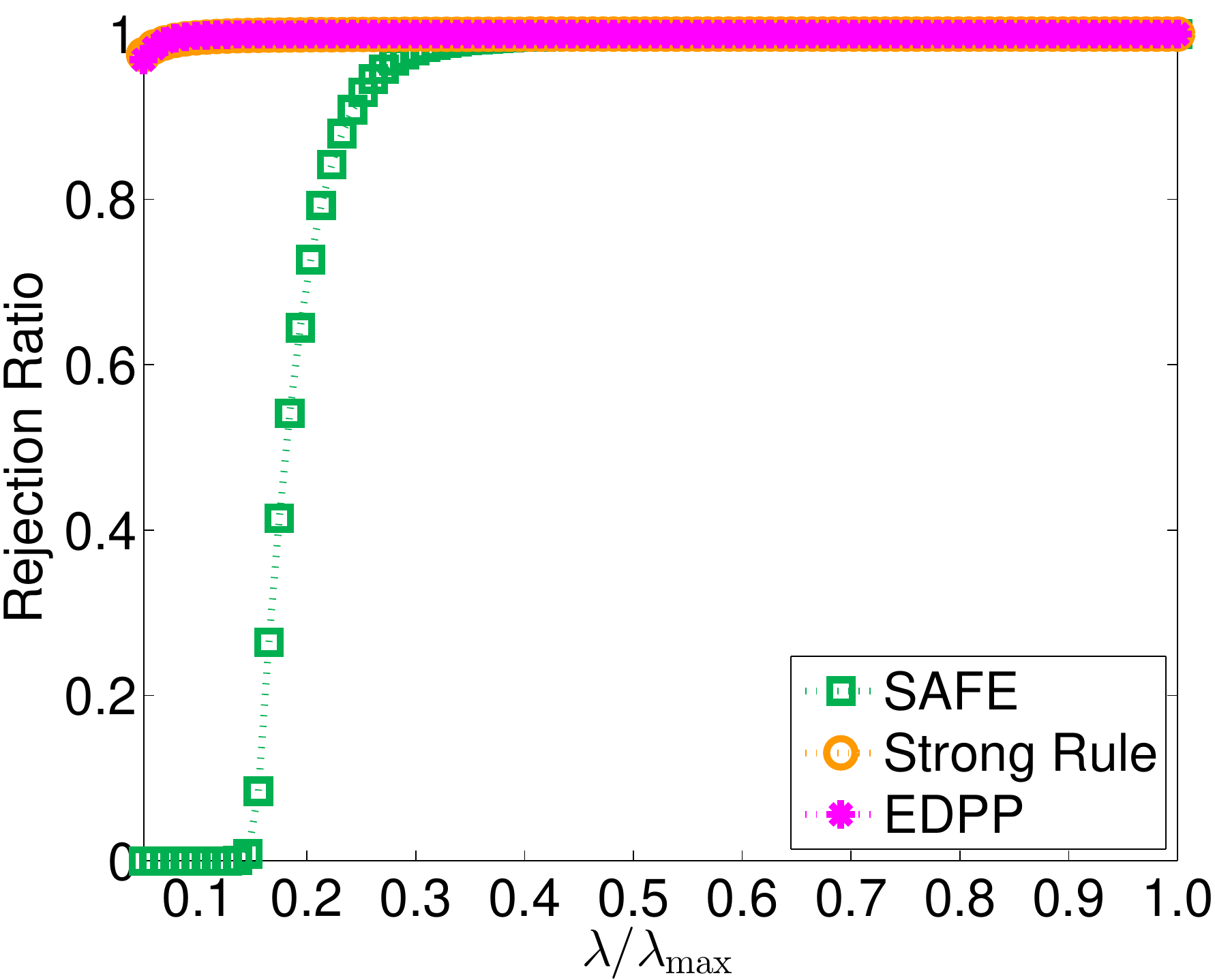}\\
			\subfigure[PIE, ${\bf X}\in\mathbb{R}^{1024\times 11553}$] { \label{fig:pie_e2}
				\includegraphics[width=0.31\columnwidth]{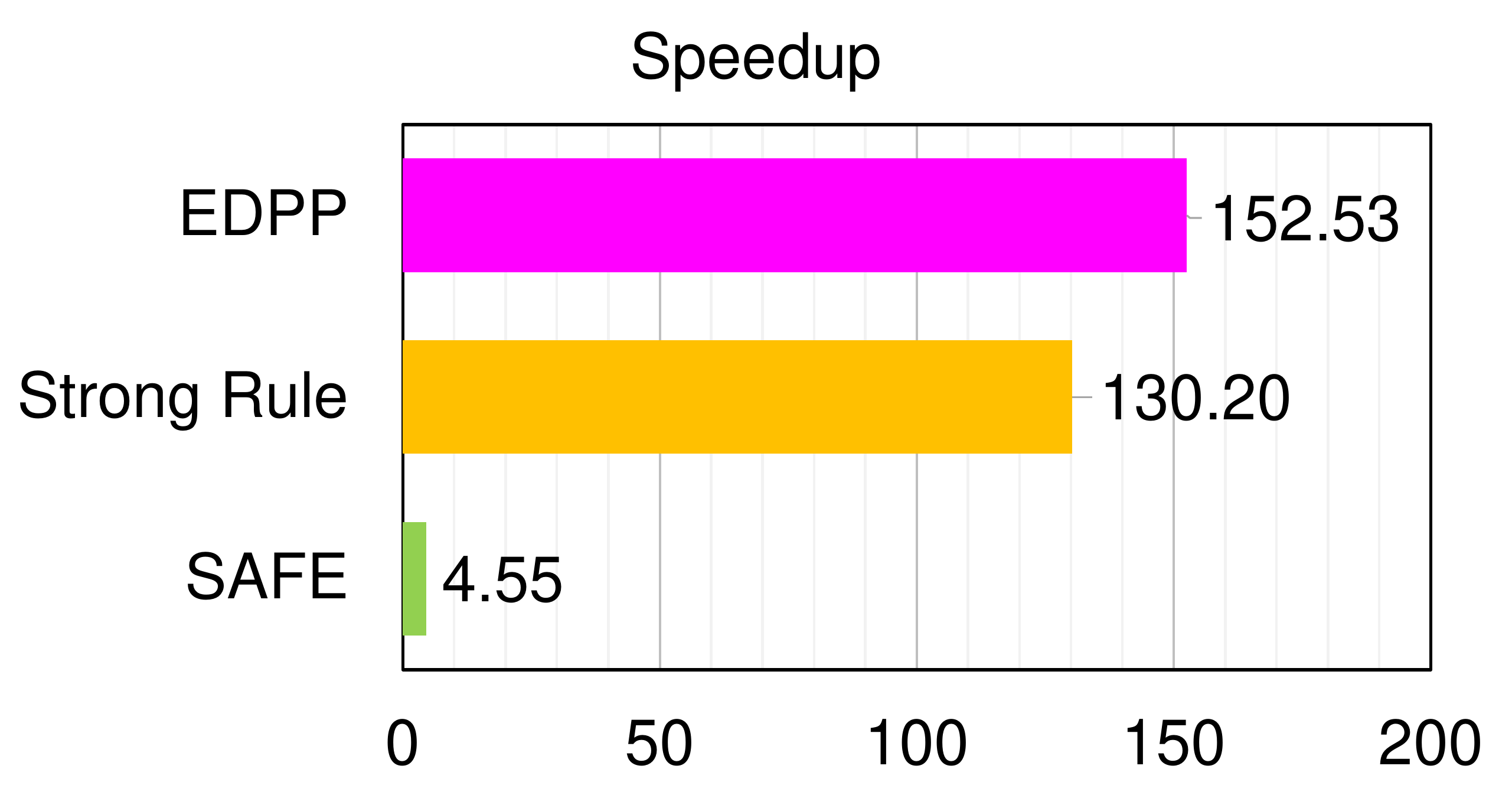}
			}
			\subfigure[MNIST, ${\bf X}\in\mathbb{R}^{784\times 50000}$] { \label{fig:mnist_e2}
				\includegraphics[width=0.31\columnwidth]{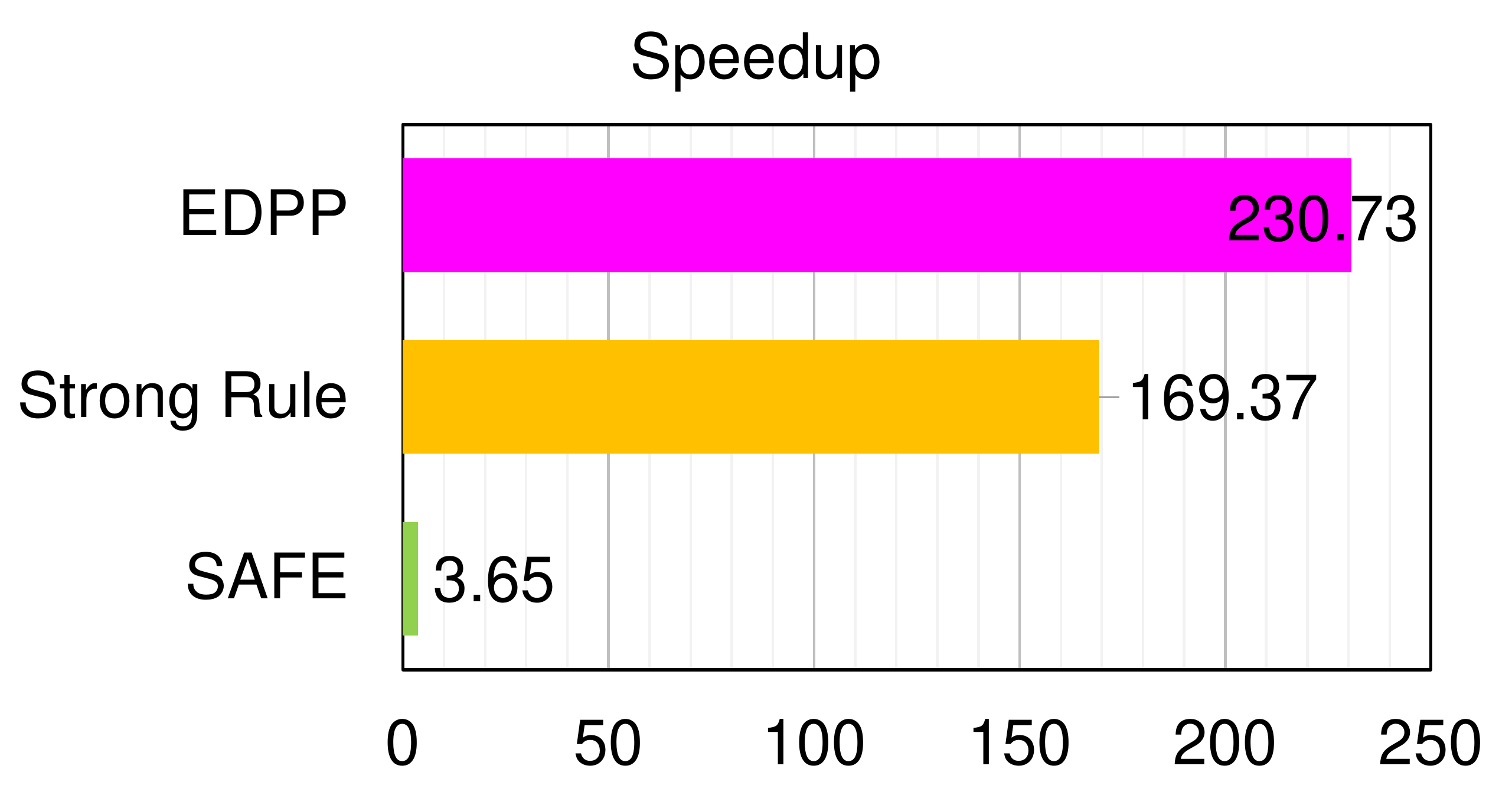}
			}
			\subfigure[SVHN, ${\bf X}\in\mathbb{R}^{3072\times 99288}$] { \label{fig:svhn}
				\includegraphics[width=0.31\columnwidth]{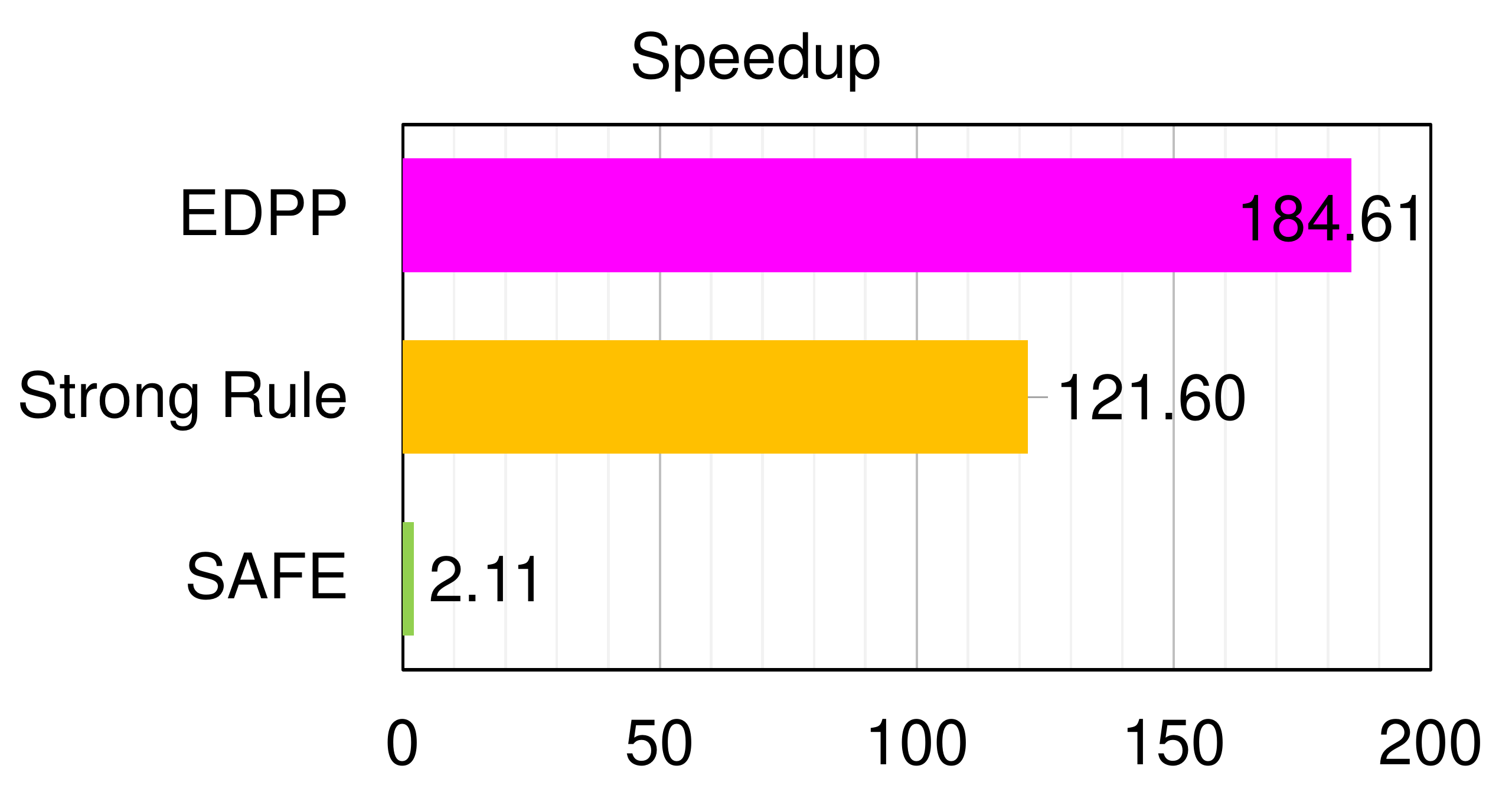}
			}
		}
		\caption{Comparison of SAFE, Strong Rule, and EDPP on six real data sets.  }
		\label{fig:lasso_real}
	\end{figure*}
	
	\vspace{2mm}
	{\bf The Breast Cancer Data Set} This data set contains $44$ tumor samples, each of which is represented by $7129$ genes. Therefore, the data matrix ${\bf X}$ is of $44\times 7129$. The response vector ${\bf y}\in\{1,-1\}^{44}$ contains the binary label of each sample.
	
	{\bf The Leukemia Data Set} This data set is a DNA microarray data set, containing $52$ samples and $11225$ genes. Therefore, the data matrix ${\bf X}$ is of $55\times 11225$. The response vector ${\bf y}\in\{1,-1\}^{52}$ contains the binary label of each sample.
	
	{\bf The SVHN Data set} The SVHN data set contains color images of street view house numbers, including $73257$ images for training and $26032$ for testing. The dimension of each image is $32\times 32$. In each trial, we first randomly select an image as the response ${\bf y}\in\mathbb{R}^{3072}$, and then use the remaining ones to form the data matrix ${\bf X}\in\mathbb{R}^{3072\times 99288}$. We run $100$ trials and report the average performance.
	
	\vspace{2mm}
	The description and the experiment settings for the Prostate Cancer data set, the PIE face image data set and the MNIST handwritten digit data set are given in Section \ref{subsubsection:EDPP}.
	
	\setlength{\tabcolsep}{.18em}
	\begin{table}
		\begin{center}
			\begin{footnotesize}
				\def\arraystretch{1.25}
				\begin{tabular}{ |l||c||c||c|c|c||c|c|c| }
					\hline
					Data &  solver & SAFE+solver &  Strong Rule+solver & EDPP+solver & SAFE & Strong Rule & EDPP  \\
					\hline\hline
					Breast Cancer  &  12.70 & 7.20 & 1.31 & 1.24 & 0.44 & 0.06 & 0.05  \\ \hline
					Leukemia  &  16.99 & 9.22 & 1.15 & 1.03 & 0.91 & 0.09 & 0.07  \\\hline
					Prostate Cancer   &  121.41 & 47.17 & 4.83 & 3.70 & 3.60 &  0.46 & 0.23 \\\hline
					PIE  &  629.94 & 138.33 & 4.84 & 4.13 & 19.93 & 2.54 & 1.33 \\\hline
					MNIST  &  2566.26 & 702.21 & 15.15 & 11.12 & 64.81 & 8.14 & 4.19  \\\hline
					SVHN  &  11023.30 & 5220.88 & 90.65 & 59.71 & 583.12 & 61.02 &  31.64 \\\hline
				\end{tabular}
			\end{footnotesize}
		\end{center}
		\caption{Running time (in seconds) for solving the Lasso problems along a sequence of $100$ tuning parameter values equally spaced on the scale of ${\lambda}/{\lambda_{\rm max}}$ from $0.05$ to $1$ by (a): the solver \citep{SLEP} (reported in the second column) without screening; (b): the solver combined with different screening methods (reported in the $3^{rd}$ to the $5^{th}$ columns).
			The last three columns report the total running time (in seconds) for the screening methods.
		}
		\label{table:lasso_real_time}
	\end{table}
	
	From \figref{fig:lasso_real}, we can see that the rejection ratios of strong rule and EDPP are comparable to each other. Compared to SAFE, both of strong rule and EDPP are able to identify far more inactive features, leading to a much higher speedup. However, because strong rule needs to check the KKT conditions to ensure the correctness of the screening results, the speedup gained by EDPP is higher than that by strong rule. When the size of the data matrix is not very large, e.g., the Breast Cancer and Leukemia data sets, the speedup gained by EDPP are slightly higher than that by strong rule. However, when the size of the data matrix is large, e.g., the MNIST and SVHN data sets, the speedup gained by EDPP are significantly higher than that by strong rule. Moreover, we can also observe from \figref{fig:lasso_real} that, the larger the data matrix is, the higher the speedup can be gained by EDPP. More specifically, for the small data sets, e.g., the Breast Cancer, Leukemia and Prostate Cancer data sets, the speedup gained by EDPP is about $10$, $17$ and $30$ times. In contrast, for the large data sets, e.g., the PIE, MNIST and SVHN data sets, the speedup gained by EDPP is two orders of magnitude. Take the SVHN data set for example. The solver without screening needs about $3$ {\it hours} to solve the $100$ Lasso problems. Combined with the EDPP rule, the solver only needs less than $1$ {\it minute} to complete the task.
	
	Clearly, the proposed EDPP screening rule is very effective in accelerating the computation of Lasso especially for large-scale problems, and outperforms the state-of-the-art approaches like SAFE and strong rule. Notice that, the EDPP method is safe in the sense that the discarded features are guaranteed to have zero coefficients in the solution.
	
	\vspace{2mm}
	\noindent\textbf{EDPP with Least-Angle Regression (LARS)}
	
	As we mentioned in the introduction, we can combine EDPP with any existing solver. In this experiment, we integrate EDPP and strong rule with another state-of-the-art solver for Lasso, i.e., Least-Angle Regression (LARS) \citep{Efron04}. We perform experiments on the same real data sets used in the last section with the same experiment settings. Because the rejection ratios of screening methods are irrelevant to the solvers, we only report the speedup. Table \ref{table:lasso_real_time_LARS} reports the running time of LARS with or without screening for solving the 100 Lasso problems, and that of the screening methods. \figref{fig:lasso_real_LARS} shows the speedup of these two methods. We can still observe a substantial speedup gained by EDPP. The reason is that EDPP has a very low computational cost (see  Table \ref{table:lasso_real_time_LARS}) and it is very effective in discarding inactive features (see \figref{fig:lasso_real}).
	
	\setlength{\tabcolsep}{.18em}
	\begin{table}
		\begin{center}
			\begin{footnotesize}
				\def\arraystretch{1.25}
				\begin{tabular}{ |l||c||c||c|c|c||c|c|c| }
					\hline
					Data &  LARS &  Strong Rule+LARS & EDPP+LARS & Strong Rule & EDPP  \\
					\hline\hline
					Breast Cancer  &  1.30 &  0.06 & 0.04 & 0.04 & 0.03  \\ \hline
					Leukemia  &  1.46 & 0.09 & 0.05 & 0.07 & 0.04  \\\hline
					Prostate Cancer   &  5.76 & 1.04 & 0.37 & 0.42 & 0.24 \\\hline
					PIE  &  22.52 & 2.42 & 1.31 & 2.30 & 1.21 \\\hline
					MNIST  &  92.53 & 8.53 & 4.75 & 8.36 & 4.34  \\\hline
					SVHN  &  1017.20 & 65.83 & 35.73 & 62.53 &  32.00 \\\hline
				\end{tabular}
			\end{footnotesize}
		\end{center}
		\caption{Running time (in seconds) for solving the Lasso problems along a sequence of $100$ tuning parameter values equally spaced on the scale of ${\lambda}/{\lambda_{\rm max}}$ from $0.05$ to $1$ by (a): the solver \citep{Efron04,Mairal2010} (reported in the second column) without screening; (b): the solver combined with different screening methods (reported in the $3^{rd}$ and $4^{th}$ columns).
			The last two columns report the total running time (in seconds) for the screening methods.
		}
		\label{table:lasso_real_time_LARS}
	\end{table}
	
	\begin{figure*}[ht]
		\centering{
			\subfigure[Breast Cancer, ${\bf X}\in\mathbb{R}^{44\times 7129}$] { \label{fig:breastcancer_LARS}
				\includegraphics[width=0.31\columnwidth]{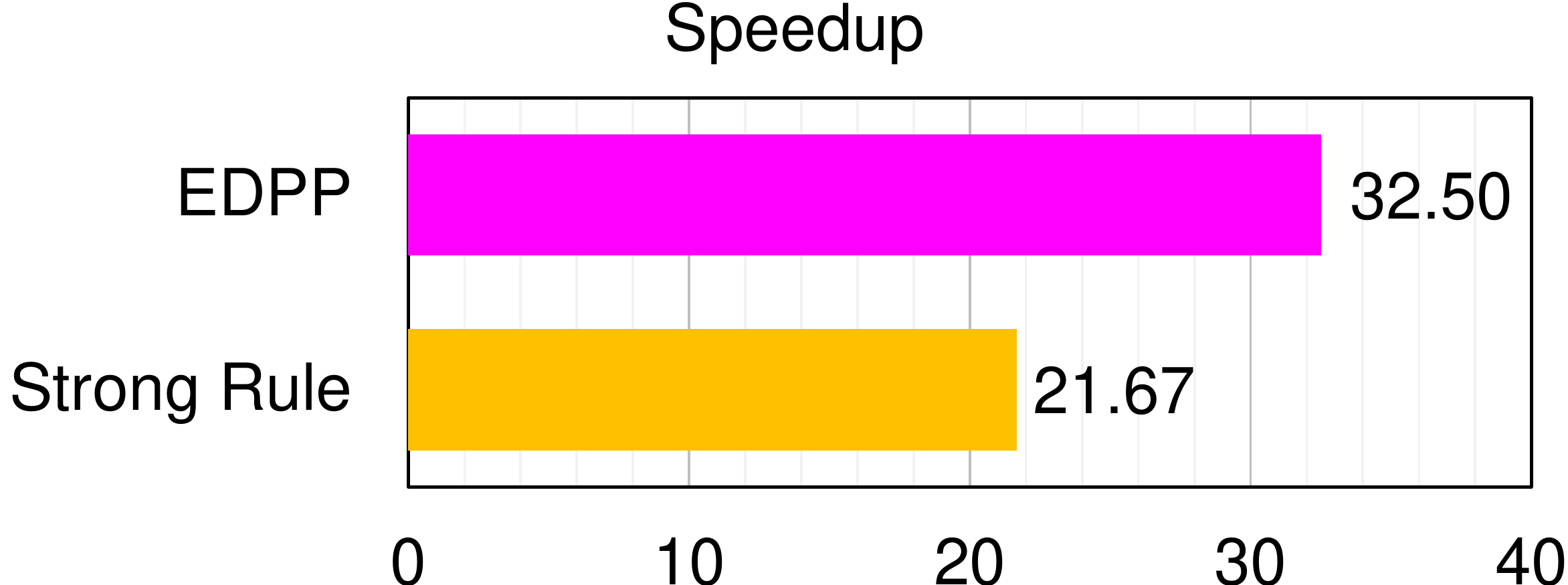}
			}
			\subfigure[Leukemia, ${\bf X}\in\mathbb{R}^{55\times 11225}$] { \label{fig:Leukemia_LARS}
				\includegraphics[width=0.31\columnwidth]{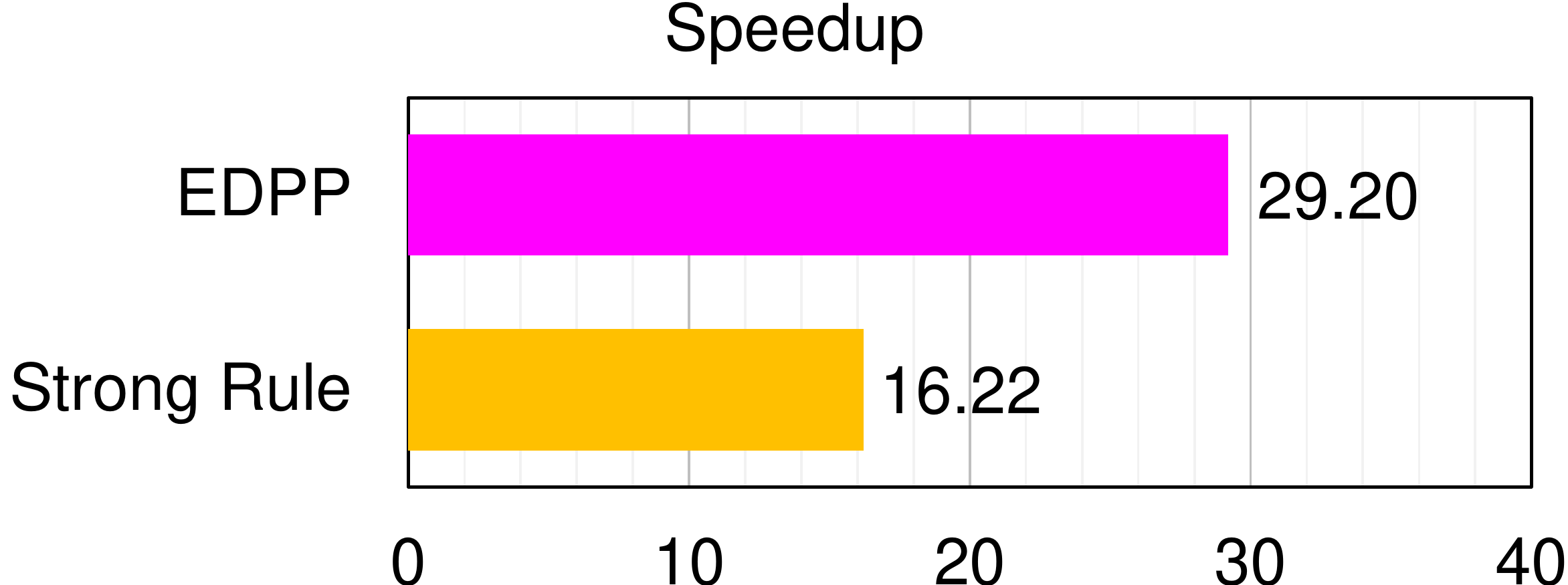}
			}
			\subfigure[{\scriptsize Prostate Cancer}, ${\bf X}\in\mathbb{R}^{132\times 15154}$] { \label{fig:prostatecancer_LARS}
				\includegraphics[width=0.31\columnwidth]{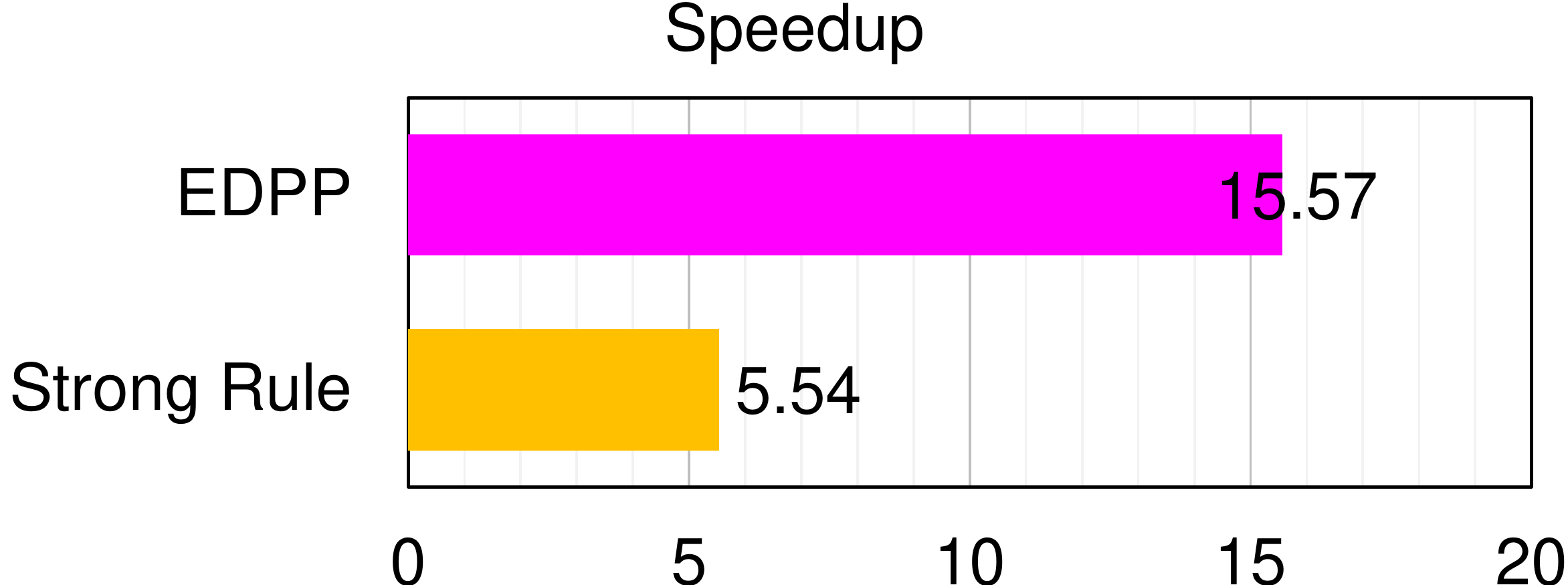}
			}\\
			\subfigure[PIE, ${\bf X}\in\mathbb{R}^{1024\times 11553}$] { \label{fig:pie_LARS}
				\includegraphics[width=0.31\columnwidth]{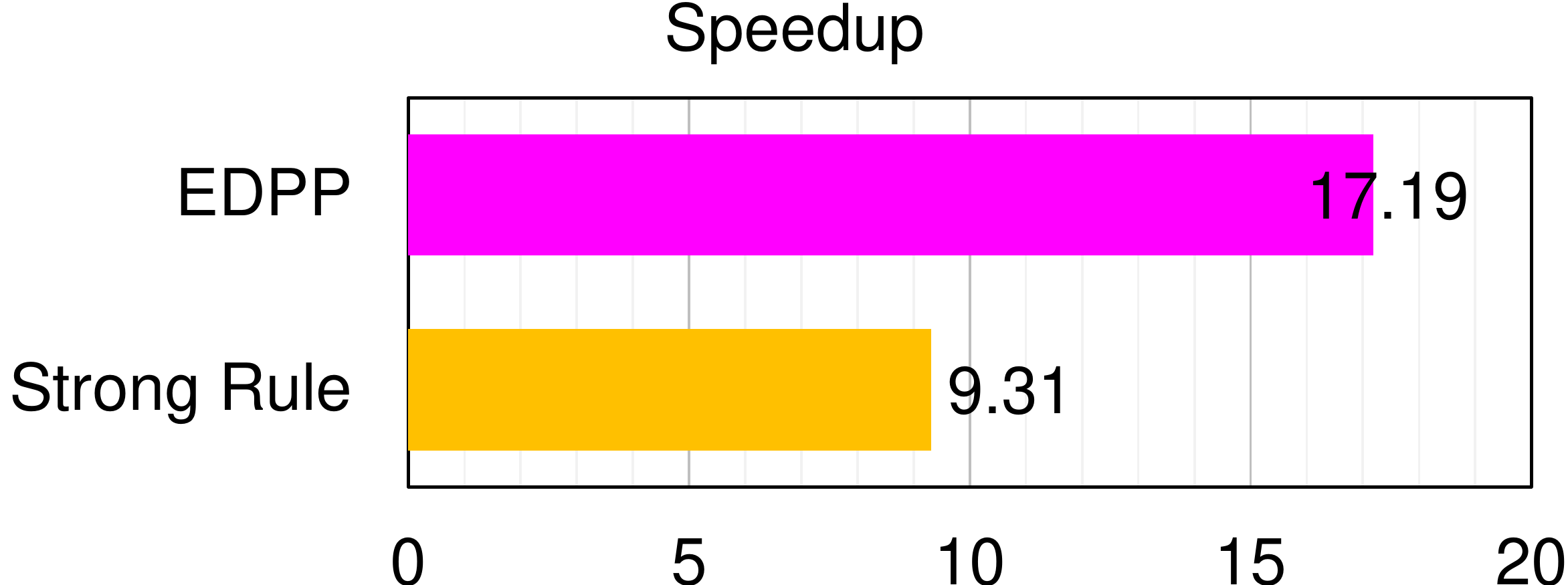}
			}
			\subfigure[MNIST, ${\bf X}\in\mathbb{R}^{784\times 50000}$] { \label{fig:mnist_LARS}
				\includegraphics[width=0.31\columnwidth]{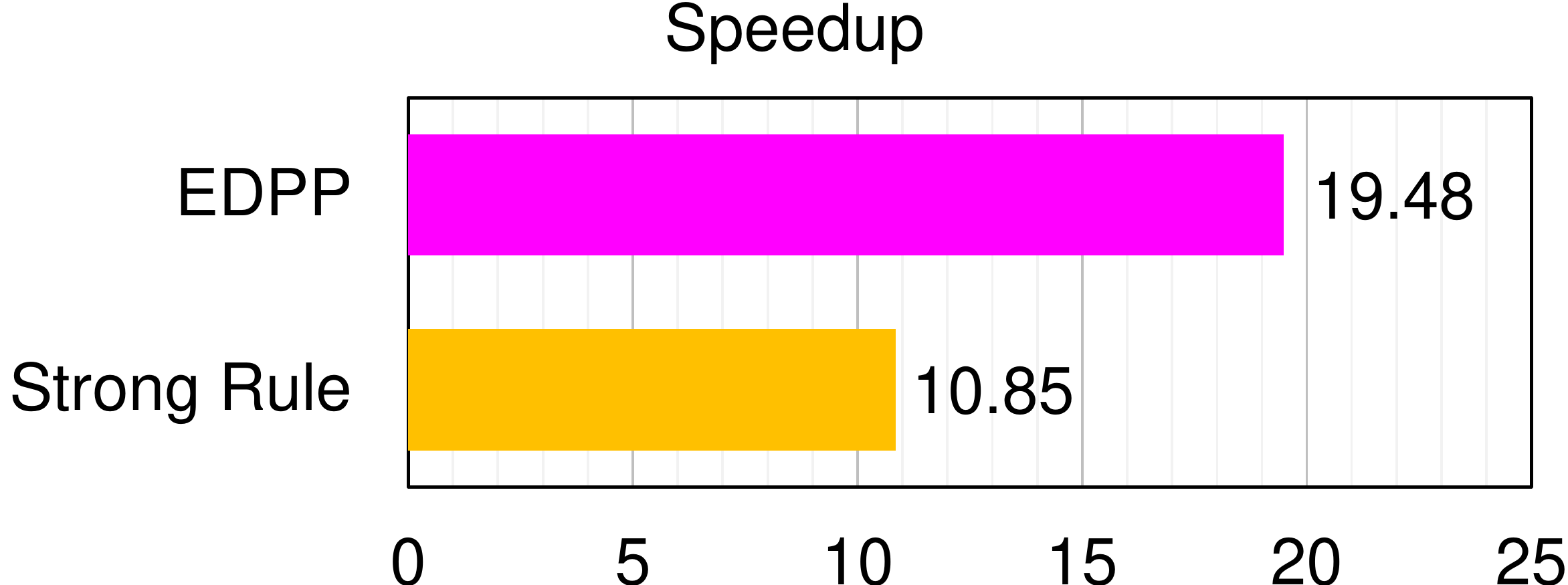}
			}
			\subfigure[SVHN, ${\bf X}\in\mathbb{R}^{3072\times 99288}$] { \label{fig:svhn_LARS}
				\includegraphics[width=0.31\columnwidth]{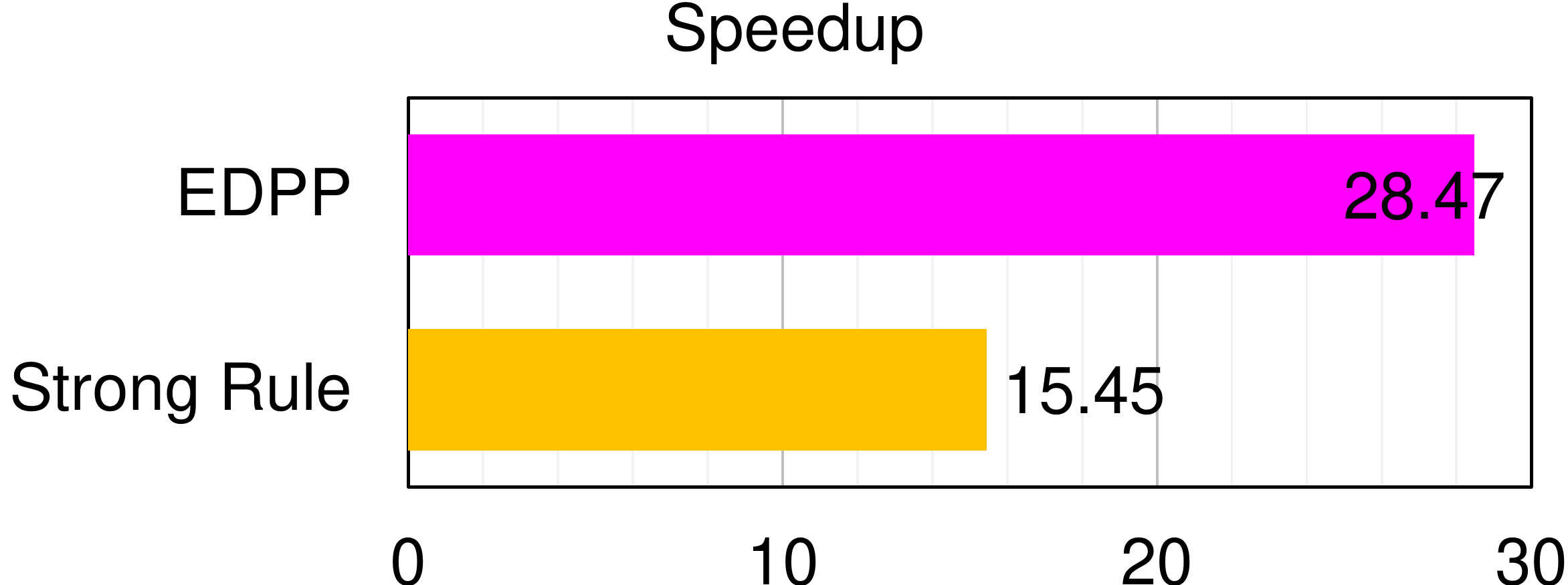}
			}
		}
		\caption{The speedup gained by Strong Rule and EDPP with LARS on six real data sets.}
		\label{fig:lasso_real_LARS}
	\end{figure*}
	
	\subsection{EDPP for the Group Lasso Problem}\label{subsection:experiment_glasso}
	
	In this experiment, we evaluate the performance of EDPP and strong rule with different numbers of groups. The data matrix ${\bf X}$ is fixed to be $250\times200000$. The entries of the response vector ${\bf y}$ and the data matrix ${\bf X}$  are generated i.i.d. from a standard Gaussian distribution. For each experiment, we repeat the computation $20$ times and report the average results. Moreover, let $n_g$ denote the number of groups and $s_g$ be the average group size. For example, if $n_g$ is $10000$, then $s_g=p/n_g=20$.
	
	\begin{figure}[th!]
		\centering{
			\includegraphics[width=0.31\columnwidth]{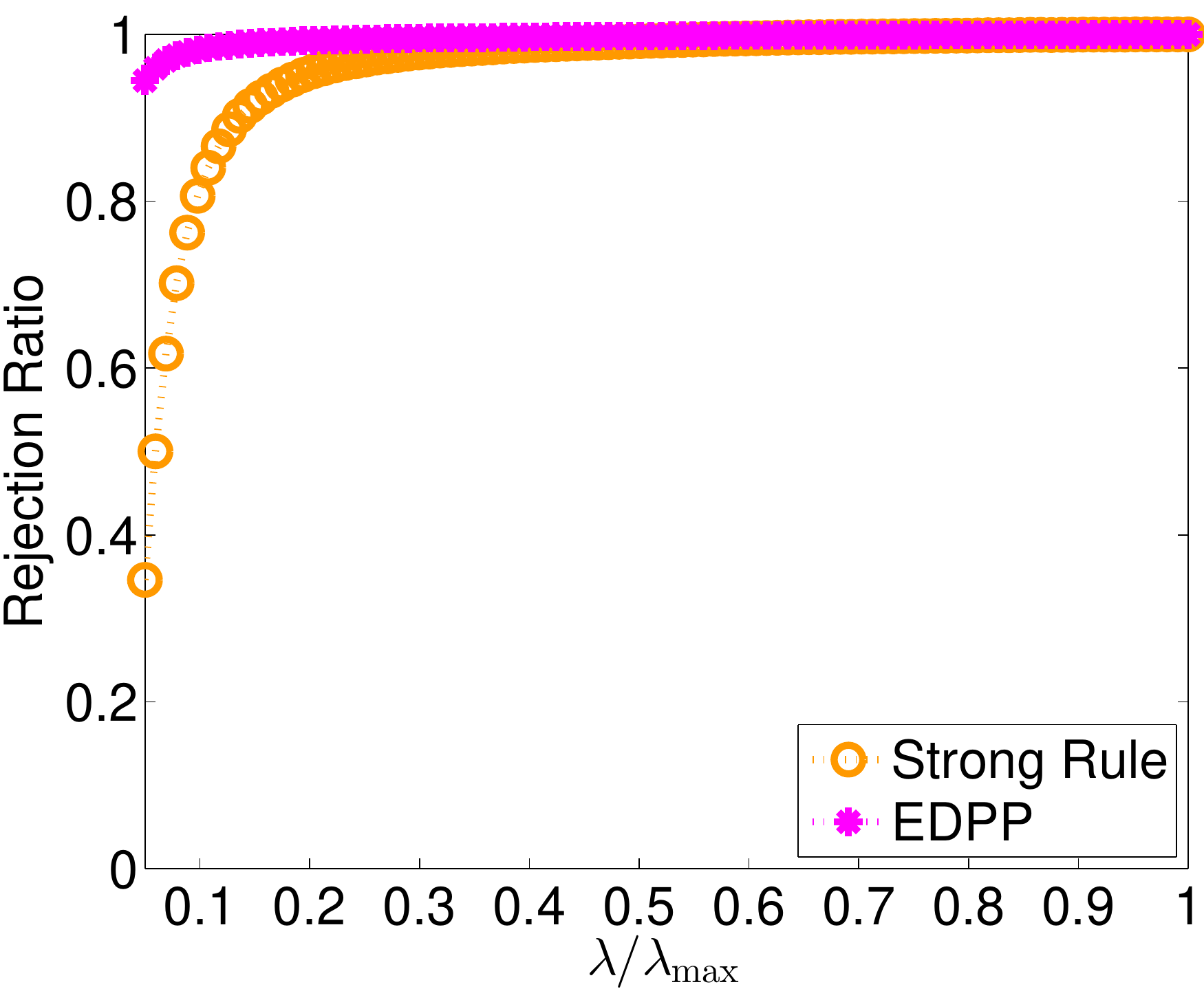}
			\includegraphics[width=0.31\columnwidth]{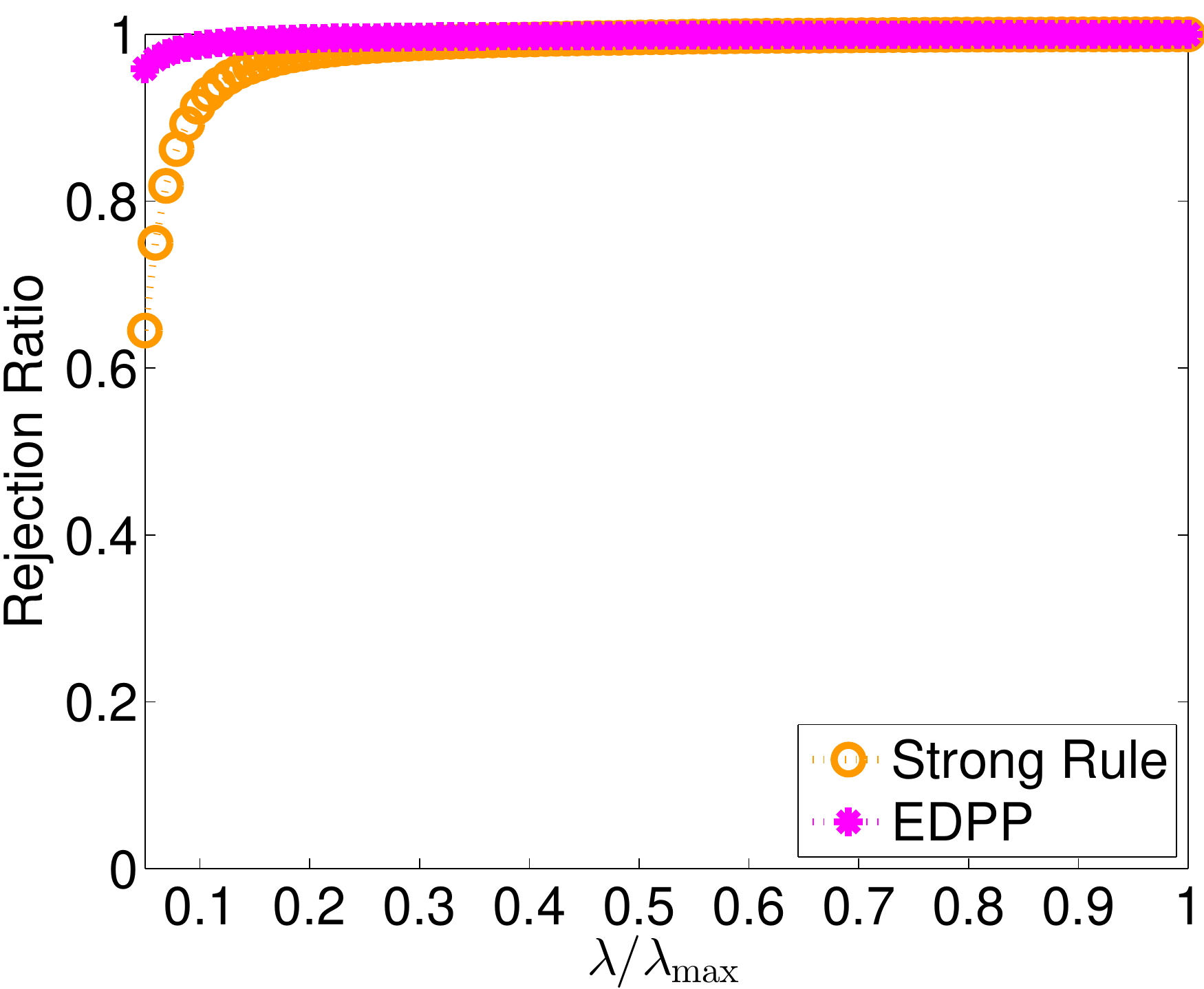}
			\includegraphics[width=0.31\columnwidth]{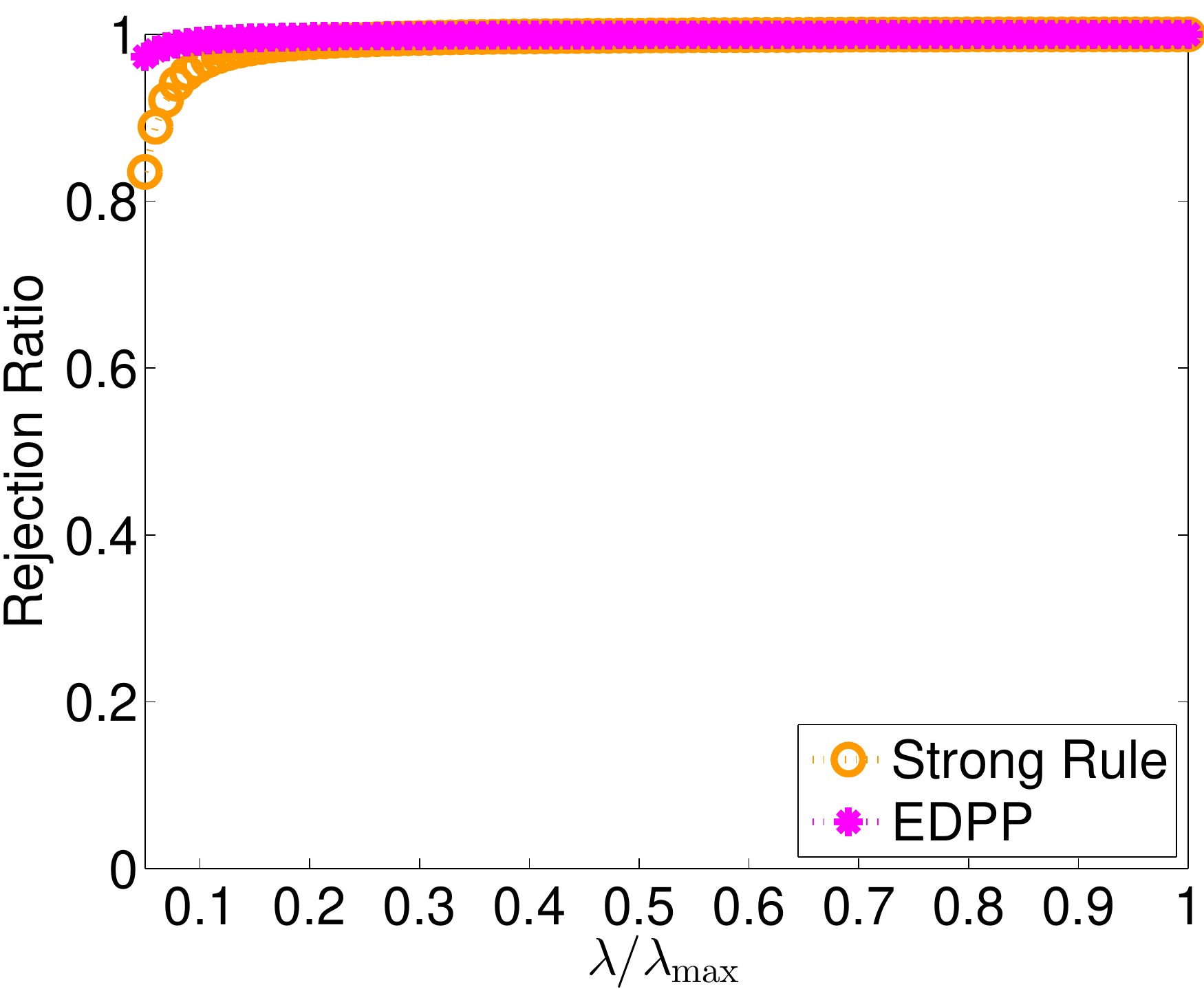}\\
			\subfigure[$n_g=10000$] { \label{fig:ng10000}
				\includegraphics[width=0.31\columnwidth]{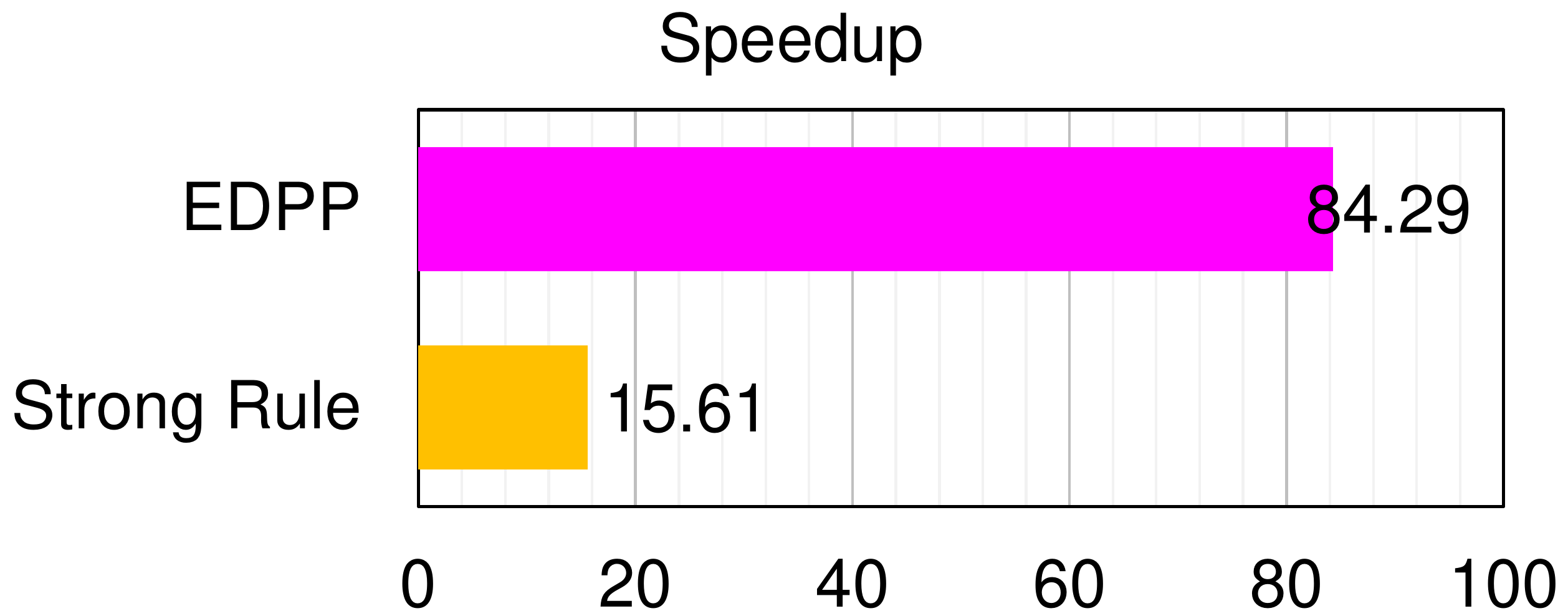}
			}
			\subfigure[$n_g=20000$] { \label{fig:ng20000}
				\includegraphics[width=0.31\columnwidth]{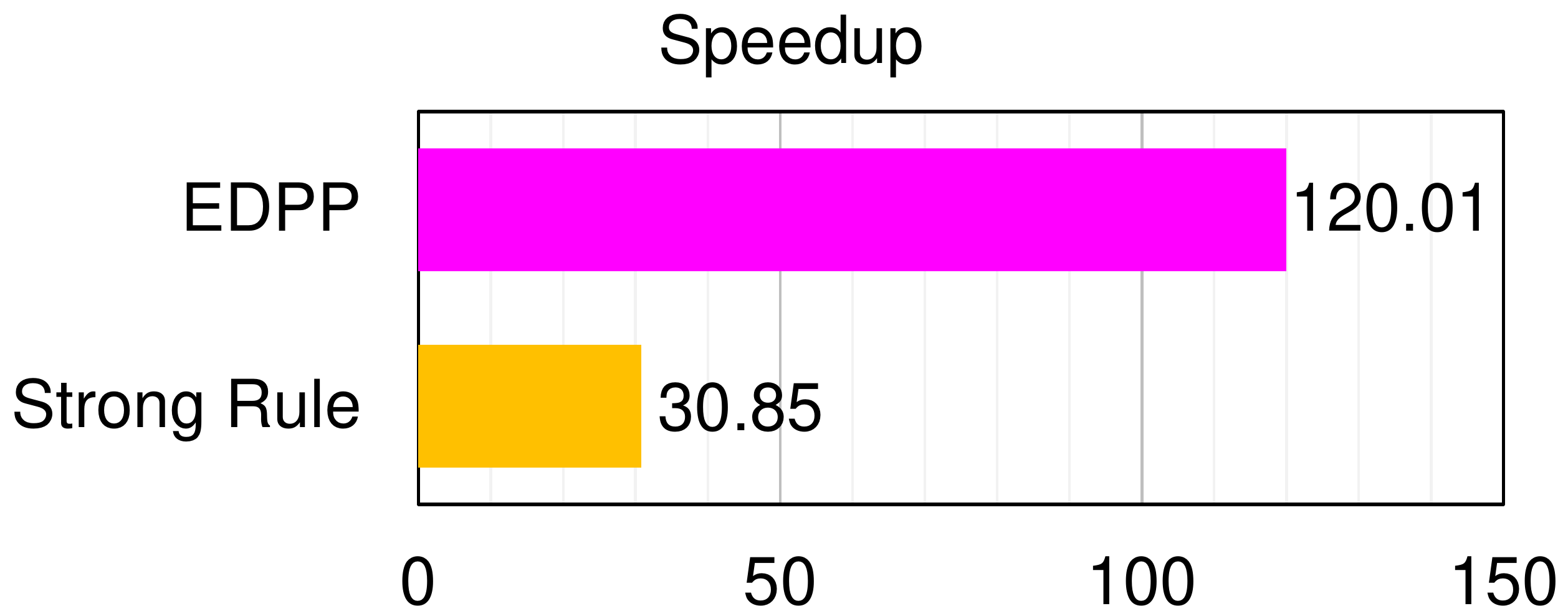}
			}
			\subfigure[$n_g=40000$] { \label{fig:ng40000}
				\includegraphics[width=0.31\columnwidth]{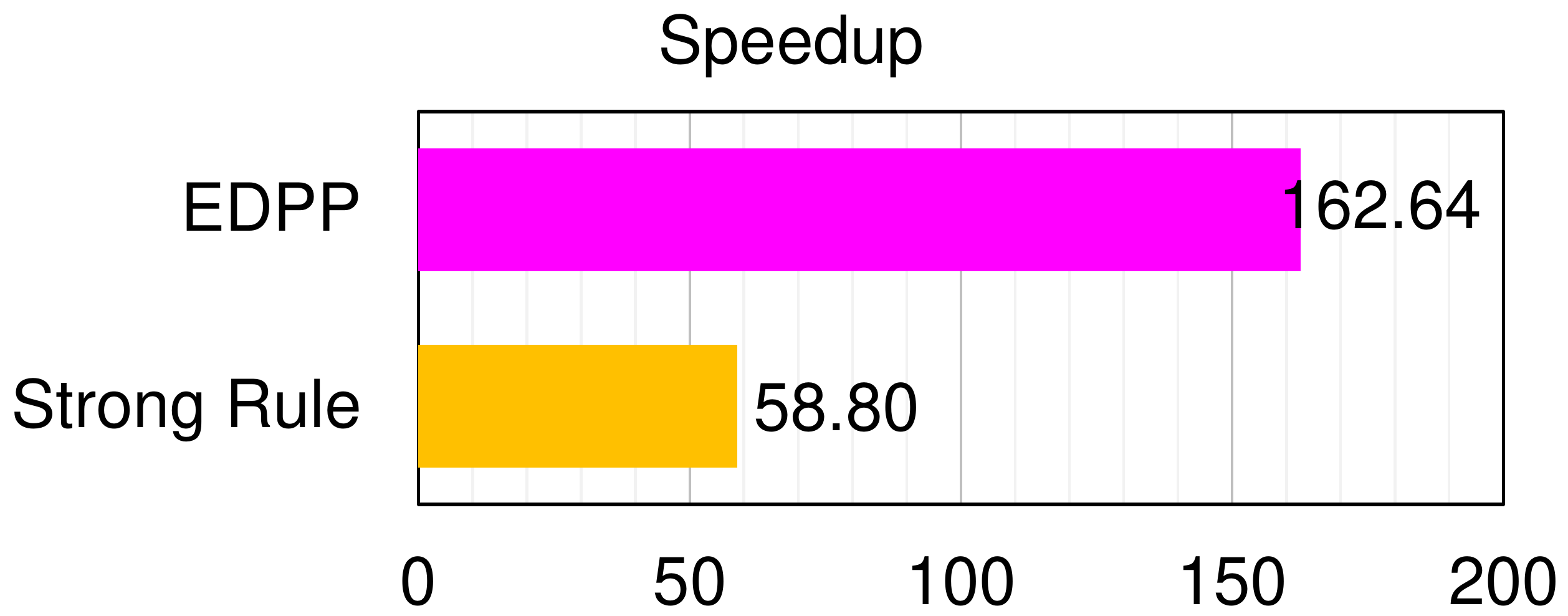}
			}
		}
		\caption{Comparison of EDPP and strong rules with different numbers of groups.}
		\label{fig:rej_ratio_e2}
	\end{figure}
	
	From Figure \ref{fig:rej_ratio_e2}, we can see that EDPP and strong rule are able to discard more inactive groups when the number of groups $n_g$ increases. The intuition behind this observation is that the estimation of the dual optimal solution is more accurate with a smaller group size. Notice that, a large $n_g$ implies a small average group size. Figure \ref{fig:rej_ratio_e2} also implies that compared to strong rule, EDPP is able to discard more inactive groups and is more robust with respect to different values of $n_g$.
	
	\begin{table}
		\begin{center}
			\begin{footnotesize}
				\def\arraystretch{1.25}
				\begin{tabular}{ |l||c||c|c|c||c|c|c| }
					\hline
					$n_g$ & solver & Strong Rule+solver & EDPP+solver &  Strong Rule & EDPP  \\
					\hline\hline
					$10000$  & 4535.54 &  296.60 & 53.81 &  13.99 & 8.32 \\\hline
					$20000$  & 5536.18 &  179.48 & 46.13 &  14.16 & 8.61 \\\hline
					$40000$  & 6144.48 &  104.50 & 37.78 &  13.13 & 8.37 \\\hline
				\end{tabular}
			\end{footnotesize}
		\end{center}
		\vspace{0.1in}
		\caption{Running time (in seconds) for solving the group Lasso problems along a sequence of $100$ tuning parameter values equally spaced on the scale of ${\lambda}/{\lambda_{\rm max}}$ from $0.05$ to $1.0$ by (a): the solver from SLEP (reported in the second column) without screening; (b): the solver combined with different screening methods (reported in the $3^{rd}$ and $4^{th}$ columns).
			The last two columns report the total running time (in seconds) for the screening methods. The data matrix ${\bf X}$ is of size $250\times200000$.
		}
		\label{table:time_diff_ng}
	\end{table}
	
	Table \ref{table:time_diff_ng} further demonstrates the effectiveness of EDPP in improving the efficiency of the solver. When $n_g=10000$, the efficiency of the solver is improved by about $80$ times. When $n_g=20000$ and $40000$, the efficiency of the solver is boosted by about $120$ and $160$ times with EDPP respectively.

\section{Conclusion}\label{section:conclusion}
In this paper, we develop new screening rules for the Lasso problem by making use of the properties of the projection operators with respect to a closed convex set. Our proposed methods, i.e., DPP screening rules, are able to effectively identify inactive predictors of the Lasso problem, thus greatly reducing the size of the optimization problem. Moreover, we further improve DPP rule and propose the enhanced DPP rule, which is more effective in discarding inactive features than DPP rule. The idea of the family of DPP rules can be easily generalized to identify the inactive groups of the group Lasso problem. Extensive numerical experiments on both synthetic and real data demonstrate the effectiveness of the proposed rules. It is worthwhile to mention that the family of DPP rules can be combined with any Lasso solver as a speedup tool.
In the future, we plan to generalize our ideas to other sparse formulations consisting of more general structured sparse penalties, e.g., tree/graph Lasso, fused Lasso.


\newpage

\appendix
\section*{Appendix A.}
In this appendix, we give the detailed derivation of the dual problem of Lasso.
\subsection*{A1. Dual Formulation}
Assuming the data matrix is ${\bf X}\in\mathbb{R}^{N\times p}$, the standard Lasso problem is given by:
\begin{equation}\label{equation:origin_prob}
	\inf_{\beta\in\mathbb{R}^p}\frac{1}{2}\|{\bf y - X\beta}\|_2^2+\lambda\|{\beta}\|_1.
\end{equation}
For completeness, we give a detailed deviation of the dual formulation of (\ref{equation:origin_prob}) in this section. Note that problem (\ref{equation:origin_prob}) has no constraints. Therefore the dual problem is trivial and useless. A common trick \citep{Boyd04} is to introduce a new set of variables ${\bf z = y - X\beta}$ such that problem (\ref{equation:origin_prob}) becomes:
\begin{align}\label{Equation origin_prob_new}
	\inf_{\beta}\qquad&\frac{1}{2}\|{\bf z}\|_2^2+\lambda\|{\beta}\|_1,\\ \nonumber
	\mbox{subject to} \qquad&{\bf z = y - X\beta}.
\end{align}

By introducing the dual variables $\eta\in\mathbb{R}^N$, we get the Lagrangian of problem (\ref{Equation origin_prob_new}):
\begin{equation}\label{Equation Lagrangin}
	L({\beta, {\bf z}, \eta})= \frac{1}{2}\|{\bf z}\|_2^2+\lambda\|{\beta}\|_1 + \eta^{T}\cdot({\bf y - X\beta - z}).
\end{equation}

For the Lagrangian, the primal variables are ${\beta}$ and ${\bf z}$.
And the dual function $g(\eta)$ is:
\begin{align}\label{Equation Dual}
	g(\eta) = \inf_{\beta, {\bf z}}L({\beta, {\bf z}, \eta})=\eta^{T}{\bf y}+\inf_{\beta}(-\eta^{T}{\bf X\beta}+\lambda\|{\beta}\|_1)+\inf_{\bf z}\big(\frac{1}{2}\|{\bf z}\|_2^2-\eta^{T}{\bf z}\big).
\end{align}

In order to get $g(\eta)$, we need to solve the following two optimization problems.
\begin{equation}\label{Equation minw}
	\inf_{\beta}-\eta^{T}{\bf X\beta}+\lambda\|{\beta}\|_1,
\end{equation}
and
\begin{equation}\label{Equation miny}
	\inf_{\bf z}\frac{1}{2}\|{\bf z}\|_2^2-\eta^{T}{\bf z}.
\end{equation}

Let us first consider problem (\ref{Equation minw}). Denote the objective function of problem (\ref{Equation minw}) as
\begin{equation}\label{Equation obj_minw}
	f_1({\beta}) = -\eta^{T}{\bf X\beta}+\lambda\|{\beta}\|_1.
\end{equation}
$f_1(\beta)$ is convex but not smooth. Therefore let us consider its subgradient $$\partial f_1({\beta}) = -{\bf X}^{T}\eta + \lambda{\bf v},$$ in which $\|{\bf v}\|_{\infty}\leq 1$ and ${\bf v^{T}\beta}=\|{\beta}\|_1$, i.e., ${\bf v}$ is the subgradient of $\|\beta\|_1$.

The necessary condition for $f_1$ to attain an optimum is $$\exists\,{\beta}^{\prime}, \mbox{ such that }0\in\partial f_1({\beta}^{\prime}) = \{-{\bf X}^{T}\eta + \lambda{\bf v}^{\prime}\},$$
where ${\bf v}^{\prime}\in\partial\|\beta^{\prime}\|_1$.
In other words, ${\beta^{\prime}, {\bf v}^{\prime}}$ should satisfy
$${\bf v^{\prime}} = \frac{{\bf X}^{T}\eta}{\lambda},\|{\bf v}^{\prime}\|_{\infty}\leq 1, {\bf v^{\prime}}^{T}{\beta}^{\prime}=\|{\beta}^{\prime}\|_1,$$
which is equivalent to
\begin{equation}\label{Equation constraint}
	|{\bf x}_i^{T}\eta|\leq\lambda, i = 1, 2, \ldots, p.
\end{equation}
Then we plug ${\bf v^{\prime}} = \frac{{\bf X}^{T}\eta}{\lambda}$ and ${\bf v^{\prime}}^{T}{\beta}^{\prime}=\|{\beta}^{\prime}\|_1$ into \eqref{Equation obj_minw}:
\begin{equation}
	f_1({\beta}^{\prime}) = \inf_{\beta}f_1({\beta})=-\eta^{T}{\bf X}{\beta}^{\prime}+\lambda\big(\frac{{\bf X}^{T}\eta}{\lambda}\big)^{T}{\beta}^{\prime}=0.
\end{equation}
Therefore, the optimum value of problem (\ref{Equation minw}) is $0$.

Next, let us consider problem (\ref{Equation miny}). Denote the objective function of problem (\ref{Equation miny}) as $f_2({\bf z})$. Let us rewrite $f_2({\bf z})$ as:
\begin{equation}
	f_2({\bf z}) = \frac{1}{2}(\|{\bf z} - \eta\|_2^2 - \|\eta\|_2^2).
\end{equation}
Clearly, $${\bf z}^{\prime} = \mathop{\argmin}_{\bf z}f_2({\bf z}) = \eta,$$
and $$\inf_{\bf z}f_2({\bf z}) = -\frac{1}{2}\|\eta\|_2^2.$$

Combining everything above, we get the dual problem:
\begin{align}\label{Equation dual_prob}
	\sup_{\eta}\quad &g(\eta)=\eta^{T}{\bf y}-\frac{1}{2}\|\eta\|_2^2,\\ \nonumber
	\mbox{subject to}\quad &|{\bf x}_i^{T}\eta|\leq \lambda,\,i=1,2,\ldots,p.
\end{align}
which is equivalent to
\begin{align}\label{Equation dual_prob_new1}
	\sup_{\eta}\quad &g(\eta)=\frac{1}{2}\|{\bf y}\|_2^2 - \frac{1}{2}\|\eta - {\bf y}\|_2^2,\\ \nonumber
	\mbox{subject to}\quad &|{\bf x}_i^{T}\eta|\leq \lambda,\,i=1,2,\ldots,p.
\end{align}

By a simple re-scaling of the dual variables $\eta$, i.e., let $\theta = \frac{\eta}{\lambda}$, problem (\ref{Equation dual_prob_new1}) transforms to:
\begin{align}\label{Equation dual_prob_new2}
	\sup_{\theta}\quad &g(\theta)=\frac{1}{2}\|{\bf y}\|_2^2 - \frac{\lambda^2}{2}\|\theta - \frac{{\bf y}}{\lambda}\|_2^2,\\ \nonumber
	\mbox{subject to}\quad &|{\bf x}_i^{T}\theta|\leq 1,\,i=1,2,\ldots,p.
\end{align}

\subsection*{A2. The KKT Conditions}

Problem (\ref{Equation origin_prob_new}) is clearly convex and its constraints are all affine. By Slater's condition, as long as problem (\ref{Equation origin_prob_new}) is feasible we will have strong duality. Denote ${\beta}^{\ast}$, ${\bf z}^{\ast}$ and $\theta^{\ast}$ as optimal primal and dual variables. The Lagrangian is

\begin{equation}\label{Equation Lagrangin_new}
	L({\beta, {\bf z}, \theta})= \frac{1}{2}\|{\bf z}\|_2^2+\lambda\|{\beta}\|_1 + \lambda \theta^{T}\cdot({\bf y - X\beta - z}).
\end{equation}
From the KKT condition, we have
\begin{equation}\label{Equation KKT_w}
	0\in\partial_{\beta}L({\beta}^{\ast}, {\bf z}^{\ast}, \theta^{\ast})=-\lambda{\bf X}^{T}\theta^{\ast}+\lambda{\bf v}, \mbox{ in which }\|{\bf v}\|_{\infty}\leq 1 \mbox{ and }{\bf v}^{T}{\beta}^{\ast}=\|{\beta}^{\ast}\|_1,
\end{equation}

\begin{equation}\label{Equation KKT_y}
	\nabla_{\bf z}L({\beta}^{\ast}, {\bf z}^{\ast}, \theta^{\ast})={\bf z}^{\ast}-\lambda\theta^{\ast}=0,
\end{equation}

\begin{equation}\label{Equation KKT_theta}
	\nabla_{\theta}L({\beta}^{\ast}, {\bf z}^{\ast}, \theta^{\ast})=\lambda({\bf y} - {\bf X}{\beta}^{\ast}-{\bf z}^{\ast})=0.
\end{equation}

From \eqref{Equation KKT_y} and (\ref{Equation KKT_theta}), we have:
\begin{equation}\label{Equation primal_dual_variables_lasso}
	{\bf y}={\bf X}{\beta}^{\ast}+\lambda\theta^{\ast}.
\end{equation}
From \eqref{Equation KKT_w}, we know there exists ${\bf v}^{\ast}\in\partial\|\beta^{\ast}\|_1$ such that $${\bf X}^{T}\theta^{\ast}={\bf v}^{\ast},\,\|{\bf v}^{\ast}\|_{\infty}\leq 1\mbox{ and } ({\bf v^{\ast}})^{T}{\beta}^{\ast}=\|{\beta}^{\ast}\|_1,$$
which is equivalent to
\begin{equation}\label{Equation dual_theta_b}
	|{\bf x}_i^{T}\theta^{\ast}|\leq 1, i=1,2,\ldots,p,\mbox{ and }(\theta^{\ast})^{T}{\bf X}{\beta}^{\ast}=\|{\beta}^{\ast}\|_1.
\end{equation}
From \eqref{Equation dual_theta_b}, it is easy to conclude:
\begin{equation}\label{Equation dual_w}
	(\theta^{\ast})^{T}{\bf x}_i\in
	\begin{cases}
		{\sign(\beta^*_i)}\mbox{ if }\beta_i^{\ast}\neq 0,  \\
		[-1, 1]\hspace{3.5mm}\mbox{ if }\beta_i^{\ast}=0.   \\
	\end{cases}
\end{equation}

\section*{Appendix B.}
In this appendix, we present the detailed derivation of the dual problem of group Lasso.
\subsection*{B1. Dual Formulation}
Assuming the data matrix is ${\bf X}_g\in\mathbb{R}^{N\times n_g}$ and $p=\sum_{g=1}^{G}n_g$, the group Lasso problem is given by:
\begin{equation}\label{problem:group_lasso_primal}
	\inf_{\beta\in\mathbb{R}^{p}}\frac{1}{2}\|{\bf y}-\sum_{g=1}^{G}{\bf X}_g\beta_g\|_2^2+\lambda\sum_{g=1}^{G}\sqrt{n_g}\|\beta_g\|_2.
\end{equation}
Let ${\bf z}={\bf y}-\sum_{g=1}^{G}{\bf X}_g\beta_g$ and problem (\ref{problem:group_lasso_primal}) becomes:
\begin{align}\label{problem:group_lasso_primal_v1}
	\inf_{\beta}\qquad&\frac{1}{2}\|{\bf z}\|_2^2+\lambda\sum_{g=1}^{G}\sqrt{n_g}\|\beta_g\|_2,\\ \nonumber
	\mbox{subject to} \qquad&{\bf z}={\bf y}-\sum_{g=1}^{G}{\bf X}_g\beta_g.
\end{align}

By introducing the dual variables $\eta\in\mathbb{R}^N$, the Lagrangian of problem (\ref{problem:group_lasso_primal_v1}) is:
\begin{equation}\label{equation:g_Lagrangin}
	L({\beta, {\bf z}, \eta})= \frac{1}{2}\|{\bf z}\|_2^2+\lambda\sum_{g=1}^{G}\sqrt{n_g}\|\beta_g\|_2 + \eta^{T}\cdot({\bf y}-\sum_{g=1}^{G}{\bf X}_g\beta_g - {\bf z}).
\end{equation}
and the dual function $g(\eta)$ is:
\begin{align}\label{equation:g_Dual}
	g(\eta) = \inf_{\beta, {\bf z}}L({\beta, {\bf z}, \eta})=\eta^{T}{\bf y}+\inf_{\beta}\bigg(-\eta^{T}\sum_{g=1}^{G}{\bf X}_g\beta_g+\lambda\sum_{g=1}^{G}\sqrt{n_g}\|\beta_g\|_2\bigg)+\inf_{\bf z}\big(\frac{1}{2}\|{\bf z}\|_2^2-\eta^{T}{\bf z}\big).
\end{align}

In order to get $g(\eta)$, let us solve the following two optimization problems.
\begin{equation}\label{equation:g_minbeta}
	\inf_{\beta}-\eta^{T}\sum_{g=1}^{G}{\bf X}_g\beta_g+\lambda\sum_{g=1}^{G}\sqrt{n_g}\|\beta_g\|_2,
\end{equation}
and
\begin{equation}\label{equation:g_minz}
	\inf_{\bf z}\frac{1}{2}\|{\bf z}\|_2^2-\eta^{T}{\bf z}.
\end{equation}

Let us first consider problem (\ref{equation:g_minbeta}). Denote the objective function of problem (\ref{equation:g_minbeta}) as
\begin{align}\label{equation:g_obj_minbeta}
	\hat f({\beta}) = -\eta^{T}\sum_{g=1}^{G}{\bf X}_g\beta_g+\lambda\sum_{g=1}^{G}\sqrt{n_g}\|\beta_g\|_2,
\end{align}
Let $$\hat f_g(\beta_g)=-\eta^{T}{\bf X}_g\beta_g+\lambda\sqrt{n_g}\|\beta_g\|_2,\qquad g=1,2,\ldots,G.$$
then we can split problem (\ref{equation:g_minbeta}) into a set of subproblems.
Clearly $\hat f_g(\beta_g)$ is convex but not smooth because it has a singular point at $0$. Consider the subgradient of $\hat f_g$,
$$\partial\hat f_g({\beta}_g) = -{\bf X}_g^{T}\eta + \lambda\sqrt{n_g}{\bf v}_g,\qquad g=1,2,\ldots,G,$$
where ${\bf v}_g$ is the subgradient of $\|\beta_g\|_2$:
\begin{equation}\label{equation:g_subgradient}
	{\bf v}_g\in
	\begin{cases}
		\frac{\beta_g}{\|\beta_g\|_2}\qquad\qquad\mbox{ if }\beta_g\neq 0,  \\
		{\bf u},\, \|{\bf u}\|_2\leq 1\quad\mbox{ if }\beta_g=0.   \\
	\end{cases}
\end{equation}
Let $\beta_g'$ be the optimal solution of $\hat f_g$, then $\beta_g'$ satisfy $$\exists {\bf v}_g'\in\partial\|\beta_g'\|_2,\quad -{\bf X}_g^{T}\eta + \lambda\sqrt{n_g}{\bf v}_g' = 0.$$

If $\beta_g'=0$, clearly, $\hat f_g(\beta_g')=0$. Otherwise, since $\lambda\sqrt{n_g}{\bf v}_g' = {\bf X}_g^{T}\eta$ and ${\bf v}_g'=\frac{\beta_g'}{\|\beta_g'\|_2}$, we have
$$\hat f_g(\beta_g')=-\lambda\sqrt{n_g}\frac{(\beta_g')^T}{\|\beta_g'\|_2}\beta_g'+\lambda\sqrt{n_g}\|\beta_g'\|_2=0.$$

All together, we can conclude the
$$\inf_{\beta_g} \hat f_g(\beta_g)=0,\quad g = 1,2,\ldots,G$$
and thus $$\inf_{\beta}\hat f(\beta)=\inf_{\beta}\sum_{g=1}^{G}\hat f_g(\beta_g)=\sum_{g=1}^{G}\inf_{\beta_g}\hat f_g(\beta_g)=0.$$
The second equality is due to the fact that $\beta_g$'s are independent.

Note, from \eqref{equation:g_subgradient}, it is easy to see $\|{\bf v}_g\|_2\leq 1$. Since $\lambda\sqrt{n_g}{\bf v}_g' = {\bf X}_g^{T}\eta$, we get a constraint on $\eta$, i.e., $\eta$ should satisfy:
$$\|{\bf X}_g^{T}\eta\|_2\leq\lambda\sqrt{n_g},\qquad g=1,2,\ldots,G.$$

Next, let us consider problem (\ref{equation:g_minz}). Since problem (\ref{equation:g_minz}) is exactly the same as problem (\ref{Equation miny}), we conclude:
$${\bf z}^{\prime} = \mathop{\argmin}_{\bf z}\frac{1}{2}\|{\bf z}\|_2^2-\eta^T{\bf z} = \eta,$$
and $$\inf_{\bf z}\frac{1}{2}\|{\bf z}\|_2^2-\eta^T{\bf z} = -\frac{1}{2}\|\eta\|_2^2.$$
Therefore the dual function $g(\eta)$ is:
$$g(\eta)=\eta^T{\bf y}-\frac{1}{2}\|\eta\|_2^2.$$

Combining everything above, we get the dual formulation of the group Lasso:
\begin{align}\label{problem:group_lasso_dual}
	\sup_{\eta}\quad &g(\eta)=\eta^T{\bf y}-\frac{1}{2}\|\eta\|_2^2,\\ \nonumber
	\mbox{subject to}\quad &\|{\bf X}_g^{T}\eta\|_2\leq \lambda\sqrt{n_g},\,g=1,2,\ldots,G.
\end{align}
which is equivalent to
\begin{align}\label{Equation dual_prob_new1_glasso}
	\sup_{\eta}\quad &g(\eta)=\frac{1}{2}\|{\bf y}\|_2^2 - \frac{1}{2}\|\eta - {\bf y}\|_2^2,\\ \nonumber
	\mbox{subject to}\quad &\|{\bf X}_g^{T}\eta\|_2\leq \lambda\sqrt{n_g},\,g=1,2,\ldots,G.
\end{align}

By a simple re-scaling of the dual variables $\eta$, i.e., let $\theta = \frac{\eta}{\lambda}$, problem (\ref{Equation dual_prob_new1_glasso}) transforms to:
\begin{align}\label{problem:group_lasso_dual_f}
	\sup_{\theta}\quad &g(\theta)=\frac{1}{2}\|{\bf y}\|_2^2 - \frac{\lambda^2}{2}\|\theta - \frac{{\bf y}}{\lambda}\|_2^2,\\ \nonumber
	\mbox{subject to}\quad &\|{\bf X}_g^{T}\theta\|_2\leq \sqrt{n_g},\,g=1,2,\ldots,G.
\end{align}

\subsection*{B2. The KKT Conditions}

Clearly, problem (\ref{problem:group_lasso_primal_v1}) is convex and its constraints are all affine. By Slater's condition, as long as problem (\ref{problem:group_lasso_primal_v1}) is feasible we will have strong duality. Denote ${\beta}^{\ast}$, ${\bf z}^{\ast}$ and $\theta^{\ast}$ as optimal primal and dual variables. The Lagrangian is

\begin{equation}\label{Equation g_Lagrangin_new}
	L({\beta, {\bf z}, \theta})= \frac{1}{2}\|{\bf z}\|_2^2+\lambda\sum_{g=1}^{G}\sqrt{n_g}\|\beta_g\|_2 + \lambda\theta^{T}\cdot({\bf y}-\sum_{g=1}^{G}{\bf X}_g\beta_g - {\bf z}).
\end{equation}
From the KKT condition, we have
\begin{equation}\label{Equation g_KKT_beta}
	0\in\partial_{\beta_g}L({\beta}^{\ast}, {\bf z}^{\ast}, \theta^{\ast})=-\lambda{\bf X}_g^{T}\theta^{\ast}+\lambda\sqrt{n_g}{\bf v}_g, \mbox{ in which }{\bf v}_g\in\partial\|\beta_g^{\ast}\|_2,\quad g=1,2,\ldots,G,
\end{equation}

\begin{equation}\label{Equation g_KKT_z}
	\nabla_{\bf z}L({\beta}^{\ast}, {\bf z}^{\ast}, \theta^{\ast})={\bf z}^{\ast}-\lambda\theta^{\ast}=0,
\end{equation}

\begin{equation}\label{Equation g_KKT_theta}
	\nabla_{\theta}L({\beta}^{\ast}, {\bf z}^{\ast}, \theta^{\ast})=\lambda\cdot({\bf y}-\sum_{g=1}^{G}{\bf X}_g\beta_g^{\ast} - {\bf z}^{\ast})=0.
\end{equation}

From \eqref{Equation g_KKT_z} and (\ref{Equation g_KKT_theta}), we have:
\begin{equation}\label{Equation primal_dual_variables_glasso}
	{\bf y}=\sum_{g=1}^{G}{\bf X}_g\beta_g^{\ast}+\lambda\theta^{\ast}.
\end{equation}

From \eqref{Equation g_KKT_beta}, we know there exists ${\bf v}_g'\in\partial\|\beta_g^{\ast}\|_2$ such that $${\bf X}_g^{T}\theta^{\ast}=\sqrt{n_g}{\bf v}_g'$$
and
\begin{equation*}
	{\bf v}_g'\in
	\begin{cases}
		\frac{\beta_g^{\ast}}{\|\beta_g^{\ast}\|_2}\qquad\qquad\mbox{ if }\beta_g^{\ast}\neq 0,  \\
		{\bf u},\, \|{\bf u}\|_2\leq 1\quad\mbox{ if }\beta_g^{\ast}=0,   \\
	\end{cases}
\end{equation*}

Then the following holds:
\begin{equation}\label{equation:g_KKT}
	{\bf X}_g^{T}\theta^{\ast}\in
	\begin{cases}
		\sqrt{n_g}\frac{\beta_g^{\ast}}{\|\beta_g^{\ast}\|_2}\hspace{10.5mm}\mbox{ if }\beta_g^{\ast}\neq 0,  \\
		\sqrt{n_g}{\bf u},\,\|{\bf u}\|_2\leq 1\mbox{ if }\beta_g^{\ast}= 0,   \\
	\end{cases}
\end{equation}
for $g=1,2,\ldots,G$. Clearly, if $\|{\bf X}_g^{T}\theta^{\ast}\|_2<\sqrt{n_g}$, we can conclude $\beta_g^{\ast}=0$.

\vskip 0.2in
\bibliographystyle{plainnat}
\bibliography{refs}

\end{document}